\documentclass{article}
\usepackage[preprint]{neurips_2026}
\usepackage{natbib}

\usepackage[utf8]{inputenc} 
\usepackage[T1]{fontenc}    
\usepackage{hyperref}       
\usepackage{url}            
\usepackage{booktabs}       
\usepackage{amsfonts}       
\usepackage{amsmath}
\usepackage{amssymb}
\usepackage{amsthm}
\usepackage{mathrsfs}
\usepackage{nicefrac}       
\usepackage{microtype}      
\usepackage{xcolor}         
\usepackage{bm}
\usepackage{bbm}
\usepackage{appendix}

\newtheorem{theorem}{Theorem}
\newtheorem{lemma}[theorem]{Lemma} 
\newtheorem{proposition}[theorem]{Proposition} 
\newtheorem{remark}[theorem]{Remark}
\newtheorem{corollary}[theorem]{Corollary}
\newtheorem{definition}[theorem]{Definition}

\newtheorem{assumption}{Assumption}
\newtheorem{criterion}{Criterion}

\usepackage{algorithm}
\usepackage{algpseudocode}

\newcommand{\R}{\mathbb{R}}

\newcommand{\E}{\mathbb{E}}

\title{Uniform Scaling Limits in AdamW-Trained Transformers}

\author{
  William ~Gibson\\
  Mathematical Institute\\
  University of Oxford\\
  Oxford, OX2 6GG, UK\\
  \texttt{william.gibson@maths.ox.ac.uk}
   \and
  Christoph ~Reisinger \\ Mathematical Institute \\
  University of Oxford \\ Oxford, OX2 6GG, UK\\
  \texttt{christoph.reisinger@maths.ox.ac.uk}
}

\begin{document}

\maketitle

\begin{abstract}
  We study the large-depth limit of transformers trained with AdamW, by modelling the hidden-state dynamics as an interacting particle system (IPS) coupled through the attention mechanism.
  Under appropriate scaling of the attention heads, we prove that the joint dynamics of the hidden states and backpropagated variables converge in $L^2$, uniformly over the initial condition, to the solution of a forward--backward system of ODEs at rate $\mathcal O(L^{-1}+L^{-1/3}H^{-1/2})$. Here, $L$ and $H$ denote the depth and number of heads of the transformer, respectively. The limiting system of ODEs can be identified with a McKean--Vlasov ODE (MVODE) when the attention heads do not incorporate causal masking. By using the flow maps associated with this MVODE and applying concentration of measure techniques, we obtain bounds on the difference between the discrete and continuous models that are uniform over compact sets of initial conditions. As this is achieved without resorting to a covering argument, the constants in our bounds are independent of the number of tokens. Furthermore, under a suitable adaptation to AdamW, the bounds become independent of the token embedding dimension.
\end{abstract}
\section{Introduction}\label{sec:Intro}
The empirical scaling laws for Large Language Models \cite{kaplan2020scalinglawsneurallanguage} drive an industry's pursuit to increase the number of parameters and training data samples. Therefore, a theoretical understanding of how model design choices affect generalizability, robustness, and related properties is crucial for guiding best practices as models continue to scale. Furthermore, infinite depth and width limits of neural networks offer a more tractable setting for analysis, as demonstrated by the growing literature on the convergence of gradient descent to global optimality \cite{barboni2025understandingtraininginfinitelydeep,chizat2018globalconvergencegradientdescent,isobe2024convergenceresultcontinuousmodel}. Transferring results from these limiting models to those used in practice requires convergence bounds to hold uniformly over the model's input. Accordingly, this paper studies the uniform convergence of the unmasked transformer \cite{NIPS2017_3f5ee243} as the depth $L$, the number of heads $H$ tend to infinity. The main result is given in Theorem \ref{theorem:WTS}, which demonstrates that the uniform convergence is independent of the number of tokens. This is surprising, as one expects uniform convergence bounds to scale with the effective dimension of the class over which the supremum is taken. Bounds of this form arise from covering arguments, such as Dudley's Entropy Integral Theorem \cite{Wainwright_2019}, and would therefore introduce dependency on the number of tokens.

\subsection{Setup}
 The transformer neural network \cite{NIPS2017_3f5ee243} is a ResNet formed from alternating multi-head attention \eqref{eqn:MHA} and multi-layer perceptron (MLP) layers. Each attention head consists of Query, Key, Value and Output matrices $Q_r^h,K_r^h,V_r^h,O_r^h\in \mathbb R^{k \times d}$, respectively, where $d$ is the token embedding dimension, $k$ is the head dimension, and $r,h$ indicate the layer and head number. Given a matrix of token embeddings $\mathcal X\in \R^{N\times d}$ whose $i^{th}$ row is denoted $x^i\in \mathbb R^d$, the unmasked multi-head attention layer is defined by
\begin{equation}\label{eqn:MHA}
    \mathrm{MHA}_{r}(x,\mathcal X) :=   \sum_{h=1}^H (O^h_r)^T\frac{\sum_{i=1}^N \exp(\beta\langle Q_r^hx, K_r^hx^{i}\rangle )\, V_r^hx^i}{\sum_{i=1}^N \exp(\beta\langle Q_r^hx,K_r^hx^{i}\rangle)},
\end{equation}
where $\beta > 0$ is the inverse temperature introduced in \cite{geshkovski2025mathematicalperspectivetransformers}, which controls the concentration of the attention matrix to individual tokens.
As seen in \cite{geshkovski2024emergenceclustersselfattentiondynamics,furuya2024transformersuniversalincontextlearners,geshkovski2025mathematicalperspectivetransformers,rigollet2025meanfielddynamicstransformers,sander2022sinkformerstransformersdoublystochastic}, the attention mechanism is well-defined on compactly supported measures by replacing the summation over the tokens $x^i$ in equation \eqref{eqn:MHA} with an integral against the empirical measure of tokens. We denote the corresponding function on measure $\mu \in \mathcal P_c(\R^d)$ by
\begin{equation}\label{eqn:gamma}
\gamma : \R^d\times \mathcal P_c(\R^d)\to \R^d, \qquad
    \gamma(z,\mu) = \frac{\int \exp( \left\langle z,  y \right\rangle)\,y\,d\mu(y)}{\int \exp( \left\langle z, y \right\rangle)\,d\mu(y)},
\end{equation}
where the token embedding $x^i$ and parameters $\beta K^TQ$ have been absorbed into one variable \\$z=\beta K^TQx$ so that Lipschitz continuity with respect to parameters and token embeddings can be studied simultaneously. We write $\theta\in \R^{k\times 4d}$ for the parameters in a layer and head, split as $\theta = (\theta_Q,\theta_K,\theta_V,\theta_O)$ with each $\theta_\cdot\in\R^{k\times d}$.  Multi-head attention averages $\gamma$ over the heads. Therefore, with parameters $\nu\in \mathcal P_c(\R^{k\times 4d})$, it is given by
\begin{equation}\label{eqn:Gamma}
\Gamma:\R^d\times \mathcal P_c(\R^d)\times \mathcal P_c(\R^{k\times 4d}) \mapsto \R^d, \qquad
    \Gamma(x,\mu,\nu) =  \int \theta_O^T\theta_V\,\gamma\left(\beta \theta_K^T\theta_Q\,x,\mu\right)d\nu(\theta).
\end{equation}
The embedding of the $i^{th}$ token propagates from layer $r$ to $r+1$ in the attention-only transformer according to
\begin{equation}\label{eqn:IPS}
    x^i_{r+1} = x_{r}^i+\frac{1}{L}\Gamma\left(x^i_r,m^N_r,\underline{\nu}^H_r\right),
\end{equation}
where $m^N_r, \underline{\nu}^H_r$ are the empirical measures
\begin{equation}\label{eqn:nu_H}
    m^N_r := \frac{1}{N}\sum_{i=1}^N\delta_{x^i_r}, \qquad \underline{\nu}^H_r:= \frac{1}{H}\sum_{h=1}^H \delta_{(Q^h_r,K^h_r,V^h_r,O^h_r)}.
    \end{equation}
 For simplicity, we restrict our analysis to the attention-only transformer without any normalisation layers, having depth-scaled the residual to ensure that the hidden states and backpropagated variables do not explode. In fact, incorporating a post-normalisation layer \cite{xiong2020layernormalizationtransformerarchitecture} improves the bounds given in Theorem \ref{theorem:WTS}. We use Gr\"onwall's Lemma to obtain exponential bounds on the norm of the hidden states. This exponential growth does not occur when normalisation layers are incorporated. The transformer given by \eqref{eqn:IPS} uses attention heads without causal masking. Therefore, our work is relevant to Vision Transformers \cite{dosovitskiy2021imageworth16x16words}, Diffusion Transformers \cite{peebles2023scalable}, and BERT-style language models \cite{devlin2019bert}.

Note that the residual in the transformer given by \eqref{eqn:IPS} has been explicitly scaled by $(LH)^{-1}$, where the $H^{-1}$ scaling is implicit to the integration against the empirical measure $\underline{\nu}^L$. Alternatively, transformers with $L^{-1/2}$-scaled residual have been studied in \cite{agazzi2026stochastic,fedorov2026clustering,koubbi2026homogenized} and results in the hidden states converging to the solution of a system of It\^o Stochastic Differential Equations (SDEs). However, these papers focus on the token clustering phenomena and do not consider the training dynamics. A complication that would arise when adapting the methodology in the present paper to the SDE scaling regime under training is that in the limit, the adjoint variable (see Section \ref{sec:adjoint}) is expected to satisfy a backward SDE, whose solution consists of a pair of adapted processes. As in the classical setting of forward-backward stochastic differential equations (see \cite{Ma2007}), the second process is required to ensure that the adjoint process is adapted to the same filtration as the forward process. A priori, it is unclear whether the adjoint variable of the transformer is adapted to the filtration generated by the parameter initialisation. 
Extending our analysis to this regime is left for future research.
\subsection{Related Work}
The large depth limit for ResNets whose residual has been scaled by $L^{-1}$ towards neural ODEs has previously been studied in \cite{avelin2020neuralodesdeeplimit,  marion2024implicit}. These works consider parameters that are either shared across layers \cite{avelin2020neuralodesdeeplimit} or dependently initialised across layers \cite{marion2024implicit}.
This does not reflect common practice where parameters are initialised $i.i.d.$ from a (sub-)Gaussian distribution. However, there is no meaningful continuous-time analogue to $i.i.d.$ parameter initialisation as we will elaborate in Section \ref{cont_time}. 
Instead, we consider the simultaneous infinite-depth and infinite-width limit, which is well-posed due to the continuity of the distribution of parameters across layers. This simultaneous large depth and width limit was established in \cite{ding2}, where the limiting hidden state dynamics are governed by a neural mean ODE with a velocity field defined by the expectation over the parameter distribution. They obtained an error bound between the hidden states of the mean ODE and the discrete model that was of order $\mathcal O(L^{-1}+ H^{-1/2})$. However, they still required parameter dependence across the layers at initialisation. This infinite depth-and-width limit was extended to the transformer and to include $i.i.d$ parameter initialisation by \cite{gao2024global}, yielding a convergence rate $\mathcal O(L^{-1}+\sqrt{\log(L+1)/H})$.
In contrast, \cite{chizat2025hiddenwidthdeepresnets}, and subsequently \cite{chaintron2026resnets}, views the ResNet with $i.i.d.$ parameter initialisation as a stochastic approximation, in the sense of \cite{doi:10.1137/S0363012993253534}, to the mean ODE. In this regime, the authors of \cite{chizat2025hiddenwidthdeepresnets} demonstrated that large-depth and width convergence occurs with high probability at an improved rate $\mathcal O(L^{-1}+(LH)^{-1/2})$. Thus, ResNets with $i.i.d.$ initialised parameters converge to the corresponding mean ODE without the requirement for the width to tend to infinity.

The aforementioned works consider ResNets trained using gradient descent or gradient flows. However, large neural networks, especially transformers \cite{zhao2025deconstructingmakesgoodoptimizer}, fail to train using SGD without careful hyperparameter tuning. Accordingly, the AdamW optimiser \cite{loshchilov2018decoupled} dominates the training of state-of-the-art Large Language Models, \cite{brown2020languagemodelsfewshotlearners,grattafiori2024llama3herdmodels,deepseekai2025deepseekv3technicalreport}. Also, a common assumption in the mean-field Neural ODE literature, used in \cite{ding2021globalconvergencegradientdescent,gao2024global,isobe2024convergenceresultcontinuousmodel} for example, is that the gradient of the residual with respect to the parameters grows (sub-)linearly with the parameters.  This is similar to the quadratic parameter growth assumption of \cite{barboni2025understandingtraininginfinitelydeep}. However, in the softmax self-attention proposed by \cite{NIPS2017_3f5ee243}, this gradient grows cubically, due to the products between the pairs of matrices $Q,K$ and $O,V$. This cubic growth dominates $\ell^2$ regularisation, causing theoretical guarantees for the Lipschitz constants to grow exponentially with training. Hence, a naive extension of \cite{chizat2018globalconvergencegradientdescent}, for ResNets trained with SGD, would produce convergence bounds that grow triply exponentially in the number of training steps. Instead, under weak assumptions, we show in Lemma \ref{lemma:GlobalWellDefined} that the decoupled weight-decay of AdamW forces the support of the parameter measures $\underline{\nu}_r^H$ to lie in a compact set that is common to all training steps. This provides us with Lipschitz constants that are global in training, see Remark \ref{remark:GlobalLip}.
\subsection{Our Contributions}
\begin{itemize}
    \item The main result is given in Theorem \ref{theorem:WTS}, where we study the unmasked attention-only transformer given by \eqref{eqn:IPS} with parameters initialised $i.i.d.$ from a distribution with compact support, eg. Xavier Uniform \cite{pmlr-v9-glorot10a}. We demonstrate that the transformer converges in $L^2$, uniformly over the initial conditions, to an interacting particle system of mean ODEs at rate $\mathcal O(L^{-1}+L^{-1/3}H^{-1/2})$. This extends the results of \cite{chizat2025hiddenwidthdeepresnets,chaintron2026resnets}\footnote{In \cite{chaintron2026resnets}, the dependence on the cardinality of the training data set $n$ in Theorem 2.5 is stated in Remark 4.}, who demonstrate uniform convergence for the classical ResNet over a finite training data set of cardinality $n$, with bounds that grow with $\sqrt{\log(n)}$. In contrast, we consider training data sampled $i.i.d.$ from an arbitrary distribution $\Xi$ supported on $(\bar B(R_0))^N$ and establish uniform convergence over $(\bar B(R_0))^N$ with bounds that are independent of the number of tokens $N$. Accordingly, our results hold for initial conditions that were not included in the training data.
    \item We extend the large-depth convergence results of \cite{avelin2020neuralodesdeeplimit,ding2021globalconvergencegradientdescent,gao2024global,marion2024implicit} to transformers trained using AdamW \cite{loshchilov2018decoupled}. This is crucial for the transformer as the decoupled weight-decay of AdamW can be used to show local Lipchitz continuity bounds for $\gamma$ and its derivatives that are global in training step, see Lemma \ref{lemma:GlobalWellDefined} and Remark \ref{remark:GlobalLip}.
    \item Using the Blockwise AdamW optimiser, introduced in \cite{xie2024adam} and outlined in Section \ref{sec:AdamW}, our bounds become independent of the token embedding dimension $d$ as well.
    \item  We explicitly compute Lions' derivative on the measure component of $\gamma$, given in Lemma \ref{lemma:Lderivative}. As a result, we find that $\gamma$ is locally Lipschitz continuous with respect to the $p$-Wasserstein distance (see Appendix \ref{App:A}) for every $p\in[1,\infty]$, where the Lipschitz constant is non-exponential when $p=\infty$. Furthermore, we provide explicit bounds for the local Lipschitz continuity of its derivatives $\nabla_x\Gamma$ and $\partial_\mu \Gamma$ as well, given in Lemma \ref{lemma:Lipderivs}. 
\end{itemize}
We refer the reader to Appendix \ref{App:A1} for notation and conventions.

\section{Mean-Field Attention and its Derivatives}\label{sec:MFA}
The generalised attention $\gamma$ defined on compactly supported measures was introduced in \eqref{eqn:gamma}.
For Lions' derivative (see Appendix \ref{App:A}) of $\gamma$ to be well-defined, the Fr\'echet derivative of its lifting $\hat {\gamma}$ must belong to $L^2$; latter requires the measure to have finite exponential moments to be well-defined. To circumvent this, fix some radius  $R>0$, which will then be chosen according to Lemma \ref{lemma:GlobalWellDefined}, and let $P_R: \R^d\mapsto \R^d$ be the projection map
\begin{equation}\label{eqn:proj_map}
    P_R(x) = \frac{x}{1+\frac{1}{8(R\vee 1)^2}(\Vert x\Vert-R)_+^2}, \qquad (r)_+=\mathrm{max}\{r,0\}.
\end{equation}
A direct computation yields that $\Vert P_R(x)\Vert \le \mathrm{min}\{2(R\vee 1), \Vert x\Vert\}$ and $\Vert DP_R(x)\Vert_\mathrm{op}\le 2$ for any $x \in \R^d$. Then, to recover a well-defined $L^2$ lifting, we define
\begin{equation}\label{eqn:gammaExt}
    \gamma^R(z,\mu) := \gamma(z,P_R{}_\#\mu),
\end{equation}
for any $z\in \R^d,\mu\in \mathcal P_2(\R^d)$. 
\begin{lemma}\label{lemma:Lderivative}
Denoting Lions' derivative by $\partial_{\mu}$, $\gamma^R$ is Lions differentiable at any $y,z\in \R^d,\mu \in \mathcal{P}_2(\mathbb{R}^d)$ with derivative given by
    \begin{align*}
        \partial_{\mu} \gamma^R(z,\mu)(y) = &\frac{\exp(\langle z,P_R(y)\rangle)}{\gamma^R_{2}(z,\mu)} \left[(P_R(y)-\gamma^R(z,\mu))z^T+I \right]D P_R(y),
    \end{align*}
    where $DP_R$ is the Fr\'echet derivative of $P_R$ and $\gamma_2^R$ is the normalisation constant
    \begin{equation}\label{eqn:Z}
    \gamma_{2}^R(z,\mu) = \int \exp(\langle z, P_R(y)\rangle) \,d\mu(y).
\end{equation}
\end{lemma}
Throughout the paper, we shall use the $p$-Wasserstein distance as our metric on the space of Borel probability measures $\mathcal P_p(\R^d)$ for $p\in[1,\infty]$, which is detailed in Appendix \ref{App:A}. Using Lions' derivative of $\gamma^R$, the Lipschitz continuity of $\gamma^R$ in its measure argument is established as follows.
\begin{corollary}\label{corr:LipMu} 
Take any $z \in \mathbb{R}^d$ and $p\in [1,\infty]$. Then, for any $\mu,\tilde{\mu}\in \mathcal{P}_p(\bar B(R))$, there exists $\Lambda_{\mu,p}: \mathbb R_{\ge0}\to \mathbb R_{\ge 0}$  such that
    \[ \Vert \gamma^R(z,\mu)-\gamma^R(z,\tilde{\mu})\Vert \le \Lambda_{\mu,p}(\Vert z \Vert) W_p(\mu, \tilde{\mu}), \]
    where $\Lambda_{\mu,p}$ is given by
$
\Lambda_{\mu,p}(\tilde R) =(1+2R\tilde R)\exp(2R\tilde Rp^{-1})
$.
\end{corollary}
\begin{remark}
 The Lipschitz continuity of $\gamma$ with respect to the measure has been studied previously in \cite{geshkovski2024emergenceclustersselfattentiondynamics,how-smooth-is-attention} for $W_2$.
 Deviating from their approach, we use calculus on $\mathcal P_2(\R^d)$ in the present work, which allows to deduce strict bounds. The Lipschitz continuity with respect to $W_p$ is studied in Appendix B8.4 of \cite{castin2025unifiedperspectivedynamicsdeep}, where the Lipschitz constants derived are independent of $p$. Thus, they observe exponential behaviour.
\end{remark}
\begin{remark} The $z$ argument of $\gamma$ is taken to be $\beta \theta_K^T\theta_Q x$ in the definition of $\Gamma$ \eqref{eqn:Gamma} and the parameters $\theta$ are integrated with respect to $\nu$. Thus, for $p<\infty$,  $\Gamma$ is Lipschitz continuous in its measure argument with respect to the $p$-Wasserstein distance only if $\nu$ has exponential moments.
\end{remark}
In Lemma \ref{lemma:GlobalWellDefined}, we demonstrate that the hidden states and adjoint variables (see Section \ref{sec:adjoint}) in our infinitely deep and wide transformer can be respectively bounded by $ R_X,R_a$ uniformly in training. Therefore, by taking $R\ge R_X\vee R_a$, the projection map $P_R$ is the identity map over the domain in which all variables in our analysis are defined. Accordingly, the superscript $R$ and projection map $P_R$ are dropped for notational convenience.
\section{Uniform Scaling Limits in Transformers}
As seen in Section \ref{sec:MFA}, for $\Gamma(x,m,\nu)$ defined in \eqref{eqn:Gamma} to be locally Lipschitz continuous with respect to the $2$-Wasserstein distance, it is necessary that $\nu$ possess finite exponential moments. A sufficient and strictly stronger condition is that $\nu$ has compact support. Thus, we assume that the parameters of the transformer are suitably initialised to guarantee that they remain uniformly bounded throughout training, see Lemma \ref{lemma:GlobalWellDefined}.
\begin{assumption}\label{assm:Init}
\textbf{Parameter initialisation}\\
    Let $\pi \in \mathcal P_c(\R^{k\times 4d})$ and suppose that there exists $\lambda>0$ such that the support of $\pi$ is contained in $\{v \in \R^{k\times 4d}: \Vert \mathcal{R}(v) \Vert_\infty \le \lambda^{-1}\}$, where $\mathcal R$ is the map used in Algorithm \ref{alg:AdamW}. The parameters of the transformer are initialised according to
    \[
    \theta^h_r \overset{i.i.d.}{\sim} \pi, \qquad \text{for} \; 0 \le r < L,\;1\le h\le H.
    \]
    Further assume that the mean-field parameters of the continuous-time model, defined in \eqref{eqn:ansatz}, satisfy $\nu_{t,0}=\pi$ at initialisation.
\end{assumption}
This assumption holds for common transformer initialisation strategies, such as Xavier Uniform \cite{pmlr-v9-glorot10a} or sampling parameters from a truncated normal distribution. 
\subsection{The Adjoint Variable}\label{sec:adjoint}
 For the gradients of the loss function to be well-defined and so that we can bound the difference between the gradients of two models, we require the following assumption on the loss function  $\ell: \mathcal P_2(\R^d)\mapsto \R_{\ge 0}$.
\begin{assumption}\label{assm:Obj}
    Assume that $\ell: \mathcal P_2(\R^d)\mapsto \R_{\ge 0}$ is differentiable in the sense of Lions (see Appendix \ref{App:A}), and that its Lions derivative satisfies the following Lipschitz condition: for every $R>0$, there exists $\Lambda_\ell=\Lambda_\ell(R)$, such that for all $y_1,y_2\in \bar B(R)$, and $\mu_1,\mu_2\in \mathcal P(\bar B(R))$,    \[
    \Vert \partial_\mu \ell(\mu_1)(y_1)-  \partial_\mu \ell(\mu_2)(y_2)\Vert \le \Lambda_\ell(R)\left( \Vert y_1-y_2\Vert + W_2(\mu_1,\mu_2)\right).
    \]
    Then as a result of the continuity of $\partial_\mu\ell$, there also exists a constant $K=K(R)$ such that $\Vert \partial_\mu\ell(\mu)(y)\Vert \le K(R)$ for every $y\in \bar B(R)$, $\mu \in \mathcal P(\bar B(R))$.
\end{assumption}
\begin{remark}
    Let $(x^i_0)_{i=1}^N$ be the input token embeddings, with associated targets $(c_i)_{i=1}^N$ that lie in some space $\mathcal S$. Then, the supervised learning problem can be embedded into our framework by augmenting the state space of the IPS to $(x,c)\in \R^d\times \mathcal S$ and using $(x^i_0,c_i)_{i=1}^N$ as the initial condition. We define the augmented IPS dynamics so that the velocity field is independent of $(c_i)_{i=1}^N$ and acts only on the $\R^d$ component. Accordingly, the supervised learning objective can be written as  $\ell(\widetilde{m}^N_L)$, where $\widetilde{m}^N_L$ is the empirical measure on the terminal distribution of the augmented IPS. The above construction corresponds to taking $c\in \mathcal S$ to be the noise injected onto an image in a (single step) Diffusion Model\footnote{We do not consider the Adaptive Layer Norm that is used in DiTs \cite{peebles2023scalable}, which depends on the diffusion step. This modification to the architecture can be ignored in the single diffusion step setting \cite{SSDMs}.},  while $c\in S$ corresponds to the token that has been masked in BERT-style LLMs. Analogous constructions hold for Vision Transformers.
\end{remark}
Let $m^N_r$, $x^i_r$ be defined according to \eqref{eqn:IPS}. The key step in determining the gradients of the loss function is to find the derivative of the loss with respect to $x^i_r$, which we call the adjoint variable. By the chain rule, the adjoint variable satisfies a backward recursion, whose terminal condition is given by the derivative of $\ell(m^N_L)$ with respect to one of its atoms $x^i_L$. Proposition 5.35 in \cite{carmona2018probabilistic} states that this terminal condition coincides with Lions' derivative of $\ell$ evaluated at $(\mu^N_L,x^i_L)$ multiplied by $N^{-1}$. We can rescale the adjoint variable by $N$ so that the norm of the terminal condition is independent of $N$. Correspondingly, summations over tokens are replaced by integrals against the empirical measure of tokens. Therefore, a direct computation yields that the rescaled adjoint variable satisfies
\begin{align}\label{adjoint}
  \underline{a}^{i}_{r} &= \underline{a}^{i}_{r+1} +\frac{1}{L}\mathcal K(x^i_r,\underline{\rho}^N_r,\underline{\nu}^H_r,\underline{a}^i_{r+1}),\qquad \underline{\rho}^N_{r}= \frac{1}{N}\sum_{i=1}^N \delta_{x^i_{r}}\delta_{\underline{a}^{i}_{r+1}},\\
   \underline{a}^i_{L} &= \partial_\mu \ell(m^N_L)(x^i_L),\label{eqn:adjoint_TC}
\end{align}
 where, for ease of notation, we have introduced $\mathcal K$ defined by
\begin{align}\label{eqn:Hamil}
    \mathcal H(x,\mu,\nu,a) &= a^T\Gamma(x,\mu,\nu),\\
  \mathcal K(x,\rho,\nu,a)&= \nabla_x \mathcal H(x,\rho|_x,\nu,a)+ \int \partial_\mu \mathcal H(y,\rho|_x,\nu,p)(x)d\rho(y,p).    \label{eqn:mathcal_K}
\end{align}
Here, $\rho|_x$ denotes the $x$-marginal of $\rho$. Then, the gradient of the loss function with respect to the parameters in the $r^{th}$ layer and the $h^{th}$ head is computed via
\begin{align}
 \nabla_{\theta^h_r}\ell(m^N_L)&=\frac{1}{NLH}\sum_{i=1}^N \partial_\nu\mathcal H (x^i_{r}, m^N_{r},\underline{\nu}^H_r,\underline{a}^i_{r+1})(\theta^h_r).\label{eqn:grad}
\end{align}
 Proposition 5.35 in \cite{carmona2018probabilistic} is invoked to identify the derivative of $\mathcal H(x,m_r^N,\underline{\nu}_r^H,a)$, evaluated at the empirical measure $\underline{\nu}_r^H$, with respect to one of its atoms $\theta^h_r$ as the Lions derivative rescaled by $H^{-1}$. For Lions' derivative to be well-defined, the Fr\'echet derivative of its lifting $\hat {\mathcal H}$ must be in $L^2$. However, for a finite-width transformer, the gradient with respect to the $Q$ and $K$ parameters grows cubically in the parameters. Similarly to the construction of $\gamma^R$ in Section \ref{sec:MFA}, we can define an extension to $\mathcal H$  by pushing forward any $\nu \in \mathcal P_2(\R^{k\times 4d})$ onto a compactly supported measure. This projection can be chosen to be the identity map over the support of $\nu_{t,\tau},\bar \nu^{N,L}_{t,\tau}$ for any $t\in[0,1]$ and $\tau\ge 0$, and so it will not affect the dynamics of the hidden states and adjoint variables. Thus, we do not incorporate it explicitly and note that this is a method to extend $\mathcal H$ to a function whose Lions' derivative is well-defined. 

\subsection{Continuous Time Models}\label{cont_time}
Since we are interested in how the discrepancy between the infinitely wide and deep model and the discrete-time model evolves over training, a subscript $\tau$ is introduced to denote which training step a variable is taken at. Let $\bm{y} =(y^1,\ldots, y^N)\in (\R^d)^N$. Then, the continuous time extension (CTE) to the hidden states $(x^i_r)_{r=0,i=1}^{L-1,N}$, adjoint variables $(\underline{a}^i_r)_{r=0,i=1}^{L-1,N}$ and parameters $(\underline{\nu}^H_r)_{r=0}^{L-1}$ of the discrete time IPS defined in equations \eqref{eqn:IPS}, \eqref{eqn:nu_H} and \eqref{adjoint} with $x^i_0=y^i$ is constructed as follows
\begin{align}\label{eqn:CTE1}
    \bar X^{i,L,\bm{y}}_{t,\tau} &= \sum_{r=0}^L x^{i,L,\bm{y}}_{r,\tau} \mathbbm{1}_{Lt \in[r,r+1)}, \quad \bar \mu^{N,L,\bm{y}}_{t,\tau}= \frac{1}{N}\sum_{i=1}^N \delta_{\bar X^{i,L,\bm{y}}_{t,\tau}},\\
    \bar a^{i,L,\bm{y}}_{t,\tau} &= \sum_{r=0}^L \underline{a}^{i,L,\bm{y}}_{r+1,\tau} \mathbbm{1}_{Lt \in(r,r+1]}, \quad \bar \rho^{N,L,\bm{y}}_{t,\tau}= \frac{1}{N}\sum_{i=1}^N \delta_{\bar X^{i,L,\bm{y}}_{t,\tau}}\delta_{\bar a^{i,L,\bm{y}}_{t,\tau}},\\
    \bar \nu^{N,L}_{t,\tau}&=\frac{1}{H}\sum_{h=1}^H \sum_{r=0}^{L-1}\delta_{\theta_{r,\tau}^h}\mathbbm{1}_{Lt \in[r,r+1)}.\label{eqn:CTE2}
\end{align}
Since, there is no measurable stochastic process $\Theta_{t,0}$ with $\Theta_{t,0},\Theta_{s,0}$ independent and identically distributed for $t,s$ arbitrarily close, the large depth limit of \eqref{eqn:CTE1} cannot converge to a finite width model under Assumption \ref{assm:Init}. Therefore, we only consider the infinitely wide and deep model and initialise its parameters $\nu_{t,\tau}$  according to
\begin{equation*}\label{eqn:initContinuousModel}
     \nu_{t,0} = \pi, \qquad \text{for} \; t\in[0,1],
 \end{equation*}  
 where $\pi$ is the distribution from which the parameters of the discrete-time transformer are sampled at initialisation. Then, the proposed continuous time limit of \eqref{eqn:CTE1}-\eqref{eqn:CTE2} is 
\begin{align}\label{eqn:ansatz}
    \frac{dX^{i,\bm{y}}_{t,\tau}}{dt} &=  \Gamma(X^{i,\bm{y}}_{t,\tau},\mu^{N,\bm{y}}_{t,\tau},\nu_{t,\tau}),\qquad&\mu^{N,\bm{y}}_{t,\tau}= \frac{1}{N}\sum_{i=1}^N\delta_{X^{i,\bm{y}}_{t,\tau}},\\    \frac{da^{i,\bm{y}}_{t,\tau}}{dt} &= - \mathcal K(X^{i,\bm{y}}_{t,\tau},\rho^{N,\bm{y}}_{t,\tau}, \nu_{t,\tau},a^{i,\bm{y}}_{t,\tau}),\qquad & \rho^{N,\bm{y}}_{t,\tau} =\frac{1}{N} \sum_{i=1}^N \delta_{X^{i,\bm{y}}_{t,\tau}}\delta_{a^{i,\bm{y} }_{t,\tau}},
\end{align}
with initial and terminal conditions
\begin{equation}\label{eqn:anstaz_IC}
    X^{i,\bm{y}}_{0,\tau}=y^i,\qquad a^{i,\bm{y}}_{1,\tau} = \partial_\mu \ell(\mu^{N,\bm{y}}_{1,\tau})(X^{i,\bm{y}}_{1,\tau}).
\end{equation}
\subsection{Parameter Updates}
We now introduce the data that our transformer models are trained on. 
\begin{assumption}\label{assm:TD2}
    Let  $\Xi\in \mathcal P((\R^d)^N)$ be the law of the training data, and suppose that there exists a constant $R_0>0$ such that $\bm{y}\in (\bar B(R_0))^N$  $\Xi$-a.s. Assume that, at step $\tau\ge 0$, the training data consists of $i.i.d.$ samples $\{\bm{y}_{b,\tau}\}_{b=1}^B$ drawn from $\Xi$.
\end{assumption}
  We denote by $(X^{i,b}_{t,\tau},a^{i,b}_{t,\tau})_{i=1}^N$ the solution to \eqref{eqn:ansatz} with initial condition $\bm{y}=\bm{y}_{b,\tau}$. The same superscript notation is used for the empirical measures $\mu^{N,b}_{t,\tau},\rho^{N,b}_{t,\tau}$, and for the corresponding discrete-time variables. To compensate for explicitly scaling the residual by $(LH)^{-1}$ in \eqref{eqn:IPS}, the gradients used to update the parameters are rescaled by $LH$, as done in \cite{,chizat2018globalconvergencegradientdescent,chaintron2026resnets,marion2024implicit}. Then, the rescaled gradients of the loss given in equation \eqref{eqn:grad} at the $(\tau+1)^{th}$ training step can be obtained by evaluating the map
\begin{align}
      & g_{t,\tau+1}[\bm{\bar\nu}^{N,L}_\tau](\theta)= \frac{1}{B}\sum_{b=1}^B \int \partial_\nu\mathcal H(x,\bar \mu^{N,L,b}_{t,\tau},\bar \nu^{N,L}_{t,\tau},a)(\theta)d\bar \rho^{N,L,b}_{t,\tau}(x,a),  \label{eqn:UpdateG}
 \end{align}
at $\theta = \theta^h_r$ for any $ 1\le h\le H$ and $0 \le r < L$. 
The gradient map for the continuous-time transformer \eqref{eqn:ansatz} is defined similarly. It differs from \eqref{eqn:UpdateG} only in how $\rho^{N,b}_{t,\tau}$ and $\bar \rho^{N,L,b}_{t,\tau}$ are calculated. 

Let $\widetilde{\bm\nu} \in (C([0,1]; \mathcal P_c(\R^{k\times 4d})))^T$. We call $\Phi_{t,\tau}[\widetilde{\bm\nu}]: \R^{k\times 4d}\mapsto \R^{k\times 4d}$, defined in equation \eqref{eqn:AdamWFlowMap}, the AdamW flow map associated to the sequence of parameters $\widetilde{\bm\nu}$. It is constructed iteratively so that $\Phi_{t,\tau}[\widetilde{\bm\nu}](\theta)$ coincides with the result of applying $\tau $ AdamW updates to $\theta$, using gradients given by $g_{t,j+1}[\widetilde{\bm\nu}_j]$ evaluated at $\Phi_{t,j}(\theta)$ for $j<\tau$.
Therefore, at every training step $\tau\ge 0$, our mean-field parameters satisfy
\begin{align}\label{eqn:barnu}
\bar\nu^{N,L}_{t,\tau} &=\Phi_{t,\tau}[\bm{\bar\nu}^{N,L}_{< \tau}]_\#\bar\nu^{N,L}_{t,0},\qquad
    \bm{\bar\nu}^{N,L}_{< \tau}=(\bar\nu^{N,L}_{0},\ldots, \bar\nu^{N,L}_{\tau-1}),\\
    \nu_{t,\tau} &=\Phi_{t,\tau}[\bm\nu_{< \tau}]_\#\pi,
    \qquad\qquad\,\,\, \bm\nu_{< \tau}=(\bm\nu_{0},\ldots,\bm\nu_{\tau-1}).\label{eqn:nu}
\end{align}
The subscript ``$<\tau$'', indicating the length of the sequence of mean-field parameters, will be dropped, as it will be clear from the context. See Appendix \ref{sec:AdamW} for further details and variants of the AdamW algorithm.

\subsection{Main Results}\label{sec:MainTheorem}
Let $(\Omega^0,\mathcal F^0,\mathbb P^0)$ and $(\Omega^1,\mathcal F^1,\mathbb P^1)$ be complete probability spaces. We assume that the training data $((\bm{y}_{\tau,b})_{b=1}^B)_{\tau\ge 0}$ are random variables defined on $(\Omega^0,\mathcal F^0,\mathbb P^0)$, and the transformer parameters $((\theta^h_r)_{h=1}^H)_{r=0}^{L-1}$ at initialisation are random variables defined on $(\Omega^1,\mathcal F^1,\mathbb P^1)$. Therefore, all variables are defined on the product space $(\Omega,\mathcal F,\mathbb P)$, where $\Omega=\Omega^0\times \Omega^1$ and $(\mathcal F,\mathbb P)$ is the completion of $(\mathcal F^0\otimes\mathcal F^1,\mathbb P^0\otimes \mathbb P^1)$. The only source of randomness in the continuous-time model \eqref{eqn:ansatz} is the sampling of training data, and so any of its associated variables are defined on $(\Omega^0,\mathcal F^0,\mathbb P^0)$. Finally, we write $\E^1$ to denote the expectation with respect to the measure $\mathbb  P^1$. 

Consider parameters $\hat \nu_{t,\tau}$ which are initialised identically to the discrete model $\hat \nu_{t,0} =\bar \nu^{N,L}_{t,0}$ but trained using the gradients of the continuous model. That is, for any $t\in[0,1]$, $\tau \in[0:T]$  and $r_t= \lfloor Lt\rfloor /L$, define
\begin{equation}\label{eqn:hatnu}
    \hat \nu^{N,L}_{t,\tau}  = \Phi_{r_t,\tau}[\bm\nu_{<\tau}]_\# \bar \nu^{N,L}_{t,0}.
\end{equation}
The discrepancy between $ \hat \nu^{N,L}_{t,\tau}$ and $\bar \nu^{N,L}_{t,\tau}$ encapsulates the error between the training algorithm of the continuous and discrete time models. This source of error is bounded in the following Lemma.
\begin{lemma}\label{lemma:ParamDiv}
    Suppose that Assumptions \ref{assm:Init}-\ref{assm:TD2} hold and take $\hat\nu_{t,\tau}$ in \eqref{eqn:hatnu} and let $\bar\nu^{N,L}_{t,\tau}$ be the discrete time model's parameters \eqref{eqn:CTE1} at the $\tau^{th}$ training step. Then, after training the parameters for $T$ steps using AdamW, Algorithm \ref{alg:AdamW} with $\mathcal R$ equal to the identity map, $0<\beta_1 \le \beta_2<1$, weight decay $\lambda$ and step sizes satisfying $\eta_k\in(0,1/\lambda)$, there exists a constant $C_1=C_1(\beta_1,\beta_2,\Lambda_{\ell},R_0,\lambda,T, d,k,\varepsilon)$ such that for any $0\le \tau \le T$, 
    \begin{equation*}
        \mathbb E^1\left[\max_{0\le r < L}  W_2^2(\hat\nu_{\frac{r}{L},\tau},\bar \nu^{N,L}_{\frac{r}{L},\tau})\right]\le C_1\left(\frac{1}{L^{2}}+\frac{1}{ L^{\frac{2}{3}}H}\right), \qquad \mathbb P^0-a.s.
        \end{equation*}
 Furthermore, the constant $C_1$ can be taken independent of $d,k$ if the models are trained via \\Algorithm \ref{alg:AdamW} with $\mathcal R$ given by equation \eqref{eqn:R}.
\end{lemma}
Our uniform convergence result for the unmasked attention-only transformer is as follows.\\ 

\begin{theorem}\label{theorem:WTS}
    Suppose that we are in the same setting as Lemma \ref{lemma:ParamDiv} and that the same assumptions hold. Then, there exists a constant $C_2=C_2(\beta_1,\beta_2,\Lambda_\ell,R_0,\lambda,T, d,k,\varepsilon)$, such that for any training step $\tau  \in [0:T]$
    \[
    \mathbb E^1\left[\sup_{t \in [0,1]}\sup_{\bm{y}\in(\bar B(R_0))^N}\max_{i\in[1:N]}\Vert \bar X^{i,L,\bm{y}}_{t,\tau} - X^{i,\bm{y}}_{t,\tau}\Vert^2+ \Vert \bar a^{i,L,\bm{y}}_{t,\tau} - a^{i,\bm{y}}_{t,\tau}\Vert^2 \right] \le C_2 \left(\frac{1}{L^{2}}+\frac{1}{ L^{\frac{2}{3}}H}\right).
    \] 
    The above inequality holds $\mathbb P^0$-a.s. Furthermore, the constant $C_2$ can be taken independent of $d,k$ if Algorithm \ref{alg:AdamW} is performed with $\mathcal R$ given by equation \eqref{eqn:R}.
\end{theorem}
\begin{remark}
The Euclidean norm of the input data to neural networks typically scales as $\sqrt{d}$. Our framework can be adapted to incorporate this by choosing $\beta=\mathcal O(d^{-1})$, and rescaling the output by $d^{-1/2}$ before applying the loss and rescaling the gradients by $d^{-1}$. Having done this, the bound in Theorem \ref{theorem:WTS} remains valid on the rescaled variables $\widetilde{X}^{i,\bm{y}}_{t,\tau}=\sqrt{d}X^{i,\bm{y}}_{t,\tau}$ and $\widetilde{a}^{i,\bm{y}}_{t,\tau}=d\,a^{i,\bm{y}}_{t,\tau}$, with the Euclidean norms replaced by the RMS norm. The adjoint variable is rescaled by $d$ so that $\widetilde{a}^{i,\bm{y}}_{t,\tau}$ has RMS norm of order one, with the additional factor of $\sqrt{d}$ arising from rescaling the output before applying the loss. For further details on this regime, see \cite{chaintron2026resnets} and Section 4 in \cite{chizat2025hiddenwidthdeepresnets}; Our rescalings differ due to alternative parameter initialisation assumptions. 
\end{remark}
\subsection{Sketch Proof}
 Firstly, we show that Adam updates are bounded (see Lemma \ref{lemma:BoundDelta}), and so the weight decay dominates for sufficiently large parameters. Therefore, there exists a compact set $\mathfrak{S}\in \R^{k\times 4d}$, that depends only on  $\beta_1,\beta_2$ $\lambda$, such that $\nu_{t,\tau}$ and $\bar \nu^{N,L}_{t,\tau}$ are supported in $\mathfrak{S}$ uniformly in $t\in[0,1]$ and training step $\tau\ge 0$. Under AdamW updates, $\mathfrak{S}$ depends on $d$ and $k$. However, this dependence is removed  when using Blockwise AdamW, Algorithm \ref{alg:AdamW}, with $\mathcal R$ given by \eqref{eqn:R}. This implies that the estimates of Appendix \ref{App:D} ensure that the velocity fields $\Gamma$ and $\mathcal K$ in \eqref{eqn:ansatz} have local Lipschitz continuity and exhibit linear growth.

We consider a forward-backwards system of McKean--Vlasov ODEs (MVODEs) \eqref{eqn:MVRODE1}-\eqref{eqn:MVRODE2} associated to the IPS \eqref{eqn:ansatz}. Given initial condition $\xi \in L^2(\Omega^1,\mathcal F^1,\mathbb P^1;\R^d)$ and by writing the solution $(Y^{\xi}_t,p^\xi_t)$ to the forward-backwards system of MVODEs as a fixed point of an operator, we demonstrate that the Picard iterates associated with this operator preserve measurability with respect to the noise induced by the training data sampling. These Picard iterates converge to $(Y^{\xi}_t,p^\xi_t)$ due to Theorem 4.21 in \cite{carmona2018probabilistic} Vol. I. By the stability of measurability under pointwise limits, it follows that $(Y^{\xi}_t,p^\xi_t)$ is $\mathcal F$-measurable, and that its conditional law $\mathcal L^1(Y^{\xi}_t,p^\xi_t)$ given  $\mathcal F^0$ (ie. the training data) is $\mathcal F^0$-measurable for every $t\in[0,1]$. The solution $(Y^{\xi}_t,p^\xi_t)$ conditioned on $\xi=x$ satisfies the ODEs \eqref{eqn:FMY}-\eqref{eqn:FMp}, whose solution we denote $(X^{x,\zeta}_t,a^{x,\zeta}_t)$ where $\zeta$ is the law of $\xi$. By comparing ODEs \eqref{eqn:FMY}-\eqref{eqn:FMp} to  \eqref{eqn:ansatz}, we obtain that the IPS \eqref{eqn:ansatz} has a unique $\mathcal F^0$-measurable solution $(X^{i,\bm{y}}_{t,\tau},a^{i,\bm{y}}_{t,\tau})$, which coincides with $(X^{x,\zeta}_t,a^{x,\zeta}_t)$ when $x=y_i$ and $\zeta$ is the empirical measure on the atoms $(y_i)_{i=1}^N$. This provides a framework to study the transformer that is independent of the number of tokens.

Next, define
    \[
    M^{x,\zeta}_{r,\tau}= \int_0^{\frac{r}{L}}\Gamma\left(X^{x,\zeta}_s, (X^{\cdot,\zeta}_s)_\#\zeta, \nu_{r_s,\tau}\right)- \Gamma(X^{x,\zeta}_s,(X^{\cdot,\zeta}_s)_\#\zeta,\hat\nu^{N,L}_{s,\tau}) ds,
    \]
for $\zeta\in \mathcal P(\bar B(R_0))$ and $x\in\bar B(R_0)$. A more general version of $M^{x,\zeta}_{r,\tau}$ is introduced in Appendix \ref{App:H} to study the forward and backward equations at the same time. In Lemma \ref{lemma:doobL2}, we show that $M^{x,\zeta}_{r,\tau}$ is a martingale with respect to the natural filtration generated by the initialised parameters. Accordingly, the supremum over $(x,\zeta)$ on $\Vert M^{x,\zeta}_{r,\tau}\Vert$, which we denote $M^*_{r,\tau}$, is a sub-martingale. Hence,  Doob's $L^2$ inequality \cite{revuz2013continuous} can be applied to bound the $L^2$ norm on the maximum value of $M^*_{r,\tau}$ over the layers $r\in[0:L]$ by the $L^2$ norm on $M^*_{L,\tau}$. 

Notice that $\gamma(z,\mu)$ and its derivatives, given in equations \eqref{eqn:gamma}, \eqref{eqn:Dzgamma} and \eqref{eqn:Dmugamma} respectively, can be expressed as ratios of functions that depend linearly on either the measure $\mu$ or the product measure $\mu\otimes \mu$. By exploiting the independence of the parameters at initialisation, concentration of measure techniques can be applied to bound the $L^1$ norm on $M^*_{L,\tau}$ above by an expression that depends linearly in $\zeta$ or $\zeta\otimes \zeta$. This argument is formalised in Lemma \ref{lemma:L1Mstar}. Linear functionals on probability measures $\zeta\in \mathcal P(\bar B(R_0))$ have extreme points at Dirac masses, which reduces the space over which the supremum is taken from the infinite-dimensional $\mathcal P (\bar B(R_0))$ to $\bar B(R_0)\times \bar B(R_0) $. Thus, the uniform bounds lose their dependence on the number of tokens.
    
In Lemma \ref{lemma:TaylorApproxE}, we split the layers into $\lceil L/m\rceil$ blocks, and approximate $X^{x,\zeta}_t$ within each block by a Taylor expansion in $t$. The $(x,\zeta)$ dependence in the bound on the $L^1$ norm of $M^*_{L,\tau}$ from Lemma \ref{lemma:L1Mstar} only involves terms that are multi-linear in $X^{x,\zeta}_t$. Furthermore, for each $l\in[0:\lceil L/m\rceil-1]$, the Taylor approximation on layers $r\in[lm:(l+1)m-1]$ only involves derivatives of $X^{x,\zeta}_t$ evaluated at $r=lm$. Therefore, we can factor out the dependency of $\partial^j_tX^{x,\zeta}_{t}$ from each block of $m$ layers. By bounding these derivatives uniformly in $(x,\zeta)$, we obtain a bound on the $L^1$ norm of $M^*_{L,\tau}$ without using a covering argument over the domain of the supremum. Having split the layers into blocks and factored out the $X_t^{x,\zeta}$ dependence, we observe averaging over $mH$ parameters in each block. Since the parameters are initialised independently and have been explicitly scaled by $(LH)^{-1}$, the $L^2$ norm on this average scales as $\sqrt{mL^{-2}H^{-1}}$. By bounding the error between $X^{x,\zeta}_t$ and its Taylor approximation in standard fashion, we optimise over $m$ to obtain our bound on the $L^1$ norm of $M^*_{L,\tau}$. Finally, noting that $\Gamma$ satisfies a bounded differences property whenever the measure $\nu$ on the parameters is empirical, McDiarmid's inequality is used in Lemma \ref{lemma:doobL2} to control the difference between the $L^1$ and $L^2$ norms.

 In Lemma \ref{lemma:StochasticApprox}, the error $\Vert X^{i,\bm{y}}_{(r+1)/L,\tau}-\bar X^{i,L,\bm{y}}_{(r+1)/L,\tau}\Vert$ is bounded by three terms. The first is $M^{x,\zeta}_{r+1,\tau}$ for $x,\zeta$ as above. The second arises from evaluating the parameters $\nu_{r_{s}}$ in $M^{x,\zeta}_{r,\tau}$ at the discrete time $r_s$, and is controlled using the Lipschitz continuity of $t\mapsto \nu_{t,\tau}$ established in Lemma \ref{lemma:Lipt}. The third term is bounded using the Lipschitz continuity of $\Gamma$, yielding the expression given in Lemma \ref{lemma:ParamDiv} together with terms involving $\Vert X^{i,\bm{y}}_{l/L,\tau}-\bar X^{i,L,\bm{y}}_{l/L,\tau}\Vert$ for $l\le r$, which are handled by Gr\"onwall's Lemma. By construction, the parameters $\bm{\hat \nu}_0^{N,L}$ and $\bm{\bar \nu}_0^{N,L}$ agree at initialisation. Hence, using the coupling induced by pushing forward $\bar \nu^{N,L}_{t,0}$ under $(\Phi_{t,\tau}[\bm \nu],\Phi_{t,\tau}[\bm{\bar {\nu}}^{N,L}])$, we find that
 \[
  W_2^2(\hat\nu_{t,\tau},\bar \nu^{N,L}_{t,\tau}) \le \int \Vert \Phi_{t,\tau}[\bm \nu](\theta)- \Phi_{t,\tau}[\bm{\bar\nu}^{N,L}](\theta)\Vert_\mathrm{F}^2\,d\bar\nu^{N,L}_{t,0}(\theta).
 \]
 By Lemma \ref{lemma:AdamLip}, the difference between $\theta_{t,\tau}:=\Phi_{t,\tau}[\bm\nu](\theta)$ and $\vartheta_{t,\tau}:=\Phi_{t,\tau}[\bm{\bar{\nu}}^{N,L}](\theta)$ is bounded by the differences in the gradients $g_{t,j+1}[\nu_j](\theta_{t,j})$ and $g_{t,j+1}[\bar\nu^{N,L}_j](\vartheta_{t,j})$, for $j<\tau$. In turn, by Lemma \ref{lemma:GradCont}, the squared Frobenius norm of the difference between the gradient maps $g_{t,j+1}[\bar\nu^{N,L}_j](\theta_{t,j})$ and $g_{t,j+1}[\bar\nu^{N,L}_j](\vartheta_{t,j})$ is bounded above by $\Vert \theta_{t,j}-\vartheta_{t,j}\Vert^2_\mathrm{F}+ \widetilde{\mathscr{L}}_j$, where $\widetilde{\mathscr{L}}_j$ is defined by
 \[
 \widetilde{\mathscr{L}}_j :=\sup_{t\in[0,1]}\sup_{\bm{y}\in (\bar B(R_0))^N}\max_{i\in[1:N]}\Vert X^{i,\bm{y}}_{t,j}-\bar X^{i,L,\bm{y}}_{t,j}\Vert^2+\Vert a^{i,\bm{y}}_{t,j}-\bar a^{i,L,\bm{y}}_{t,j}\Vert^2
 \]
 Therefore, the squared Frobenius norm between $\Phi_{t,\tau}[\bm\nu](\theta)$ and $\Phi_{t,\tau}[\bm{\bar\nu}^{N,L}](\theta)$ can be $\mathbb P$-a.s. bounded uniformly in $\theta\in \mathfrak{S}$ by a linear combination of $\mathscr{L}_j$ for $j<\tau$. Then, taking the supremum over the initial conditions, particle index and time in the three errors described above and then taking the expectation under $\mathbb P^1$, we obtain that $\E^1[\widetilde{\mathscr{L}}_\tau]$ is bounded by $C(L^{-2}+ L^{-2/3}H^{-1})$, for some $C>0$, plus a linear combination of $\E^1[\widetilde{\mathscr{L}}_j]$ for $j<\tau$. Hence, in Appendix \ref{app:J},  Lemma \ref{lemma:ParamDiv} and Theorem \ref{theorem:WTS} are concluded by a Gr\"onwall type argument over the layers and training steps.

\section{Conclusion and Future Work}
We demonstrated that the dynamics of the hidden state and adjoint variables of AdamW-trained transformers converge in $L^2$ to a forward--backward system of ODEs. The convergence holds uniformly over the initial conditions, with bounds independent of the number of tokens. If a Blockwise AdamW optimiser, introduced in \cite{xie2024adam}, is used instead, the bounds can be made independent of the token embedding dimension. We have had to sacrifice the order $\mathcal O(L^{-1}+(LH)^{-1/2})$ rate of convergence seen in \cite{chizat2025hiddenwidthdeepresnets} to  $\mathcal O(L^{-1}+L^{-1/3}H^{-1/2})$ to produce these uniform bounds. Future research will be done to close this gap. Furthermore, our bounds hold up to finite training time. We expect that, similarly to \cite{marion2024implicit}, they can be extended globally in training by exploiting the adjoint variable tending to zero as the parameters approach a local minimum of the loss.

\section{Acknowledgements}
The authors would also acknowledge support from His Majesty’s Government in the development of this research.

\bibliographystyle{plain}
\bibliography{Transformer}

\begin{thebibliography}{10}

\bibitem{agazzi2026stochastic}
Andrea Agazzi, Giuseppe Bruno, Eloy~Mosig Garc{\'\i}a, Samuele Saviozzi, and
  Marco Romito.
\newblock {Stochastic Scaling Limits and Synchronization by Noise in Deep
  Transformer Models}.
\newblock {\em arXiv preprint arXiv:2604.26898}, 2026.

\bibitem{NIPS2017_3f5ee243}
{Ashish Vaswani and Noam Shazeer and Niki Parmar and Jakob Uszkoreit and Llion
  Jones and Aidan N Gomez and \L{}ukasz Kaiser and Illia Polosukhin}.
\newblock Attention is {A}ll you {N}eed.
\newblock In I.~Guyon, U.~Von Luxburg, S.~Bengio, H.~Wallach, R.~Fergus,
  S.~Vishwanathan, and R.~Garnett, editors, {\em Advances in Neural Information
  Processing Systems}, volume~30. Curran Associates, Inc., 2017.

\bibitem{avelin2020neuralodesdeeplimit}
Benny Avelin and Kaj Nystr\"{o}m.
\newblock {Neural ODEs as the Deep Limit of ResNets with Constant Weights}.
\newblock {\em Analysis and Applications}, 19(03):397--437, 2021.

\bibitem{barboni2025understandingtraininginfinitelydeep}
Rapha\"el Barboni, Gabriel Peyr\'e, and Fran\c{c}ois-Xavier Vialard.
\newblock {Understanding the Training of Infinitely Deep and Wide ResNets with
  Conditional Optimal Transport}.
\newblock {\em Communications on Pure and Applied Mathematics},
  78(11):2149--2205, 2025.

\bibitem{doi:10.1137/S0363012993253534}
Michel Benaim.
\newblock {A Dynamical System Approach to Stochastic Approximations}.
\newblock {\em SIAM Journal on Control and Optimization}, 34(2):437--472, 1996.

\bibitem{geshkovski2024emergenceclustersselfattentiondynamics}
{Borjan Geshkovski and Cyril Letrouit and Yury Polyanskiy and Philippe
  Rigollet}.
\newblock {The {E}mergence of {C}lusters in {S}elf-{A}ttention {D}ynamics}.
\newblock In {\em {Advances in Neural Information Processing Systems}},
  volume~36, pages 57026--57037. Curran Associates, Inc., 2023.

\bibitem{brown2020languagemodelsfewshotlearners}
Tom Brown, Benjamin Mann, Nick Ryder, Melanie Subbiah, Jared~D Kaplan, Prafulla
  Dhariwal, Arvind Neelakantan, Pranav Shyam, Girish Sastry, Amanda Askell,
  Sandhini Agarwal, Ariel Herbert-Voss, Gretchen Krueger, Tom Henighan, Rewon
  Child, Aditya Ramesh, Daniel Ziegler, Jeffrey Wu, Clemens Winter, Chris
  Hesse, Mark Chen, Eric Sigler, Mateusz Litwin, Scott Gray, Benjamin Chess,
  Jack Clark, Christopher Berner, Sam McCandlish, Alec Radford, Ilya Sutskever,
  and Dario Amodei.
\newblock {Language Models are Few-Shot Learners}.
\newblock In H.~Larochelle, M.~Ranzato, R.~Hadsell, M.F. Balcan, and H.~Lin,
  editors, {\em Advances in Neural Information Processing Systems}, volume~33,
  pages 1877--1901. Curran Associates, Inc., 2020.

\bibitem{carmona2018probabilistic}
Ren{\'e} Carmona and Fran{\c{c}}ois Delarue.
\newblock {\em Probabilistic theory of {M}ean {F}ield {G}ames with
  {A}pplications I-II}, volume~3.
\newblock Springer, 2018.

\bibitem{castin2025unifiedperspectivedynamicsdeep}
Val{\'e}rie Castin, Pierre Ablin, Jos{\'e}~Antonio Carrillo, and Gabriel
  Peyr{\'e}.
\newblock A {U}nified {P}erspective on the {D}ynamics of {D}eep {T}ransformers.
\newblock {\em arXiv preprint arXiv:2501.18322}, 2025.

\bibitem{how-smooth-is-attention}
Val\'{e}rie Castin, Pierre Ablin, and Gabriel Peyr\'{e}.
\newblock How {S}mooth is {A}ttention?
\newblock In {\em Proceedings of the 41st International Conference on Machine
  Learning}, ICML'24, pages 5817 -- 5840. JMLR.org, 2024.

\bibitem{chizat2025hiddenwidthdeepresnets}
L{\'e}na{\"\i}c Chizat.
\newblock {The Hidden Width of Deep ResNets: Tight Error Bounds and Phase
  Diagrams}.
\newblock {\em arXiv preprint arXiv:2509.10167}, 2025.

\bibitem{chizat2018globalconvergencegradientdescent}
Lenaic Chizat and Francis Bach.
\newblock {On the Global Convergence of Gradient Descent for Over-parameterized
  Models using Optimal Transport}.
\newblock In {\em {Advances in Neural Information Processing Systems}},
  Proceedings of the 32nd International Conference on Neural Information
  Processing Systems, pages 3040--3050, Montr{\'e}al, Canada, December 2018.

\bibitem{ding2021globalconvergencegradientdescent}
Zhiyan Ding, Shi Chen, Qin Li, and Stephen Wright.
\newblock {On the Global Convergence of Gradient Descent for Multi-Layer
  ResNets in the Mean-Field Regime}.
\newblock {\em arXiv preprint arXiv:2110.02926}, 2021.

\bibitem{ding2}
Zhiyan Ding, Shi Chen, Qin Li, and Stephen Wright.
\newblock {Overparameterization of Deep ResNet: Zero Loss and Mean-Field
  Analysis}.
\newblock {\em J. Mach. Learn. Res.}, 23(1), January 2022.

\bibitem{dosovitskiy2021imageworth16x16words}
Alexey Dosovitskiy, Lucas Beyer, Alexander Kolesnikov, Dirk Weissenborn,
  Xiaohua Zhai, Thomas Unterthiner, Mostafa Dehghani, Matthias Minderer, Georg
  Heigold, Sylvain Gelly, Jakob Uszkoreit, and Neil Houlsby.
\newblock An {I}mage is worth 16x16 words: {T}ransformers for {I}mage
  {R}ecognition at {S}cale.
\newblock In {\em International Conference on Learning Representations}, 2021.

\bibitem{Durrett_2019}
Rick Durrett.
\newblock {\em {Probability: Theory and Examples}}.
\newblock {Cambridge Series in Statistical and Probabilistic Mathematics}.
  Cambridge University Press, 5 edition, 2019.

\bibitem{Eng89}
Ryszard Engelking.
\newblock {\em General Topology}, volume~6 of {\em Sigma Series in Pure
  Mathematics}.
\newblock Heldermann, Berlin, 1989.

\bibitem{fedorov2026clustering}
Lev Fedorov, Micha{\"e}l~E Sander, Romuald Elie, Pierre Marion, and Mathieu
  Lauri{\`e}re.
\newblock {Clustering in Deep Stochastic Transformers}.
\newblock {\em arXiv preprint arXiv:2601.21942}, 2026.

\bibitem{furuya2024transformersuniversalincontextlearners}
Takashi Furuya, Maarten~V. de~Hoop, and Gabriel Peyr{\'e}.
\newblock Transformers are {U}niversal {I}n-context {L}earners.
\newblock In {\em The Thirteenth International Conference on Learning
  Representations}, 2025.

\bibitem{gao2024global}
Cheng Gao, Yuan Cao, Zihao Li, Yihan He, Mengdi Wang, Han Liu, Jason~M.
  Klusowski, and Jianqing Fan.
\newblock {Global Convergence in Training Large-Scale Transformers}.
\newblock In {\em Advances in Neural Information Processing Systems},
  volume~37, pages 29213--29284. Curran Associates, Inc., 2024.

\bibitem{geshkovski2025mathematicalperspectivetransformers}
Borjan Geshkovski, Cyril Letrouit, Yury Polyanskiy, and Philippe Rigollet.
\newblock {A Mathematical Perspective on Transformers}.
\newblock {\em Bulletin of the American Mathematical Society}, 62(3):427--479,
  2025.

\bibitem{pmlr-v9-glorot10a}
Xavier Glorot and Yoshua Bengio.
\newblock {Understanding the Difficulty of Training Deep Feedforward Neural
  Networks}.
\newblock In Yee~Whye Teh and Mike Titterington, editors, {\em Proceedings of
  the Thirteenth International Conference on Artificial Intelligence and
  Statistics}, volume~9 of {\em Proceedings of Machine Learning Research},
  pages 249--256, Chia Laguna Resort, Sardinia, Italy, 13--15 May 2010. PMLR.

\bibitem{grattafiori2024llama3herdmodels}
Aaron Grattafiori, Abhimanyu Dubey, Abhinav Jauhri, Abhinav Pandey, Abhishek
  Kadian, Ahmad Al-Dahle, Aiesha Letman, Akhil Mathur, Alan Schelten, Alex
  Vaughan, et~al.
\newblock {The Llama 3 Herd of Models}.
\newblock {\em arXiv preprint arXiv:2407.21783}, 2024.

\bibitem{Han2017}
Xiaoying Han and Peter~E. Kloeden.
\newblock {\em Random Ordinary Differential Equations}, pages 15--27.
\newblock Springer Singapore, Singapore, 2017.

\bibitem{devlin2019bert}
{Jacob Devlin and Ming-Wei Chang and Kenton Lee and Kristina Toutanova}.
\newblock {Bert: Pre-training of Deep Bidirectional Transformers for Language
  Understanding}.
\newblock In {\em {Proceedings of the 2019 Conference of the North American
  Chapter of the Association for Computational Linguistics: Human Language
  Technologies, Volume 1 (long and short papers)}}, pages 4171--4186, 2019.

\bibitem{kaplan2020scalinglawsneurallanguage}
Jared Kaplan, Sam McCandlish, Tom Henighan, Tom~B Brown, Benjamin Chess, Rewon
  Child, Scott Gray, Alec Radford, Jeffrey Wu, and Dario Amodei.
\newblock Scaling {L}aws for {N}eural {L}anguage {M}odels.
\newblock {\em arXiv preprint arXiv:2001.08361}, 2020.

\bibitem{koubbi2026homogenized}
Hugo Koubbi, Borjan Geshkovski, and Philippe Rigollet.
\newblock {Homogenized Transformers}.
\newblock {\em arXiv preprint arXiv:2604.01978}, 2026.

\bibitem{Kumar_2021}
Chaman Kumar and Neelima.
\newblock {On Explicit Milstein-type scheme for McKean--Vlasov Stochastic
  Differential Equations with Super-Linear Drift Coefficient}.
\newblock {\em Electronic Journal of Probability}, 26(none), January 2021.

\bibitem{deepseekai2025deepseekv3technicalreport}
Aixin Liu, Bei Feng, Bing Xue, Bingxuan Wang, Bochao Wu, Chengda Lu, Chenggang
  Zhao, Chengqi Deng, Chenyu Zhang, Chong Ruan, et~al.
\newblock {Deepseek-v3 Technical Report}.
\newblock {\em arXiv preprint arXiv:2412.19437}, 2024.

\bibitem{liu2025muon}
Jingyuan Liu, Jianlin Su, Xingcheng Yao, Zhejun Jiang, Guokun Lai, Yulun Du,
  Yidao Qin, Weixin Xu, Enzhe Lu, Junjie Yan, et~al.
\newblock {Muon is Scalable for LLM Training}.
\newblock {\em arXiv preprint arXiv:2502.16982}, 2025.

\bibitem{loshchilov2018decoupled}
Ilya Loshchilov and Frank Hutter.
\newblock {Decoupled Weight Decay Regularization}.
\newblock In {\em 7th International Conference on Learning Representations,
  {ICLR} 2019}, 2019.

\bibitem{chaintron2026resnets}
{Louis-Pierre Chaintron and L{\'e}na{\"\i}c Chizat and Javier Maas}.
\newblock {Resnets of All Shapes and Sizes: Convergence of Training Dynamics in
  the Large-Scale Limit}.
\newblock {\em arXiv preprint arXiv:2603.18168}, 2026.

\bibitem{Ambrosio2008}
{Luigi Ambrosio and Nicola Gigli and Giuseppe Savar\'e}.
\newblock {\em {Absolutely Continuous Curves in $\mathcal P(X)$ and the
  Continuity Equation}}, pages 167--200.
\newblock {Birkh{\"a}user Basel}, {Basel}, 2008.

\bibitem{Ma2007}
Jin Ma and Jiongmin Yiong.
\newblock {\em {Forward-Backward Stochastic Differential Equations and their
  Applications}}, pages 1--24.
\newblock Springer Berlin Heidelberg, Berlin, Heidelberg, 2007.

\bibitem{marion2024implicit}
Pierre Marion, Yu-Han Wu, Michael~Eli Sander, and G{\'e}rard Biau.
\newblock Implicit {R}egularization of {D}eep {R}esidual {N}etworks towards
  {N}eural {ODE}s.
\newblock In {\em The Twelfth International Conference on Learning
  Representations}, 2024.

\bibitem{McDiarmid_1989}
Colin McDiarmid.
\newblock {\em {On the Method of Bounded Differences}}, page 148–188.
\newblock {London Mathematical Society Lecture Note Series}. {Cambridge
  University Press}, 1989.

\bibitem{Ledoux1991}
{Michel Ledoux and Michel Talagrand}.
\newblock {\em {Gaussian Random Variables}}, pages 54--88.
\newblock {Springer Berlin Heidelberg}, {Berlin, Heidelberg}, 1991.

\bibitem{isobe2024convergenceresultcontinuousmodel}
{Noboru Isobe}.
\newblock {A Convergence Result of a Continuous Model of Deep Learning via
  \L{}ojasiewicz-Simon inequality}.
\newblock {\em CoRR}, abs/2311.15365, 2023.

\bibitem{Kallenberg2021}
{Olav Kallenberg}.
\newblock {\em {Sets and Functions, Measures and Integration}}, pages 9--32.
\newblock Springer International Publishing, Cham, 2021.

\bibitem{SSDMs}
{Pascal Vincent}.
\newblock {A Connection between Score Matching and Denoising Autoencoders}.
\newblock {\em Neural Comput.}, 23(7):1661–1674, July 2011.

\bibitem{billingsley2013convergence}
{Patrick Billingsley}.
\newblock {\em {Convergence of Probability Measures}}.
\newblock {Wiley Series in Probability and Statistics}. Wiley, 2013.

\bibitem{peebles2023scalable}
William Peebles and Saining Xie.
\newblock {Scalable Diffusion Models with Transformers}.
\newblock In {\em {Proceedings of the IEEE/CVF International Conference on
  Computer Vision}}, pages 4195--4205, 2023.

\bibitem{peng}
{Rainer Buckdahn and Juan Li and Shige Peng and Catherine Rainer}.
\newblock {Mean-Field Stochastic Differential Equations and Associates PDEs}.
\newblock {\em {The Annals of Probability}}, 45(2):824--878, 2017.

\bibitem{revuz2013continuous}
Daniel Revuz and Marc Yor.
\newblock {\em Continuous Martingales and Brownian Motion}.
\newblock Grundlehren der Mathematischen Wissenschaften. Springer Berlin
  Heidelberg, 2013.

\bibitem{rigollet2025meanfielddynamicstransformers}
Philippe Rigollet.
\newblock {The Mean-Field Dynamics of Transformers}.
\newblock {\em arXiv preprint arXiv:2512.01868}, 2025.

\bibitem{sander2022sinkformerstransformersdoublystochastic}
Michael~E Sander, Pierre Ablin, Mathieu Blondel, and Gabriel Peyr{\'e}.
\newblock Sinkformers: {T}ransformers with {D}oubly {S}tochastic {A}ttention.
\newblock In {\em International Conference on Artificial Intelligence and
  Statistics}, pages 3515--3530. PMLR, 2022.

\bibitem{santambrogio2015optimal}
Filippo Santambrogio.
\newblock {\em {Optimal Transport for Applied Mathematicians: Calculus of
  Variations, PDEs, and Modeling}}.
\newblock Progress in Nonlinear Differential Equations and Their Applications.
  Springer International Publishing, 2015.

\bibitem{AZHMYAKOV201987}
{Vadim Azhmyakov}.
\newblock {Chapter 4 - Short Course in Continuous Time Dynamic Systems and
  Control}.
\newblock In {\em {A Relaxation-Based Approach to Optimal Control of Hybrid and
  Switched Systems}}, pages 87--126. {Butterworth-Heinemann}, 2019.

\bibitem{vandeGeer2016}
Sara van~de Geer.
\newblock {\em {Symmetrization, Contraction and Concentration}}, pages
  233--238.
\newblock {Springer International Publishing}, Cham, 2016.

\bibitem{vanderVaart1996}
Aad~W. van~der Vaart and Jon~A. Wellner.
\newblock {\em {Symmetrization and Measurability}}, pages 107--121.
\newblock {Springer New York}, New York, NY, 1996.

\bibitem{villani2008optimal}
C\'{e}dric Villani.
\newblock {\em {Optimal Transport: Old and New}}.
\newblock Grundlehren der Mathematischen Wissenschaften. Springer Berlin
  Heidelberg, 2008.

\bibitem{Wainwright_2019}
Martin~J. Wainwright.
\newblock {\em {Metric Entropy and its Uses}}, pages 121--158.
\newblock Cambridge Series in Statistical and Probabilistic Mathematics.
  Cambridge University Press, 2019.

\bibitem{pmlr-v235-xie24e}
Shuo Xie and Zhiyuan Li.
\newblock {Implicit Bias of {A}dam{W}: $\ell_\infty$-Norm Constrained
  Optimization}.
\newblock In {\em Proceedings of the 41st International Conference on Machine
  Learning}, volume 235 of {\em Proceedings of Machine Learning Research},
  pages 54488--54510. PMLR, 21--27 Jul 2024.

\bibitem{xie2024adam}
Shuo Xie, Mohamad~Amin Mohamadi, and Zhiyuan Li.
\newblock {Adam Exploits $\ell_\infty$-Geometry of Loss Landscape via
  Coordinate-wise Adaptivity}.
\newblock {\em arXiv preprint arXiv:2410.08198}, 2024.

\bibitem{xiong2020layernormalizationtransformerarchitecture}
Ruibin Xiong, Yunchang Yang, Di~He, Kai Zheng, Shuxin Zheng, Chen Xing,
  Huishuai Zhang, Yanyan Lan, Liwei Wang, and Tie-Yan Liu.
\newblock On {L}ayer {N}ormalization in the {T}ransformer {A}rchitecture.
\newblock In {\em Proceedings of the 37th International Conference on Machine
  Learning}, ICML'20. JMLR.org, 2020.

\bibitem{zhao2025deconstructingmakesgoodoptimizer}
Rosie Zhao, Depen Morwani, David Brandfonbrener, Nikhil Vyas, and Sham~M.
  Kakade.
\newblock {Deconstructing What Makes a Good Optimizer for Autoregressive
  Language Models}.
\newblock In {\em The Thirteenth International Conference on Learning
  Representations}, 2025.

\end{thebibliography}

\begin{appendices}
   \section{Background Material}\label{App:A}
\subsection{Notation \& Conventions}\label{App:A1}
We shall write $r_t$ for the time discretisation $r_t=\lfloor L t \rfloor/L$. The symbols $\vee,\wedge$ denote $a\wedge b =\mathrm{min}\{a,b\}$ and $a\vee b =\mathrm{max}\{a,b\}$ for $a,b\in\R$. For $n,m \in \mathbb Z_{\ge 0}$, we write $[n:m]=\{i \in \mathbb Z: n\le i \le m\}$.

For any radius $R>0$, let us write $\bar B_d(R)$ for the subset of $\R^d$ of elements with Euclidean norm less than or equal to $R$. The subscript $d$ will be dropped when it is the token embedding dimension. Let $\mathcal P_p(\R^d)$ denote the space of Borel probability measures on $\R^d$ with finite $p$-moment, $p\in [1,\infty)$, and write $\mathcal P_c(\R^d)$ for the space of compactly supported measures on $\R^d$. Also, for any compact subset $K$ of $\R^d$, denote the space of Borel probability measures with support contained in $K$ by $\mathcal P(K)$.

On a measurable space $(\widetilde{\Omega},\widetilde{\mathcal F})$ and for any separable metric space $E$, the map $Z:\widetilde{\Omega}\mapsto E$ is said to be $\widetilde{\mathcal F}$-measurable if it is $(\widetilde{\mathcal F},\mathcal B(E))$-measurable, where $\mathcal B(E)$ is the Borel $\sigma$-algebra on $E$. For any Borel measurable $T:E_1\mapsto E_2$, where $E_1,E_2$ are separable metric spaces, we denote the pushforward measure of $\mu\in \mathcal P(E_1)$ under $T$ by $T_\#\mu$ so that $T_\#\mu(D)=\mu(T^{-1}(D))$ for every $D\in \mathcal B(E_2)$.

Motivated by bounding each block of $\theta \in \R^{k\times 4d}$ separately, define, for any $R>0$, the set 
\begin{equation}\label{eqn:mathfrak_S}
    \mathfrak{S}(R) = \{ \theta \in \R^{k\times 4d}: \Vert \mathcal R_1(\theta)\Vert_\infty \le R\},
\end{equation}
where $\mathcal R_1$ is given by equation \eqref{eqn:R}.
\subsection{Optimal Transport}
Given $\mu,\nu \in \mathcal P_p(\R^d)$, the $p$-Wasserstein distance is given by
\[
 W_{p}(\mu,\nu) =\left(\inf_{\pi\in \Pi(\mu,\nu)}\int \Vert x-y\Vert^p d\pi(x,y)\right)^{1/p}.
\]
Where $\Pi(\mu,\nu)$ denotes the subset of $\mathcal P_p(\R^d\times \R^d)$ with marginals $\mu,\nu$.
The limit as $p\to \infty$ for $ W_{p}$ is well-defined for measures $\mu,\nu \in \mathcal P_c(\R^d)$, where the limiting value coincides with
\[
 W_{\infty}(\mu,\nu) = \underset{\pi\in \Pi(\mu,\nu)}{inf} \Vert x-y\Vert_{L^{\infty}(\pi)}.
\]
See \cite{villani2008optimal} for further details on the Wasserstein distance and Optimal Transport. Further details on $ W_{\infty}$ are given in \cite{santambrogio2015optimal}. 
\subsection{Lions' Derivative}
Let  $(\widetilde{\Omega}, \widetilde{\mathcal F},\widetilde{\mathbb P})$ be an atomless Polish probability space and take $\nu \in \mathcal P_2(\R^d)$ and $H:\mathcal P_2(\R^d)\mapsto \R$. Suppose that there exists a square integrable random variable $Y$ on $(\widetilde{\Omega}, \widetilde{\mathcal F},\widetilde{\mathbb P})$, and a lifting \\$\hat H:L^2(\widetilde{\Omega}, \widetilde{\mathcal F},\widetilde{\mathbb P})\to\R$ such that $\hat H(Y)=H(\nu)$ is Fr\'echet differentiable at $Y$.  Then, by the Reisz representation theorem, there exists a unique function $\partial_\mu H(\nu):\R^d\mapsto\R^d$ independent of the choice of the random variable $Y$ and lifting $\hat H$ such that for any $Z\in L^2(\widetilde{\Omega}, \widetilde{\mathcal F},\widetilde{\mathbb P};\R^d)$
\[
\widetilde{\mathbb E}[\langle D\hat H(Y),Z\rangle]= \widetilde{\mathbb E}[\langle\partial_{\mu}H(\nu)(Y),Z\rangle],
\]
where $\widetilde{\E}$ is the expectation with respect to the measure $\widetilde{\mathbb P}$.
The Lions derivative of $H$ at $\nu$ is given by $\partial_\mu H(\nu)$. See Chapter 5.2 in \cite{carmona2018probabilistic} Vol. I for further details.
With the Lions derivative machinery, the strategy for finding Lipschitz constant using the Mean Value Theorem can be extended to measures \cite{Kumar_2021}.
\begin{lemma}\label{ProbMeasureMVT}
    Suppose $f:\mathcal{P}_2(\mathbb{R}^d) \mapsto \mathbb{R}$ is a continuous function such that its Lions derivative $\partial_{\mu}f: \mathcal{P}_2(\mathbb{R}^d)\times \mathbb{R}^d \mapsto \mathbb{R}^d$ exists. Then there exists $t\in (0,1)$ such that 
    \[
    f(\tilde\mu)-f(\mu) = \widetilde{\mathbb{E}}\left[\left\langle\partial_{\mu}f\left(\mathcal{L}(Z + t(\tilde{Z}-Z)),Z + t(\tilde{Z}-Z)\right),\tilde{Z}-Z\right\rangle\right],
    \]
    for any $\mu,\tilde{\mu}\in \mathcal{P}_2(\mathbb{R}^d)$ and random variables $Z,\tilde{Z}$, defined on an atomless Polish probability space $(\widetilde{\Omega}, \widetilde{\mathcal{F}},\widetilde{\mathbb{P}})$, such that $\mathcal{L}(Z)=\mu,\mathcal{L}(\tilde{Z})=\tilde{\mu}$. Here, $\widetilde{\E}$ is the expectation with respect to the measure $\widetilde{\mathbb P}$.
\end{lemma}
The proof is given in Lemma 5 of \cite{Kumar_2021}.
\section{AdamW \& Variants}\label{sec:AdamW}
The training of Large Language Models without careful hyperparameter tuning is reliant on adaptive step size methods \cite{zhao2025deconstructingmakesgoodoptimizer}. Consequently, the AdamW \cite{loshchilov2018decoupled} optimisation algorithm dominates the training of Large Language Models, \cite{deepseekai2025deepseekv3technicalreport,  grattafiori2024llama3herdmodels,brown2020languagemodelsfewshotlearners}. The AdamW optimisation algorithm decouples the weight decay on the parameters from the Adam updates. However, since the weight decay is not incorporated into the loss function, the effective optimisation objective is not, a priori, clear. Theorem 1.1 in \cite{pmlr-v235-xie24e} states that if AdamW with weight decay $\lambda>0$ converges, then it converges to a KKT point of the constrained optimisation problem \eqref{eqn:CO}
 \begin{equation}\label{eqn:CO}
     \min_{\theta \in \mathbb R^d} f(\theta) \quad \text{subject to} \quad \Vert \theta \Vert_{\infty} \le \lambda^{-1},
 \end{equation}
 where $f:\mathbb R^d\to \mathbb R$ is any smooth function bounded below. 
The $z$ argument of $\gamma$ is taken to be $\beta \theta_K^T\theta_Q x$ in the definition of $\Gamma$ in \eqref{eqn:Gamma}. Therefore, the Lipschitz constants for attention scale with the operator norm on $\theta_K^T\theta_Q$. However, from equation \eqref{eqn:CO}, we see that AdamW will control each component of $\theta$ separately, and so we can only hope to bound $\Vert \theta_K^T\theta_Q\Vert_\mathrm{op}$ uniformly in training step using the infinity norm on the parameters. This introduces dimensional dependence into our bounds. Instead, Algorithm \ref{alg:AdamW} generalises AdamW \cite{loshchilov2018decoupled} to regularise with alternative norms on blocks of parameters.\\
\begin{algorithm}
\caption{\textbf{Blockwise AdamW Optimisers \cite{xie2024adam}}\\
Blockwise AdamW optimisation algorithm, where the variance accumulator is performed on $\mathcal R(g)$. Here, each gradient $g$ lies in a finite-dimensional vector space $\mathscr{V}$ and $\mathcal R:\mathscr{V} \mapsto \mathscr{V}$. AdamW corresponds to $\mathcal R$ being the identity map.}\label{alg:AdamW}
\begin{algorithmic}[1]
\Require $\theta_0 \in \mathscr{V}$, $(\eta_j)_{j=1}^T$, $\lambda$, $\beta_1,\beta_2 \in (0,1)$, $\beta_2 \ge \beta_1$, $ \varepsilon>0$
\State $m_0 ,v_0 \gets \bm{0} \in \mathscr{V}$
\For{$j = 1$ to $T$}
    \State $g_j\gets \nabla_{\theta}Loss(\theta_{j-1})$
    \State $m_j\gets \beta_1 m_{j-1}+(1-\beta_1)g_j$ \qquad $v_j \gets \beta_2 v_{j-1} + (1 -\beta_2)\mathcal R (g_j)\odot \mathcal R(g_j)$
\State $\hat m_j \gets m_j(1-\beta_1^j)^{-1}$\qquad $\hat v_j \gets v_j(1-\beta_2^j)^{-1}$
\State  $\theta_{j}\gets (1-\eta_j\lambda)\theta_{j-1}-\eta_j\frac{\hat m_j}{\sqrt{\hat v_j}+ \varepsilon}$
\EndFor
\State\Return $\theta_T$
\end{algorithmic}
\end{algorithm}\\
Take $\mathcal R: \R^{k\times 4d}\mapsto \R^{k\times 4d}$ and consider adapting the variance accumulator $\hat v$ of AdamW to be calculated on $\mathcal R(g)$ rather than $g$ per layer and head. To incentivise small Frobenius norm on each $Q,K,V,O$ matrix, consider
\begin{equation}\label{eqn:R}
    \mathcal R_1(g^{(1)},g^{(2)},g^{(3)},g^{(4)}) = \left(\Vert g^{(1)} \Vert_\mathrm{F}\bm{1},\Vert g^{(2)} \Vert_\mathrm{F}\bm{1},\Vert g^{(3)}\Vert_\mathrm{F}\bm{1},\Vert g^{(4)} \Vert_\mathrm{F}\bm{1}\right), 
\end{equation}
with each $g^{(j)}\in \R^{k\times d}$, $j=1,\ldots,4$. The brackets refer to column-wise concatenation of matrices, and $\bm{1}\in\R^{k\times d}$ is the matrix where every entry is one. This is a specific example of the Blockwise AdamW algorithm introduced in \cite{xie2024adam}. Furthermore, when $\mathcal R$ is given by \eqref{eqn:R}, the resulting optimisation algorithm is similar to the Muon algorithm with decoupled weight decay proposed in \cite{liu2025muon}. Muon is applied to blocks of weight matrices. Instead of normalising the momentum accumulator by the Frobenius norm on the variance accumulators, Muon approximately orthogonalises the matrix of momentum accumulators.

To facilitate analysis, let us now define the maps and quantities associated to Algorithm \ref{alg:AdamW}.
 Given a sequence of gradients $G=(g_j)_{k=1}^T$, $g_j\in \R^{k\times 4d}$, the momentum and variance accumulators $\hat m^\mathcal R_T,\hat v^\mathcal R_T: (\R^{k\times 4d})^T\to \R^{k\times 4d}$ are respectively defined by
\begin{equation}\label{eqn:AdamWalg}
    \hat m^\mathcal R_T(G) = \frac{1-\beta_1}{1-\beta_1^T}\sum_{j=1}^{T}\beta_1^{T-j}g_j,\qquad \hat v^\mathcal R_T(G) = \frac{1-\beta_2}{1-\beta_2^T}\sum_{j=1}^{T}\beta_2^{T-j}\mathcal R(g_j)\odot \mathcal R(g_j).
\end{equation}
Here, $\odot$ refers to the componentwise (Hadamard) product. Notice that the effect of weight decay applied at step $i$ on step $\tau$ is quantified by $\alpha_{i,\tau}=\prod_{j=i}^\tau(1-\eta_j\lambda)$, where we use the convention $\alpha_{\tau+1,\tau}=1$. 

We inductively define the AdamW flow map $\Phi_{t,\tau}[\widetilde{\bm\nu}]:\R^{k\times 4d}\mapsto \R^{k\times 4d}$ associated to the sequence of gradients $\widetilde{\bm\nu}\in (C([0,1],\mathcal{P}_c(\R^{k\times 4d})))^T$. Firstly, for every $\theta\in\R^{k\times 4d}$, define 
\[
 \Phi_{t,0}[\widetilde{\bm\nu}](\theta)=\theta.
\]
Now, suppose that $\Phi_{t,j}[\widetilde{\bm\nu}]$ is well-defined for every $j <\tau$. Then, we define $\Phi_{t,\tau}[\widetilde{\bm\nu}](\theta)$ according to
\begin{align}\label{eqn:AdamWFlowMap}
\Phi_{t,\tau}[\widetilde{\bm\nu}](\theta)=\alpha_{1,\tau}\theta - \sum_{i=1}^\tau\eta_i\alpha_{i+1,\tau}\frac{\hat m^\mathcal{R}_i(G_t[\widetilde{\bm\nu}](\theta))}{\sqrt{\hat v^\mathcal {R}_i(G_{t}[\widetilde{\bm\nu}](\theta))}+\varepsilon},
\end{align}
where division is performed componentwise.
Here, $G_t[\widetilde{\bm\nu}](\theta)\in (\R^{k\times 4d})^T$ is the sequence of gradients, whose $j$-th element is given by
\begin{align}
    \left(G_t[\widetilde{\bm\nu}](\theta) \right)_j = g_{t,j}[\widetilde{\bm\nu}_{j-1}](\Phi_{t,j-1}[\widetilde{\bm\nu}](\theta)),
\end{align}
where $ g_{t,j}[\widetilde{\bm\nu}_{j-1}]$ is the gradient map defined by \eqref{eqn:UpdateG} with hidden states and adjoint variables computed according to the mean-field parameters $\widetilde{\bm\nu}_{j-1}$.

    \section{Derivatives of $\gamma$}\label{App:C}

Notice that $\gamma^R$, as defined in \eqref{eqn:gammaExt}, is an expectation under a new measure $\mathcal A_{z,\mu}$ that is define through the Radon-Nikodym derivative
\begin{equation}\label{eqn:GammaMeasureChange}
      \frac{d \mathcal A_{z,\mu}}{d\mu}(y) = \frac{\exp( \langle z, P_R(y)\rangle)}{\gamma^R_{2}(z,\mu)},
\end{equation}
where $P_R$ and $\gamma^R_{2}$ are defined in \eqref{eqn:proj_map} and \eqref{eqn:Z}, respectively.
\begin{lemma}\label{lemma:GammaLipz}
    Take $R_1>0$ and $\gamma^R$ defined in \eqref{eqn:gammaExt}. Then, $\gamma^R$ is differentiable in $z$ at any $z \in \bar B(R_1)$ at fixed $\mu\in \mathcal P(\bar B(R_1))$. Furthermore, for every $\mu\in\mathcal P(\bar B(R_1))$  and $z,\tilde{z}\in \bar B(R_1)$, we have
    \[\Vert \gamma^R(z,\mu)-\gamma^R(\tilde{z},\mu)\Vert \leq 2 R_1^2\Vert z-\tilde{z}\Vert. \]
\end{lemma}
\begin{proof}
Write $\gamma^R_{1}=\gamma^R_{2}\gamma^R$ for the unnormalized counterpart to $\gamma^R$.
Differentiation and integration in $\gamma^R_1$, $\gamma^R_2$ can be interchanged as a result of the Dominated Convergence Theorem, where we dominate using the compact support of $(P_R)_{\#}\mu$. The respective derivatives are given by
\begin{align*}
    D_z\gamma^R_1(z,\mu) &= \int \exp( \langle z,  y \rangle) \,y y^T \, d(P_R)_{\#}\mu(y), \\ \nabla_z \gamma^R_2(z,\mu) &= \int exp( \langle z,  y \rangle)\,y \, d(P_R)_{\#}\mu(y).
\end{align*}
 As a result of the quotient rule, $\gamma^R$ in \eqref{eqn:gammaExt} is G\^ateaux differentiable with respect to $z$, since $\gamma^R_2$ is bounded away from zero by $exp(-2R_1(R\vee1))$ as $\Vert P_R(y)\Vert\le 2(R\vee 1)$ for every $y\in \R^d$, whose derivative is given by
\begin{align}\label{eqn:Dzgamma}
    D_z\gamma^R(z,\mu)
=  \iint y_1 (y_1 - y_2)^T\, d\mathcal{A}_{z, \mu}(y_1)d\mathcal{A}_{z, \mu}(y_2).
\end{align}
The Radon-Nikodym derivative that defines $\mathcal A$ is continuous and bounded for $z$ ranging over compact sets, and so $\gamma^R$ is Fr\'echet differentiable.
Therefore, by the Mean Value Theorem, it suffices to bound the operator norm of the total derivative $D_z\gamma^R$ to find the Lipschitz constant for $\gamma^R$ in $z$, which we do according to 
\begin{align*}
    \Vert \partial_z\gamma^R\Vert_{\mathrm{op}}
    & \leq \iint \Vert y_1 -y_2\Vert \Vert y_1\Vert \; d\mathcal{A}_{z, \mu}(y_1)d\mathcal{A}_{z, \mu}(y_2)  \le 2R_1^2.
\end{align*}
The final inequality uses the fact that $supp \;\mathcal{A}$ is contained in $\bar B(R_1)$ since that of $\mu$ is also.
\end{proof}
\begin{lemma}
Denoting Lions' derivative by $\partial_{\mu}$, $\gamma^R$ is Lions differentiable at any $y,z\in \R^d$, and $\mu \in \mathcal{P}_2(\mathbb{R}^d)$ with derivative given by
    \begin{align*}
        \partial_{\mu} \gamma^R(z,\mu)(y) = &\frac{\exp(\langle z,P_R(y)\rangle)}{\gamma^R_{2}(z,\mu)} \left[(P_R(y)-\gamma^R(z,\mu))z^T+I \right]D P_R(y),
    \end{align*}
    where $DP_R$ is the Fr\'echet derivative of $P_R$ and $\gamma_2^R$ is the normalisation constant
    \begin{equation}\label{eqn:Z}
    \gamma_{2}^R(z,\mu) = \int \exp(\langle z, P_R(y)\rangle) \,d\mu(y).
\end{equation}
\end{lemma}
\begin{proof}
Let $(\widetilde{\Omega}, \widetilde{\mathcal{F}},\widetilde{\mathbb{P}})$ be an atomless Polish probability space where there exists a random variable $Y \in L^2(\widetilde{\Omega},\widetilde{\mathcal{F}},\widetilde{\mathbb{P}}; \mathbb{R}^d)$ such that $\mathcal{L}(Y) = \mu$. We lift $\gamma^R_{1}$, as defined in the proof of Lemma \ref{lemma:GammaLipz}, and $\gamma_{2}^R$ to $L^2(\widetilde{\Omega},\widetilde{\mathcal{F}},\widetilde{\mathbb{P}}; \mathbb{R}^d)$ by
\begin{align*}
    \hat{\gamma}_{1}^R(z,Y) &:= \widetilde{\mathbb{E}}[\exp(z^TP_R(Y))P_R(Y)],\\
  \hat \gamma_{2}^R(z,Y) &:= \widetilde{\mathbb{E}}[\exp(z^TP_R(Y))].
\end{align*}
Take $H \in L^2(\widetilde \Omega, \widetilde{\mathcal{F}}, \widetilde{\mathbb P};\R^d)$ and $\eta>0$. Consider the perturbation of $\hat \gamma_{1}^R$ at $Y$ in the direction of $H$, given by
\begin{align*}
    \frac{\hat\gamma_{1}^R(z,Y+\eta H)-\hat\gamma_{1}^R(z,Y)}{\eta}&= \widetilde{\mathbb{E}}\left[ \exp( z^TP_R(Y+\eta H) )\frac{P_R(Y+\eta H)-P_R(Y)}{\eta} \right]\\
    &+ \widetilde{\mathbb{E}} \left[ \exp( z^TP_R(Y))\frac{\exp(\langle z,P_R(Y+\eta H)-P_R(Y)\rangle)-1}{\eta}P_R(Y)\right].
\end{align*}
Letting $\eta \to 0$ and applying the Dominated Convergence Theorem to interchange the order of limit and expectation, we obtain
\begin{align}\label{eqn:DGamma^n_ext}
 \frac{\hat \gamma_{1}^R(z,Y+\eta H)-\hat \gamma_{1}^R(z,Y)}{\eta}\to \widetilde{\mathbb{E}}\left[ \exp( z^TP_R(Y) )\left\{ D P_R(Y)+P_R(Y)z^TDP_R(Y)\right\}H\right].
\end{align}
Therefore, we deduce that the G\^ateaux derivative $D\tilde \gamma_{1}^R(z,Y)$ is given by the right-hand side of the above expression. The G\^ateaux derivative coincides with the Fr\'echet derivative when $D\hat \gamma_{1}^R(z,Y)$, viewed as a linear map from $L^2(\widetilde \Omega)$ to $L^2(\widetilde \Omega)$, is continuous in $Y$. Clearly, $D\tilde \gamma_{1}^R(z,Y)$ given on the right hand side of equation \eqref{eqn:DGamma^n_ext} is continuous in $Y$, and so the Lions derivative of $\gamma_{1}^R$ is given by
\[
\partial_\mu \gamma_{1}^R(z,\mu)(y)= \exp(\langle z,P_R(y)\rangle) \left\{P_R(y)z^T+I \right\}DP_R(y)
.\]
In identical fashion, we find
\[
\partial_\mu \gamma_{2}^R(z,\mu)(y) = \exp(\langle z,P_R(y)\rangle)D P_R(y)^Tz.
\]
Since $\gamma_{2}$ is bounded away from zero, as seen in the previous proof of $D_z \gamma^R$, we can apply the quotient rule on $\gamma^R$ to find
\[
\partial_\mu \gamma^R(z,\mu)(y)= \frac{\exp(\langle z,P(y)\rangle )}{\gamma^R_2(z,\mu)} \left\{(P_R(y)-\gamma^R(z,\mu))z^T+I \right\}D P_R(y).
\]
\end{proof}
\begin{corollary} 
Take any $z \in \mathbb{R}^d$ and $p\in [1,\infty]$. Then, for any $\mu,\tilde{\mu}\in \mathcal{P}_p(\bar B(R))$, there exists $\Lambda_{\mu,p}: \mathbb R_{\ge0}\to \mathbb R_{\ge 0}$  such that
    \[ \Vert \gamma^R(z,\mu)-\gamma^R(z,\tilde{\mu})\Vert \le \Lambda_{\mu,p}(\Vert z \Vert) W_p(\mu, \tilde{\mu}), \]
    where $\Lambda_{\mu,p}$ is given by
$
\Lambda_{\mu,p}(\tilde R) =(1+2R\tilde R)\exp(2R\tilde Rp^{-1})
$.
\end{corollary}
\begin{proof}
The assumptions in Lemma \ref{ProbMeasureMVT} are satisfied by $\gamma^R$ as a result of Lemma \ref{lemma:Lderivative}. Hence, for every $u \in \mathbb S^{d-1}$, there exists $t\in(0,1)$ and random variables $Y,\tilde{Y}$ on some atomless Polish probability space $(\widetilde \Omega, \widetilde{\mathcal F}, \widetilde{\mathbb P})$ with laws $\mu,\tilde{\mu}$ respectively such that
\begin{align}\label{eqn:lipmu1}
u^T\left( \gamma^R(z,\mu)-\gamma^R(z,\tilde{\mu})\right)  &=  \widetilde{\mathbb{E}}[u^T\partial_{\mu}\gamma^R(z,\mathcal{L}(Y_t))(Y_t)(Y-\tilde{Y})]  \\
& \leq \Vert \Vert \partial_{\mu}\gamma^R(z,\mathcal{L}(Y_t))(Y_t)\Vert_{op}\Vert_{L^q(\widetilde \Omega)}\Vert Y-\tilde{Y}\Vert_{L^p(\widetilde \Omega)},\notag
\end{align}
where $Y_t=\tilde{Y} + t(Y-\tilde{Y})$ and we have used $\Vert u\Vert\le 1$ . Since $Y,\tilde{Y}$ can be any random variables with laws $\mu,\tilde{\mu}$ taking values in a Polish space, we can choose the coupling between $Y,\tilde Y $ that minimises $\Vert Y-\tilde{Y}\Vert_{L^p(\widetilde \Omega)}$. Under this coupling $\Vert Y-\tilde{Y}\Vert_{L^p(\widetilde \Omega)}$ attains the $p$-Wasserstein distance between $\mu,\tilde{\mu}$. Thus, it suffices to prove that $\Vert \partial_{\mu}\gamma^R\Vert_{op}$ is bounded in $L^q(\widetilde \Omega)$ to prove our claim. Since the supports of $\mu,\tilde \mu$ are contained in $\bar B(R)$, it follows that $Y_t\in \bar B(R)$ $ \widetilde{\mathbb P}$-a.s. Therefore, we have that $P_R(Y_t)=Y_t$ and $DP_R(Y_t)=I$, $\widetilde{\mathbb P}$-almost surely. Furthermore, $\gamma^R(x,\mathcal L(Y_t))$ is bounded $\widetilde{\mathbb P}$-a.s. by $R$, which  yields that
\begin{align*}
&\Vert\partial_{\mu}\gamma^R(z,\mathcal{L}(Y_t))(Y_t)\Vert^q_{op}\le \frac{\exp(\langle z,Y_t\rangle)}{\widetilde{\mathbb E}[\exp(\langle z,Y_t\rangle)]}(1 + 2R\Vert z\Vert)^q\cdot \left\Vert\frac{\exp(\langle z,Y_t\rangle)}{\widetilde{\mathbb E}[\exp(\langle z,Y_t\rangle)]}\right\Vert_{L^{\infty}(\widetilde \Omega)}^{q-1}\cdot \Vert DP_R(Y_t)\Vert^q_{\mathrm{op}}.
\end{align*}
holds $\widetilde{\mathbb P}$-almost surely. By taking the expectation with resepect to $\widetilde{\mathbb P}$ and applying $\Vert Y_t\Vert \le R$ $\widetilde{\mathbb P}$-a.s., we deduce that
\begin{align*}
\Vert\Vert\partial_{\mu}\gamma\Vert_{op}\Vert_{L^q(\widetilde \Omega)} \le (1+2R\Vert z\Vert)exp\left[2(1-q^{-1})R\Vert z\Vert\right]
\end{align*}
The result concludes by applying the bound on $\Vert\partial_\mu\gamma\Vert_\mathrm{op}$ derived above to \eqref{eqn:lipmu1}, then taking the supremum over $u\in \mathbb S^{d-1}$, and noting that $1-q^{-1}=p^{-1}$. 
\end{proof}
    \section{Bounds \& Lipschitz continuity of $\Gamma,\mathcal K$}\label{App:D}
Let $R_1>0$. Throughout this section, we write $\partial_\mu\gamma(z,\mu)(y)$ for the Lions derivative of $\gamma^R$ at $(\mu,y)\in \mathcal P(\bar B(R_1))\times \bar B(R_1)$ with $R>R_1$ so that the map $P_R$ in the definition of $\gamma^R$ acts trivially.

\begin{lemma}\label{lemma:2ndderivatives}
Given $\gamma: \R^d\times \mathcal P_2(\R^d)\to \R^d$ defined in \eqref{eqn:gamma} and radius $R_1>0$, for any\\ $y\in \bar B(R_1), \mu \in\mathcal P(\bar B(R_1))$ and $z,\Delta z\in \R^d$, we have
    \begin{align*}
                \left\Vert D_{zz}\gamma(z,\mu)\right\Vert_{\mathrm{HS}} &\le 8 R_1^3, \\ 
        \left\Vert \sum_{k=1}^d\partial_{z_k}\partial_\mu\gamma(z,\mu)(y)(\Delta z)_k\right\Vert_\mathrm{op}&\le 2R_1(2+3R_1\Vert z\Vert)\frac{\exp(z^Ty)}{\gamma_2(z,\mu)}\Vert \Delta z \Vert.
    \end{align*}
\end{lemma}
\begin{remark}
    We contract the third order tensor $\partial_z\partial_\mu\gamma$ in its $z$-derivative component against a test vector $\Delta z\in \R^d$ to make the operator norm on a third order tensor explicit.
\end{remark}
\begin{proof}
Recall the expression for $D_z\gamma$ given in equation \eqref{eqn:Dzgamma} and expand the definition of measure $\mathcal A$ to give
\begin{equation}\label{eqn:Dzgamma2}
    D_z\gamma(z,\mu) = \frac{\iint \exp(\langle z,y_1+y_2\rangle)\, y_1(y_1-y_2)^T\, d\mu(y_1)d\mu(y_2)}{(\gamma_2(z,\mu))^2}.
\end{equation}
Integration and differentiation can be interchanged by the Dominated Convergence Theorem. Thus, the quotient rule gives
\[
(D_{zz}\gamma(z,\mu))_{i,j,k} =\iiint (y_1)_i(y_1-y_2)_j\left[y_1+y_2 -2y_3\right]_k\,d\mathcal A_{z,\mu}(y_1)d\mathcal A_{z,\mu}(y_2)d\mathcal A_{z,\mu}(y_3),
\]
after applying the identity $\partial_z\gamma_2(z,\mu) = \gamma(z,\mu)\gamma_2(z,\mu)$ 
to the term that results from taking the derivative to the denominator of $D_z\gamma$. The bound on the Hilbert-Schmidt norm of $D_{zz}\gamma$ holds since the support of the probability measure $\mathcal A_{z,\mu}$ is contained in $\bar B(R_1)$.\\

 Recall from Lemma \ref{lemma:Lderivative} that, for any  $y \in \bar B(R_1)$ and $\mu\in \mathcal P(\bar B(R_1))$,  $\partial_\mu\gamma$ is given by
 \begin{equation}\label{eqn:Dmugamma}
               \partial_\mu \gamma(z,\mu)(y) = \frac{\exp(z^Ty)}{\gamma_2(z,\mu)}\left[  \{y-\gamma(z,\mu)\} z^T+I\right]. 
 \end{equation}
By the quotient and product rule for differentiation,  we get
\begin{align*}
    \partial_{z_k}(\partial_{\mu}\gamma(z,\mu)(y))_{i,j}& = \frac{\exp(z^Ty)}{\gamma_2(z,\mu)}\left\{(y-\gamma(z,\mu))_k\left[(y-\gamma(z,\mu))_iz_j+\delta_{ij}\right] \right\}\\&-\frac{\exp(z^Ty)}{\gamma_2(z,\mu)}\left\{(D_z\gamma(z,\mu))_{i,k}z_j+\delta_{jk}(\gamma(z,\mu)-y)_i\right\},
\end{align*}
where $\delta_{ij}$ takes the value one if $i$ equals $j$ and is otherwise zero. Consequently, contract the third-order tensor $\partial_z\partial_\mu\gamma$ against $\Delta z$ to give the matrix
\begin{align*}
    \sum_{k=1}^d\partial_{z_k}\partial_{\mu}\gamma(z,\mu)(y)(\Delta z)_k&= \frac{\exp(z^Ty)}{\gamma_2(z,\mu)}\left( (y-\gamma(z,\mu))z^T+I\right)(\Delta z^T(y-\gamma(z,\mu)))\\&-\frac{\exp(z^Ty)}{\gamma_2(z,\mu)}\left( D_z\gamma(z,\mu)\Delta z\, z^T +(\gamma(z,\mu)-y)(\Delta z)^T\right).
\end{align*}
The claim on the operator norm of $\partial_z\partial_\mu\gamma$ holds since $y,\gamma(z,\mu)$ are bounded by $R_1$, and $\Vert D_z\gamma\Vert_\mathrm{op}$ is bounded above by $2R_1^2$ according to Lemma \ref{lemma:GammaLipz}.
\end{proof}
\begin{corollary}\label{cor:LipGammaDerivs}
Given $R_1,R_z>0$, there exists a constant $\Lambda_{D\gamma,p}=\Lambda_{D\gamma,p}(R_1,R_z)$ such that for any $y_i\in \bar B(R_1), z_i\in \bar B(R_z)$ and $\mathfrak{m}_i\in \mathcal P(\bar B(R_1))$ for $i=1,2$, we have that
\begin{align*}
    \Vert D_z\gamma(z_1,\mathfrak{m}_1)-D_z\gamma(z_2,\mathfrak{m}_2)\Vert_\mathrm{op}&\le \Lambda_{D\gamma,p} \left(\Vert z_1-z_2\Vert +W_p(\mathfrak{m}_1,\mathfrak{m}_2) \right),\\
    \Vert \partial_\mu\gamma(z_1,\mathfrak{m}_1)(y_2)-\partial_\mu\gamma(z_2,\mathfrak{m}_2)(y_2)\Vert_\mathrm{op}&\le\Lambda_{D\gamma,p} \left(\Vert y_1-y_2\Vert+\Vert z_1-z_2\Vert +W_p(\mathfrak{m}_1,\mathfrak{m}_2) \right).
\end{align*}
\end{corollary}
\begin{proof}
Let $(\widetilde{\Omega},\widetilde{\mathcal F},\widetilde{\mathbb P})$ be an atomless Polish probability space. Then, for any $u,v \in \mathbb S^{d-1}$, $t\in(0,1)$ and $Y_i\in L^2(\widetilde{\Omega},\widetilde{\mathcal F},\widetilde{\mathbb P})$ with $\mathcal L(Y_i)=\mathfrak{m}_i$ for $i=1,2$, the Mean Value Theorem given in Lemma \ref{ProbMeasureMVT} implies that 
\[
u^T\left(D_{z}\gamma(z_2,\mathfrak{m}_1)-D_{z}\gamma(z_2,\mathfrak{m}_2)\right)v = \sum_{i=1}^d\widetilde{\mathbb E} \left[u_i\, v^T\partial_{z_i}\partial_\mu \gamma(z_2,\mathcal L(Y_t))(Y_t)(Y_1-Y_2) \right] ,
\]
where $Y_t=Y_2+t(Y_1-Y_2)$. Therefore, H\"older's inequality implies that for any $p,q\in [1,\infty]$ such that $p^{-1}+q^{-1}=1$, we have 
\begin{align*}
   u^T\left( D_{z}\gamma(z_2,\mathfrak{m}_1)-D_{z}\gamma(z_2,\mathfrak{m}_2)\right)v = \left\Vert \left\Vert \sum_{i=1}^du_i \partial_{z_i}\partial_\mu \gamma(z_2,\mathcal L(Y_t))(Y_t)\right\Vert_\mathrm{op}\right\Vert_{L^q(\widetilde{\Omega})} \Vert Y_1-Y_2\Vert_{L^p(\widetilde{\Omega})}.
\end{align*}
Hence, by Lemma \ref{lemma:2ndderivatives} and since $\Vert u \Vert =1$,
\begin{align*}
    u^T\left( D_{z}\gamma(z_2,\mathfrak{m}_1)-D_{z}\gamma(z_2,\mathfrak{m}_2)\right)v = 2R_1(2+3R_1R_z)\left\Vert \frac{\exp(z^TY_t)}{\gamma_2\left(z,\mathcal L(Y_t)\right)}\right\Vert_{L^q(\widetilde{\Omega})} \Vert  Y_2-Y_1\Vert_{L^p(\widetilde{\Omega})}.
\end{align*}
Evaluating the above expression at the coupling between $Y_1,Y_2$ that attains the $p$-Wasserstein distance,  we obtain
\begin{equation*}
    u^T\left(D_z\gamma(z_2,\mathfrak{m}_1)-D_{z}\gamma(z_2,\mathfrak{m}_2)\right)v\le 3R_1(1+2R_1R_z) \exp(2R_zR_1p^{-1}) W_p(\mathfrak{m}_1,\mathfrak{m}_2),
\end{equation*}
 where we have used $\Vert Y_t\Vert \le R_1 $ $\widetilde{\mathbb P}$-a.s. and  $1-q^{-1}=p^{-1}$ to produce the bound $$\left\Vert \frac{\exp(z^TY_t)}{\gamma_2(z,\mathcal L(Y_t))}\right\Vert_{L^q(\widetilde{\Omega})}=\left(\widetilde{\mathbb E}\left[\frac{\exp(z^TY_t)}{\widetilde{\mathbb E}\left[\exp(z^TY_t)\right]}\cdot \frac{\exp((q-1)z^TY_t)}{(\gamma_2(z,\mathcal L(Y_t)))^{q-1}}\right]\right)^{1/q}\le \exp\left(2R_1R_zp^{-1}\right).$$ 
 The preceding bound on the difference between $D_z\gamma(z_2,\mathfrak{m}_1)$ and $D_z\gamma(z_2,\mathfrak{m}_2)$ holds for every $u,v\in\mathbb S^{d-1}$ and thus it holds for the supremum over $u,v\in\mathbb S^{d-1}$. This supremum attains the operator norm, which implies that
\begin{equation}\label{eqn:LipDzgamma}
   \Vert D_z\gamma(z_2,\mathfrak{m}_1)-D_{z}\gamma(z_2,\mathfrak{m}_2)\Vert_\mathrm{op}\le 3R_1(1+2R_1R_z) \exp(2R_zR_1p^{-1}) W_p(\mathfrak{m}_1,\mathfrak{m}_2).
\end{equation}
 Automatically from Lemma \ref{lemma:2ndderivatives}, $D_z\gamma$ is $8R_1^3$-Lipschitz continuous in $z$ on $\bar B(R_1)\times \mathcal P(\bar B(R_1))$. Therefore, $D_z\gamma$ is Lipschitz continuous on $\bar B(R_1)\times \mathcal P(\bar B(R_1))$ with Lipschitz constant that is the maximum of $8R_1^3$ and the constant in equation \eqref{eqn:LipDzgamma}.

Now, consider finding the local Lipschitz constant of $\partial_\mu\gamma$. By the Mean Value Theorem on $\R$, for every $u,v\in \mathbb S^{d-1}$, there exists $t\in(0,1)$ with $z_t=z_1+t(z_2-z_1)$ such that
\[
u^T\left( \partial_\mu\gamma(z_1,\mathfrak{m}_1)(y_1)-\partial_\mu\gamma(z_2,\mathfrak{m}_1)(y_1)\right)v\le \left\Vert \sum_{i=1}^d (z_1-z_2)_i \partial_{z_i}\partial_\mu \gamma(z_t,\mathfrak{m}_1)(y_1) \right\Vert_\mathrm{op}.
\]
Hence, applying Lemma \ref{lemma:2ndderivatives} and then taking the supremum over $u,v \in \mathbb S^{d-1}$, we obtain
\begin{equation}\label{eqn:LipDmugamma3}
  \Vert \partial_\mu\gamma(z_1,\mathfrak{m}_1)(y_1)-\partial_\mu\gamma(z_2,\mathfrak{m}_1)(y_1)\Vert_\mathrm{op}\le 2R_1(2+3R_1R_z) \exp(2R_zR_1) \Vert z_1-z_2\Vert,   
\end{equation}
since for any $z_1,z_2\in \bar B(R_z)$ so too is $z_t\in \bar B(R_z)$. For the Lipschitz continuity of $\partial_\mu \gamma$ in measure and $y$, we use the expression for $\partial_\mu \gamma$ given in  \eqref{eqn:Dmugamma} and add and subtract suitable terms to get 
\begin{align}
    \Vert \partial_\mu\gamma(z_2,\mathfrak{m}_1)(y_1)-\partial_\mu\gamma(z_2,\mathfrak{m}_2)(y_2)\Vert_\mathrm{op}&\le \frac{\Vert z_2\Vert}{\gamma_2(z_2,\mathfrak{m}_1)}\left\Vert \exp(z_2^Ty_1)\,y_1-\exp(z_2^Ty_2)\,y_2\right\Vert\notag\\&+ \frac{\vert \exp(z_2^Ty_1)-\exp(z_2^Ty_2)\vert}{\gamma_2(z_2,\mathfrak{m}_1)}\notag\\&+\left\vert \frac{1}{\gamma_2(z_2,\mathfrak{m}_1)}-\frac{1}{\gamma_2(z_2,\mathfrak{m}_2)}\right\vert\exp(z_2^Ty_2)\left(1+2R_1R_z \right)\notag\\
    &+\frac{\exp(z_2^Ty_2)}{\gamma_2(z_2,\mathfrak{m}_2)}\Vert \gamma(z_2,\mathfrak{m}_1)-\gamma(z_2,\mathfrak{m}_2)\Vert \Vert z_2\Vert,\label{eqn:LipDmugamma}
\end{align}
Since $y_1,y_2\in \bar B(R_1)$, $z_2\in \bar B(R_z)$ and $\mathfrak{m}_1,\mathfrak{m}_2 \in \mathcal P(\bar B(R_1))$, the first term in the expression above is bounded via
\[
\frac{\Vert z_2\Vert}{\gamma_2(z_2,\mathfrak{m}_1)}\left\Vert \exp(z_2^Ty_1)\,y_1-\exp(z_2^Ty_2)\,y_2\right\Vert\le R_z\exp(2R_1R_z)\left(1+R_1R_z \right)\Vert y_1-y_2\Vert.
\]
Similarly, the second term is bounded by
\[
\frac{\vert \exp(z_2^Ty_1)-\exp(z_2^Ty_2)\vert}{\gamma_2(z_2,\mathfrak{m}_1)}\le \exp(2 R_1R_z)R_z \Vert y_1-y_2\Vert.
\]
For the third term, take any coupling $\bm{m}\in \Pi(\mathfrak{m}_1,\mathfrak{m}_2)$. By applying H\"older's inequality for any $p,q \in [1,\infty]$ with $p^{-1}+q^{-1}=1$, we obtain
\begin{align*}
    \left\vert \frac{1}{\gamma_2(z_2,\mathfrak{m}_1)}-\frac{1}{\gamma_2(z_2,\mathfrak{m}_2)}\right\vert&\le \frac{\int \exp(z_2^Ty_1)\vert 1-\exp(z_2^T(y_2-y_1))\vert\,d \bm{m}(y_1,y_2)}{\gamma_2(z_2,\mathfrak{m}_1)\gamma_2(z_2,\mathfrak{m}_2)}\\&\le\frac{\left(\int \exp(q z_2^Ty_1)d\mathfrak{m}_1\right)^{1/q}}{\gamma_2(z_2,\mathfrak{m}_1)} \frac{\exp(2R_zR_1)R_z}{\gamma(z_2,\mathfrak{m}_2)}\left( \int \Vert y_1-y_2\Vert^p d\bm{m}(y_1,y_2)\right)^{1/p}.
\end{align*}
Therefore, by taking the coupling $\bm{m}\in \Pi(\mathfrak{m}_1,\mathfrak{m}_2)$ that attains $W_p(\mathfrak{m}_1,\mathfrak{m}_2)$ yields
\begin{equation}\label{eqn:Lipinversegamma2}
    \left\vert \frac{1}{\gamma_2(z_2,\mathfrak{m}_1)}-\frac{1}{\gamma_2(z_2,\mathfrak{m}_2)}\right\vert\le R_z\exp((3+2p^{-1}) R_zR_1)W_p(\mathfrak{m}_1,\mathfrak{m}_2),
\end{equation}
Finally, by applying $\Lambda_{\mu,p}$-Lipschitz continuous of $\gamma$ with respect to its measure argument, Corollary \ref{corr:LipMu}, to the final term of \eqref{eqn:LipDmugamma} and the preceding estimates to bound the other terms of \eqref{eqn:LipDmugamma}, we deduce that there exists a constant $\Lambda_{\gamma\mu,p}'=\Lambda_{\gamma\mu,p}'(R_1,R_z)$ such that
\begin{equation}\label{eqn:LipDmugamma2}
    \Vert \partial_\mu\gamma(z_2,\mathfrak{m}_1)(y_1)-\partial_\mu\gamma(z_2,\mathfrak{m}_2)(y_2)\Vert_\mathrm{op}\le\Lambda_{\gamma\mu,p}' \left(\Vert y_1-y_2 \Vert + W_p(\mathfrak{m}_1,\mathfrak{m}_2)\right).
\end{equation}
Hence, combining bounds \eqref{eqn:LipDmugamma3} and \eqref{eqn:LipDmugamma2} yields that $\partial_\mu \gamma$ is jointly Lipschitz continuous on \\$\bar B(R_z)\times \mathcal P(\bar B(R_1)) \times \bar B(R_1)$.
\end{proof}
\begin{lemma}\label{lemma:BoundedGamma}
Let $\mathfrak{S}(R)$ be defined by \ref{eqn:mathfrak_S}. Then, for $x,y \in \bar B(R_1)$, $\theta\in \mathfrak{S}(R_2)$, $a\in \bar B(R_3)$, $\mathfrak{m}\in \mathcal P(\bar B(R_1)), \mathfrak{r}\in \mathcal P(\bar B(R_1)\times \bar B(R_3))$ and $\mathfrak{n}\in \mathcal P(\mathfrak{S}(R_2))$, we have that
    \begin{align*}
    \Vert \Gamma(x,\mathfrak{m},\mathfrak{n})\Vert &\le R_1R_2^2,\qquad
\Vert \mathcal K(x,\mathfrak{r},\mathfrak{n},a)\Vert\le R_3\tilde{B}_\mathcal{K} .
    \end{align*}
where $\tilde{B}_\mathcal K$ is independent of $R_3$ and defined as
\[
 \tilde{B}_\mathcal K:=R_2^2\left(2\beta R_1^2R_2^2+(1+2\beta R_2^2R_1^2)\exp(2\beta R_1^2R_2^2) \right).
\] 
\end{lemma}
\begin{proof}
The integral triangle inequality gives
\begin{align*}
    \Vert \Gamma(x,\mathfrak{m}, \mathfrak{n})\Vert &\le \iint \frac{\exp(\beta\langle  \theta_Qx, \theta_K y \rangle)}{\gamma_2(\beta \theta_K^T\theta_Qx,\mathfrak{m})}\Vert \theta_O^T\theta_Vy\Vert d\mathfrak{m}(y)d\mathfrak{n}(\theta) \le R_2^2R_1.
\end{align*}
Recall that $\mathcal K$, defined in equation \eqref{eqn:mathcal_K}, can be decomposed into two term. For the first term $\nabla_x \mathcal H$, we have 
\begin{align}
    \Vert \nabla_x \mathcal H(x,\mathfrak{m}, \mathfrak{n},a)\Vert&\le \beta \Vert a\Vert \int \Vert \theta_O^T\theta_VD_x\gamma(\beta \theta_K^T\theta_Qx,\mathfrak{m})\theta_K^T\theta_Q\Vert_\mathrm{op} d\mathfrak{n}(\theta)\le 2\beta R_2^4 R_1^2R_3,\label{eqn:boundmathcal_K1}
\end{align}
where we have used the bound on the operator norm of $D_z\gamma$ given in Lemma \ref{lemma:GammaLipz}. Then, for the second term of $\mathcal K$, we have
\begin{align*}
    \left\Vert \int \partial_\mu \mathcal H(y,\mathfrak{m}, \mathfrak{n},p)(x)d\mathfrak{r}(y,p)\right\Vert&\le \iint \Vert \theta_O^T\theta_V\partial_\mu\gamma(\beta \theta_K^T\theta_Qy,\mathfrak{m})(x)\Vert_\mathrm{op}\Vert p\Vert d\mathfrak{n}(\theta)d\mathfrak{r}(y,p).
\end{align*}
Recall the expression for $\partial_\mu \gamma$ given in Lemma \ref{lemma:Lderivative}, which is bounded via
\begin{align}\label{eqn:bound_gamma_mu}
 \Vert \partial_\mu\gamma(\beta \theta_K^T\theta_Q y,\mathfrak{m})(x)\Vert_\mathrm{op}&\le  \frac{\exp(\beta \langle \theta_Qy,\theta_Kx\rangle)}{\gamma_2(\beta  \theta_K^T\theta_Qy,\mathfrak{m})}\{\Vert x-\gamma(\beta \theta_K^T\theta_Qy,\mathfrak{m})\Vert \Vert \beta\theta_K^T\theta_Qy\Vert+1\}  \\
 &\le (1+2\beta R_2^2R_1^2)\exp(2\beta R_1^2R_2^2)\notag
\end{align}
By combining the two preceding estimates, we get
\begin{align}
     \bigg\Vert \int \partial_\mu \mathcal H(y,\mathfrak{m}, \mathfrak{n},p)(x)d\mathfrak{r}(y,p)\bigg\Vert
    & \le  R_3 R_2^2(1+2\beta R_2^2R_1^2)\exp(2\beta R_1^2R_2^2)\label{eqn:boundmathcal_k2}
\end{align}
The result concludes by noting that $\mathcal K$ is bounded above by the sum of the terms on the left-hand sides of equations \eqref{eqn:boundmathcal_K1} and \eqref{eqn:boundmathcal_k2}, and so $\mathcal K$ is bounded above by the sum of the right-hand sides of these equations.
\end{proof}
\begin{lemma}\label{lemma:Lipderivs}
Let $\mathfrak{S}(R)$ be defined by \ref{eqn:mathfrak_S} and, for $i=1,2$, take $x_i\in \bar B(R_1)$, $\mathfrak{m}_i \in \mathcal P(\bar B(R_1))$ and $\mathfrak{n}_i\in \mathcal P(\mathfrak{S}(R_2))$ . Then, for every $p\ge 1$, there exists a constant $\Lambda_p= \Lambda_p(R_1,R_2)$ such that
    \begin{align}
        \Vert\Gamma(x_1,\mathfrak{m}_1,\mathfrak{n}_1)&-\Gamma(x_2,\mathfrak{m}_2,\mathfrak{n}_2)\Vert \le \Lambda_{p} \left(\Vert x_1-x_2\Vert + W_p(\mathfrak{m}_1,\mathfrak{m}_2)+ W_1(\mathfrak{n}_1,\mathfrak{n}_2) \right),\label{eqn:LipGamma}\\
        \Vert D_x\Gamma(x_1,\mathfrak{m}_1,\mathfrak{n}_1)&-D_x\Gamma (x_2,\mathfrak{m}_2,\mathfrak{n}_2)\Vert_\mathrm{F}  \label{eqn:LipDxGamma}\\&\le \Lambda_{p} \left(\Vert x_1-x_2\Vert + W_p(\mathfrak{m}_1,\mathfrak{m}_2)+ W_1(\mathfrak{n}_1,\mathfrak{n}_2) \right),\notag\\
        \Vert \partial_\mu\Gamma(x_1,\mathfrak{m}_1,\mathfrak{n}_1)(y_1)&-\partial_\mu\Gamma (x_2,\mathfrak{m}_2,\mathfrak{n}_2)(y_2)\Vert_\mathrm{F}\label{eqn:LipDmuGamma}\\&\le\Lambda_{p} \left(\Vert x_1-x_2\Vert_\mathrm{F} + \Vert y_1-y_2\Vert+ W_p(\mathfrak{m}_1,\mathfrak{m}_2)+ W_1(\mathfrak{n}_1,\mathfrak{n}_2) \right).\notag
    \end{align}
\end{lemma}
\begin{proof}
Applying the $2R_1^2$ and $\Lambda_{\mu,p}$-Lipchitz continuity of $\gamma$ with respect to the $z$ and $m$ variables, respectively, from Lemma \ref{lemma:GammaLipz} and Corollary \ref{corr:LipMu}, we deduce that
   \begin{align}\label{eqn:LipGamma1}
       \Vert\Gamma(x_1,\mathfrak{m}_1,\mathfrak{n}_1)&-\Gamma(x_2,\mathfrak{m}_2,\mathfrak{n}_1)\Vert\\& \le \int \Vert \theta_O^T\theta_V (\gamma(\beta \theta_K^T\theta_Qx_1,\mathfrak{m}_1)-\gamma(\beta \theta_K^T\theta_Qx_2,\mathfrak{m}_2))\Vert d\mathfrak{n}_1(\theta)\notag\\
       & \le R_2^2\left(2\beta R_2^2R_1^2\Vert x_1-x_2\Vert+ \Lambda_{\mu,p}(\beta R_1R_2^2) W_p(\mathfrak m_1,\mathfrak m_2)\right).\notag
   \end{align}
   Take any $\theta,\vartheta \in \mathfrak{S}(R_2)$ and write $A_1=\beta \theta_K^T\theta_Q,A_2=\beta \vartheta_K^T\vartheta_Q$,  $U_1=\theta_O^T\theta_V, U_2=\vartheta_O^T\vartheta_V$. Then, the $2R^2_1$-Lipschitz continuity of $\gamma$ in $z$ and boundedness of $\gamma$ implies that
    \begin{align}\label{eqn:LipGamma3}
      \left\Vert U_1 \gamma(A_1 x_2,\mathfrak{m}_2)-U_2\gamma(A_2x_2,\mathfrak{m}_2)\right\Vert 
      & \le R_1 \Vert U_1-U_2 \Vert_\mathrm{F}+2R^3_1R_2^2\Vert A_1-A_2\Vert_\mathrm{F}\\
      & \le 2R_1R_2(1\vee 2\beta R_1^2R_2^2) \Vert \theta-\vartheta\Vert_\mathrm{F}.\notag
    \end{align}
    The final inequality results from the Cauchy-Schwarz inequality, which gives \\$\Vert U_1-U_2\Vert_\mathrm{F}\le R_2(\Vert \theta_O-\vartheta_U\Vert_\mathrm{F}+\Vert \theta_V-\vartheta_V\Vert_\mathrm{F})$, similalry for $\Vert A_1-A_2\Vert_\mathrm{F}$. Then, applying $v_1+\ldots+ v_4\le 2\sqrt{v_1^2+\ldots+v_4^2}$ to write the bound in terms of $\vartheta,\theta$ rather than its blocks. By integrating equation \eqref{eqn:LipGamma3} with respect to a $\bm \nu\in \Pi(\mathfrak{n}_1,\mathfrak{n}_2)$ that attains the $1$-Wasserstein distance between $\mathfrak{n}_1,\mathfrak{n}_2$, we get
    \begin{align}\label{eqn:LipGamma2}
\Vert\Gamma(x_2,\mathfrak{m}_2,\mathfrak{n}_1)-\Gamma(x_2,\mathfrak{m}_2,\mathfrak{n}_2)\Vert& \le \Lambda_\Gamma'W_1 (\mathfrak{n}_1,\mathfrak{n}_2),
    \end{align}
    where $\Lambda_\Gamma'=2R_1R_2(1\vee 2\beta R_1^2R_2^2)$.
    Therefore, claim \eqref{eqn:LipGamma} is deduced by adding the estimates given in equations \eqref{eqn:LipGamma1} and \eqref{eqn:LipGamma2}. In Corollary \ref{cor:LipGammaDerivs}, we deduced that $D_z\gamma$ is $\Lambda_{D\gamma,p}$-locally Lipchitz continuous with respect to both variables. Therefore, by applying this estimate with $R_z=\beta R_1R_2^2$, we obtain
    \begin{align}\label{eqn:LipDzGamma1}
        \Vert D_x\Gamma&(x_1,\mathfrak{m}_1,\mathfrak{n}_1)- D_x\Gamma(x_2,\mathfrak{m}_2,\mathfrak{n}_1) \Vert_\mathrm{F}\\&\le\beta \int \left \Vert\theta_O^T\theta_V\left(D_z\gamma(\beta \theta_K^T\theta_Q x_1,\mathfrak{m}_1)-(D_z\gamma(\beta \theta_K^T\theta_Q x_2,\mathfrak{m}_2) \right)\theta_K^T\theta_Q\right\Vert_\mathrm{F}d\mathfrak{n}_1(\theta)\notag\\
        & \le \beta R_2^4 \left( \beta R_2^2 \Lambda_{D\gamma,p}\Vert x_1-x_2\Vert+ \Lambda_{D\gamma,p}W_p(\mathfrak{m}_1,\mathfrak{m}_2)\right).\notag
    \end{align}
    Recall $U_1,U_2,A_1,A_2$ defined  in the argument for the Lipschitz continuity of $\Gamma$. After applying the $8R_1^3$ local Lipschitz continuity of $D_z\gamma$ in $z$,  Lemma \ref{lemma:2ndderivatives}, and the fact that $D_z\gamma$ is bounded by $2R_1^2$, see Lemma \ref{lemma:GammaLipz}, we get
    \begin{align}\label{eqn:LipDzgammanu}
        \Vert U_1D_z\gamma(A_1x_2,\mathfrak{m}_2)A_1&-U_2D_z\gamma(A_2x_2,\mathfrak{m}_2)A_2\Vert_\mathrm{F} \\
          & \le 2 R^2_1R_2^2 \left(\beta\Vert U_1-U_2\Vert_\mathrm{F}+\Vert A_1-A_2\Vert_\mathrm{F}\right)+8\beta R_1^4R_2^4\Vert A_1-A_2\Vert_\mathrm{F}\notag\\&\le 2 \beta R^2_1R_2^3 \left(\Vert \theta_O-\vartheta_U\Vert_\mathrm{F}+ \Vert \theta_V-\vartheta_V\Vert_\mathrm{F}+\Vert \theta_Q-\vartheta_Q\Vert_\mathrm{F}+\Vert \theta_K-\vartheta_K\Vert_\mathrm{F}\right)\notag\\&+8\beta^2R_1^4R_2^5\left(\Vert \theta_Q-\vartheta_Q\Vert_\mathrm{F}+ \Vert \theta_K-\vartheta_K\Vert_\mathrm{F}\right).\notag
    \end{align}
    The final inequality results from the Cauchy-Schwarz inequality and expanding the definition of $U_i,A_i$. Take any $\bm\nu\in \Pi(\mathfrak{n}_1,\mathfrak{n}_2)$ that attains the $1$-Wasserstein distance $W_1(\mathfrak{n}_1,\mathfrak{n}_2)$.  Since the sum of the Fr\"obenius norms on the blocks of parameters corresponding to the $Q,K,V,O$ matrices is bounded above by twice the Fr\"obenius norm between $\theta,\vartheta\in \R^{k\times4d}$, integrating the expression given on the left-hand side of equation \eqref{eqn:LipDzgammanu} with respect to $\bm{\nu}$ gives
    \[
    \Vert D_x\Gamma(x_2,\mathfrak{m}_2,\mathfrak{n}_1)-D_x\Gamma(x_2,\mathfrak{m}_2,\mathfrak{n}_2)\Vert_\mathrm{F}\le 4R_1^2R_2^3\beta(1+4\beta R_1^2R_2^2) W_1(\mathfrak{n}_1,\mathfrak{n}_2).
    \]
    Hence,  Claim \eqref{eqn:LipDxGamma} is deduced by combining the above equation with \eqref{eqn:LipDzGamma1}.

    The final claim on the Lipschitz continuity of $\partial_\mu \Gamma$ follows identically, by replacing $D_z\gamma$ by $\partial_\mu\gamma$ in the above estimates and using the second result of Corollary \ref{cor:LipGammaDerivs}.
\end{proof}
\begin{lemma}\label{lemma:LipmathcalHnu}
Let $\mathfrak{S}(R)$ be defined by \ref{eqn:mathfrak_S} and, for $i=1,2$, suppose $x_i\in \bar B(R_1), a_i \in \bar B(R_3),$\\$ \mathfrak{m}_i\in \mathcal P(\bar B(R_1)), \mathfrak{n}_i\in \mathcal P(\mathfrak{S}(R_2))$ and  $\theta,\vartheta\in \mathfrak{S}(R_2)$. Then, for every $p\ge 1$, there exists a constant $\Lambda_{\mathcal H\nu,p}=\Lambda_{\mathcal H\nu,p}(\beta,R_1,R_2,R_3)$ such that 
    \begin{align*}
            \Vert \partial_\nu\mathcal H(x_1,\mathfrak{m}_1,\mathfrak{n}_1,a_1)(\theta)&-\partial_\nu\mathcal H(x_2,\mathfrak{m}_2,\mathfrak{n}_2,a_2)(\vartheta)\Vert_\mathrm{F}\\&\le \Lambda_{\mathcal H\nu,p}\left(\Vert x_1-x_2\Vert+ \Vert a_1-a_2\Vert+ W_p(\mathfrak{m}_1,\mathfrak{m}_2)+ \Vert \theta-\vartheta\Vert_\mathrm{F}\right).
    \end{align*}

\end{lemma}
\begin{proof}
    Denote the first variation (Definition 5.43 in \cite{carmona2018probabilistic} Vol. I) of $\mathcal H$ in $\nu$ by $\delta \mathcal H/\delta \nu$, which is given by
    \[
    \frac{\delta \mathcal H}{\delta \nu}(x,\mu,\nu,a)(\theta)=a^T\theta_O^T\theta_V\gamma(\beta\theta_K^T\theta_Qx,\mu).
    \]
   Lions' derivative $\partial_\nu\mathcal H$ coincides with the gradient of the first variation of $\mathcal H$, see Proposition 5.48 in \cite{carmona2018probabilistic} Volume I. Therefore, writing $A_1=\beta \theta_K^T\theta_Q$ and $A_2=\beta\vartheta_K^T\vartheta_Q$, the difference between the gradients corresponding to the $O$ parameters is bounded by
    \begin{align*}
        \bigg\Vert \nabla_{\theta_O}\frac{\delta\mathcal H}{\delta \nu}(x_1,\mathfrak{m}_1,\mathfrak{n}_1&,a_1)(\theta)- \nabla_{\theta_O}\frac{\delta\mathcal H}{\delta \nu}\mathcal H(x_2,\mathfrak{m}_2,\mathfrak{n}_2,a_2)(\vartheta)\bigg\Vert_\mathrm{F}\\& \le \Vert (\theta_{V}\gamma(A_1x_1,\mathfrak{m}_1))a_1^T-(\vartheta_{V}\gamma(A_2x_2,\mathfrak{m}_2))a_2^T\Vert_\mathrm{F}\\
        & \le R_1R_2 \Vert a_1-a_2\Vert+R_1R_3\Vert \theta_V-\vartheta_V\Vert_\mathrm{F}+ R_2R_3 \Lambda_{\mu,p}(\beta R_1R_2^2)W_p(\mathfrak{m}_1,\mathfrak{m}_2)\\&+2\beta R_3R_1^2\Vert \theta_K^T\theta_Q x_1-\vartheta_K^T\vartheta_Q x_2\Vert,
    \end{align*}
    having applied the Lipschitz estimates of $\gamma$ deduced in Appendix \ref{App:C}.
    Therefore, there exists a constant $\Lambda_{\mathcal H\nu,p}'$ that depends only on $\beta,R_1,R_2,R_3$ such that
   \begin{align*}
       \bigg\Vert \nabla_{\theta_O}\frac{\delta\mathcal H}{\delta \nu}(x_1,\mathfrak{m}_1,\mathfrak{n}_1,a_1)(\theta)&- \nabla_{\theta_O}\frac{\delta\mathcal H}{\delta \nu}\mathcal H(x_2,\mathfrak{m}_2,\mathfrak{n}_2,a_2)(\vartheta)\bigg\Vert_\mathrm{F}\\& \le\Lambda_{\mathcal H\nu,p}'\left(\Vert a_1-a_2\Vert +\Vert x_1-x_2\Vert +W_p(\mathfrak{m}_1,\mathfrak{m}_2)+ \Vert \theta-\vartheta\Vert_\mathrm{F}  \right).\notag
   \end{align*}
   Here, we have used the fact that the sum of $\Vert \theta_V-\vartheta_V\Vert_\mathrm{F}$ and that on the $Q,K$ parameters is bounded above by twice $\Vert \theta-\vartheta\Vert_\mathrm{F}$. The Lipschitz continuity of the gradient with respect to the $V$ parameters can be deduced and bounded identically. Moreover, the difference between the gradients corresponding to the $K$ parameters is bounded by
    \begin{align*}
        \bigg\Vert \nabla_{\theta_K}&\frac{\delta\mathcal H}{\delta \nu}(x_1,\mathfrak{m}_1,\mathfrak{n}_1,a_1)(\theta)- \nabla_{\theta_K}\frac{\delta\mathcal H}{\delta \nu}\mathcal H(x_2,\mathfrak{m}_2,\mathfrak{n}_2,a_2)(\vartheta)\bigg\Vert_\mathrm{F}\\& \le \beta\Vert\theta_{Q}x_1\, (\theta_{O}a_1)^T \theta_{V}D_z\gamma(A_1x_1,\mathfrak{m}_1)- \vartheta_{Q}x_2(\vartheta_{O}a_2)^T\vartheta_{V}D_z\gamma(A_2x_2,\mathfrak{m}_2)\Vert_\mathrm{F}.    
    \end{align*}
    Therefore, by using the boundedness of $D_z\gamma$ given in Lemma \ref{lemma:GammaLipz} and the Lipschitz continuity of $D_z\gamma$ shown in Corollary \ref{cor:LipGammaDerivs}, there exists a constant $\Lambda_{\mathcal H\nu,p}''$ that depends only on $\beta,R_1,R_2,R_3$ such that
    \begin{align*}
        \bigg\Vert \nabla_{\theta_K}\frac{\delta\mathcal H}{\delta \nu}(x_1,\mathfrak{m}_1,\mathfrak{n}_1,a_1)(\theta)&- \nabla_{\theta_K}\frac{\delta\mathcal H}{\delta \nu}\mathcal H(x_2,\mathfrak{m}_2,\mathfrak{n}_2,a_2)(\vartheta)\bigg\Vert_\mathrm{F}\\& \le \Lambda_{\mathcal H\nu,p}''\left(\Vert \theta-\vartheta\Vert_\mathrm{F}+\Vert x_1-x_2\Vert +\Vert a_1-a_2\Vert+W_p(\mathfrak{m}_1,\mathfrak{m}_2)\right).\notag
    \end{align*}
    The Lipschitz continuity of the gradient with respect to the $Q$ parameters can be deduced and bounded identically. Since the squared Fr\"obenius norm on $\R^{k\times 4d}$ is bounded above by the squared Fr\"obenius norm on each $\R^{k\times d}$ block corresponding to the $Q,K,V,O$ parameters separately, $\partial_\nu\mathcal H$ is locally Lipschitz continuous with Lipschitz constant $\sqrt{2}\sqrt{(\lambda_{\mathcal H\nu,p}')^2+(\lambda_{\mathcal H\nu,p}'')^2}
    $ that depends only on $\beta,R_1,R_2,R_3$.
\end{proof}
\begin{corollary}\label{cor:LipmathcalK}
Let $\mathfrak{S}(R)$ be defined by \ref{eqn:mathfrak_S} and, for $i=1,2$, suppose $x_i\in \bar B(R_1), a_i \in \bar B(R_3), \mathfrak{n}_i\in \mathcal P(\mathfrak{S}(R_2))$ and $\rho_i\in \mathcal P(\bar B(R_1)\times \bar B(R_3))$. Then, for every $p\ge 1$, there exists a constant \\$\Lambda_{\mathcal K,p}=\Lambda_{\mathcal K,p}(\beta,R_1,R_2,R_3)$ such that 
    \begin{align*}
        \Vert \mathcal K(x_1,\rho_1,\mathfrak{n}_1,a_1)-\mathcal K(x_2,\rho_2,\mathfrak{n}_2,a_2)\Vert \le \Lambda_{\mathcal{K},p}\left(\Vert x_1-x_2\Vert+\Vert a_1-a_2\Vert +W_p(\rho_1,\rho_2)+W_1(\mathfrak{n}_1,\mathfrak{n}_2)\right).
    \end{align*}
\end{corollary}
\begin{proof}
Since $a_1,a_2$ are bounded by $R_3$, and $D_x\Gamma$ is bounded by $2\beta R_1^4R_1^2$, immediate from equation \eqref{eqn:boundmathcal_K1}, and locally Lipschitz continuous, see Lemma \ref{lemma:Lipderivs}, there exists a constant $\Lambda_{\mathcal K,p}'=\Lambda_{\mathcal K,p}'(\beta,R_1,R_2,R_3)>0$ such that 
\begin{align*}
    \Vert \nabla_x \mathcal H&(x_1, \rho_1|_x,\mathfrak{n}_1,a_1)- \nabla_x \mathcal H(x_2, \rho_2|_x,\mathfrak{n}_2,a_2)\Vert\\ &\le  \Vert a_1-a_2\Vert \Vert D_x\Gamma(x_1,\mathfrak{m}_1,\mathfrak{n}_1)\Vert_\mathrm{F} +\Vert a_2\Vert \Vert D_x\Gamma(x_1,\rho_1|_x,\mathfrak{n}_1)-D_x\Gamma(x_2,\rho_2|_x,\mathfrak{n}_2)\Vert_\mathrm{F}\\
    & \le \Lambda_{\mathcal K,p}'\left(\Vert x_1-x_2\Vert+\Vert a_1-a_2\Vert +W_p(\rho_1|_x,\rho_2|_x)+W_1(\mathfrak{n}_1,\mathfrak{n}_2)\right).
\end{align*}
Notice that the discrepancies between the second terms in $\mathcal K$ can be decomposed as
\begin{align*}
    \bigg\Vert \int \partial_\mu \mathcal H&(y,\rho_1|_x,\mathfrak{n}_1,p)(x_1)d\rho_1(y,p)- \int \partial_\mu \mathcal H(y,\rho_2|_x,\mathfrak{n}_2,p)(x_2)d\rho_2(y,p) \bigg\Vert\\
    & \le \int \underbrace{\Vert p\Vert\left\Vert \partial_\mu \Gamma(y,\rho_1|_x,\mathfrak{n}_1)(x_1)- \partial_\mu \Gamma(y,\rho_2|_x,\mathfrak{n}_2)(x_2)\right\Vert_\mathrm{F}}_{(\dagger_1)} d\rho_1(y,p)\\
    & +\underbrace{\left\Vert \int \partial_\mu\mathcal H(y,\rho_2|_x,\mathfrak{n}_2,p)(x_2)d (\rho_1-\rho_2)(y,p)) \right\Vert}_{(\dagger_2)}.
\end{align*}
For the first term $(\dagger_1)$: by the local Lipschitz continuity of $\partial_\mu\Gamma$, Lemma \ref{lemma:Lipderivs}, and the assumed support of $\rho_1$ in $\bar B(R_1)\times \bar B(R_3)$, there exists a constant $\Lambda_{\mathcal K,p}''=\Lambda_{\mathcal K,p}''(\beta, R_1,R_2,R_3)$ such that $\rho_1$-a.s. we have
\[
(\dagger_1)\le \Lambda_{\mathcal K,p}''\left( \Vert x_1-x_2\Vert +W_{p}(\rho_1|_x,\rho_2|_x)+ W_1(\mathfrak{n}_1,\mathfrak{n}_2)\right).
\]
For the second term $(\dagger_2)$:
By combining the Lipschitz continuity of $\partial_\mu\Gamma$, Lemma \ref{lemma:Lipderivs}, and the boundedness of $\partial_\mu\gamma$, equation \eqref{eqn:bound_gamma_mu},
the map $(y,p)\mapsto \partial_\mu\mathcal H(y,\rho_2|_x,\mathfrak{n}_2,p)(x_2)$ is Lipchitz on $\bar B(R_1)\times\bar B(R_3)$. Therefore, the Kantorovich-Rubinstein Theorem, see \cite{villani2008optimal}, implies that there exists a constant $\Lambda_{\mathcal K}'''$ such that
\[
(\dagger_2)\le \Lambda_{\kappa}'''W_1(\rho_1,\rho_2).
\]
Recall the definition of $\mathcal K$ in \eqref{eqn:mathcal_K}, which decomposes into the sum of a term involving $\nabla_x\mathcal H$ and a term involving $\partial_\mu\mathcal H$. Its Lipschitz constant can be estimated by controlling each term separately to give 
\begin{align*}
  \Vert \mathcal K(x_1,\rho_1,\mathfrak{n}_1,a_1)&-\mathcal K(x_2,\rho_2,\mathfrak{n}_2,a_2)\Vert \\
  &\le \Vert \nabla_x \mathcal H(x_1, \mathfrak{m}_1,\mathfrak{n}_1,a_1)- \nabla_x \mathcal H(x_2, \mathfrak{m}_2,\mathfrak{n}_2,a_2)\Vert\\&+\left\Vert \int \partial_\mu \mathcal H(y,\rho_1|_x,\mathfrak{n}_1,p)(x_1)d\rho_1(y,p)- \int \partial_\mu \mathcal H(y,\rho_2|_x,\mathfrak{n}_2,p)(x_2)d\rho_2(y,p) \right\Vert\\
  & \le \Lambda_{\mathcal K,p}'\left(\Vert x_1-x_2\Vert+\Vert a_1-a_2\Vert +W_p(\rho_1|_x,\rho_2|_x)+W_1(\mathfrak{n}_1,\mathfrak{n}_2)\right) \\&+  \Lambda_{\mathcal K,p}''\left( \Vert x_1-x_2\Vert +W_{p}(\rho_1|_x,\rho_2|_x)+ W_1(\mathfrak{n}_1,\mathfrak{n}_2)\right)+\Lambda_{\kappa}'''W_1(\rho_1,\rho_2).
\end{align*}
Then, the desired result holds with $\Lambda_{\mathcal K}=(\Lambda_{\mathcal K,p}'+\Lambda_{\mathcal K,p}''+ \Lambda_{\mathcal K}''')$, since the $p$-Wasserstein distance is increasing in $p$ by Jensen's inequality and $W_p(\rho_1|_x,\rho_2|_x)\le W_p(\rho_1,\rho_2)$ as the projection map onto the $x$ marginal in $1$-Lipschitz.
\end{proof}

    \section{The Global Well-Posedness of the Continuous-Time Model}\label{App:E}
Let $R_2>0$ and $\bm\nu\in C([0,1],\mathfrak{S}(R_2))$, and consider the McKean--Vlasov ODE, with initial condition $\xi\in L^\infty(\Omega^1,\mathcal F^1,\mathbb P^1; \R^d)$ with $\Vert \xi\Vert \le R_0$ $\mathbb P^1$-a.s., associated to the IPS \eqref{eqn:ansatz} given by
\begin{align}
    \frac{dY^{\xi}_t}{dt} &= \Gamma(Y^\xi_t,\mathcal L^1(Y^\xi_t),\nu_t), \qquad &Y^\xi_0 =\xi,\label{eqn:MVODE1}\\
    \frac{d p^{\xi}_t}{dt} &= - \mathcal K(Y^\xi_t,\mathcal L^1(Y^\xi_t,p^\xi_t),\nu_t,p^\xi_t), \qquad &p^\xi_1 = \partial_\mu\ell(\mathcal L^1(Y^\xi_1))(Y^\xi_1),\label{eqn:MVODE2}
\end{align}
where $\mathcal L^1(Z):=\mathbb P^1\circ Z^{-1}$ denotes the law of a random variable $Z$ defined on $(\Omega^1,\mathcal F^1,\mathbb P^1)$. We search for solutions $(\bm{Y}^\xi,\bm{p}^\xi)$ taking values in $$\mathscr{C}^{1,2,2d}:=L^2(\Omega^1,\mathcal F^1,\mathbb P^1; C([0,1]; \R^{2d})),$$
equipped with the norm
\[
\Vert \bm{Z}\Vert^2_{\mathscr{C}^{1,2,d}}=\mathbb E^1\left[ \sup_{t\in[0,1]} \Vert Z_t\Vert^2 \right], \qquad \bm{Z}\in\mathscr{C}^{1,2,d} .
\]
Here, the superscript one on $\mathscr{C}^{1,2,d}$ indicates that the random variables are defined on $(\Omega^1, \mathcal F^1,\mathbb P^1)$. When this superscript is zero or omitted, $\mathscr{C}^{0,2,2d}$ and $\mathscr{C}^{2,2d}$ denote the corresponding spaces defined on $(\Omega^0,\mathcal F^0,\mathbb P^0)$ and $(\Omega,\mathcal F,\mathbb P)$ instead.

Since the velocity fields $\Gamma,\mathcal K$ are only locally Lipschitz continuous in $x$ and $\rho$, we truncate the velocity fields to restore global Lipschitz continuity. With this in mind, let us introduce
\[
\Gamma^R(x,\mu,\nu) = \Gamma(P_R(x), P_R{}_\#\mu,\nu), \qquad \mathcal K^R(x,\rho,\nu,a) = \mathcal K(P_R(x), P_{R,R}{}_\#\rho,\nu,P_R(a)),
\]
where $P_R$ is defined in equation \eqref{eqn:proj_map} and
\[
P_{R,R}((x,a))= (P_R(x),P_R(a)).
\]
Recall that $P_R$ is $2$-Lipschitz continuous and is bounded by the minimum of $2(R\vee 1)$ and $\Vert x\Vert$ for every $x\in\R^d$. Hence, with $x_i,a_i\in \R^d, \mu_i\in \mathcal P_2(\R^d)$, $\rho_i\in \mathcal P_2(\R^{2d})$ and $\nu \in\mathcal P(\mathfrak{S}(R_2))$ for $i=1,2$, the following Lipschitz bounds hold:
\begin{align}
    \Vert \Gamma^R(x_1,\mu_1,\nu)&-\Gamma^R(x_2,\mu_2,\nu)\Vert \le 2\Lambda^R_\Gamma \left(\Vert x_1-x_2\Vert + W_2(\mu_1,\mu_2) \right),\notag \\
    \Vert \mathcal K^R(x_1,\rho_1,\nu,a_1)&-\mathcal K^R(x_2,\rho_2,\nu,a_2)\Vert \le 2\Lambda^R_\Gamma \left(\Vert x_1-x_2\Vert+\Vert a_1-a_2\Vert + W_2(\rho_1,\rho_2) \right),\notag
\end{align}
where $\Lambda^R_\Gamma,\Lambda^R_{\mathcal K}$ are the constants from Lemma \ref{lemma:Lipderivs} and Corollary \ref{cor:LipmathcalK} evaluated at $p=2$, $R_2$ as specified and $R_1=R_3=2(R\vee1)$. We will see that the solution to equations \eqref{eqn:MVODE1}-\eqref{eqn:MVODE2} is $\mathbb P^1$-a.s  bounded, and so the truncation maps $P_R$ have no effect, if $R$ is taken sufficiently large. 
\begin{lemma}\label{lemma:MVODE_WellDEfined}
    Let $R_2>0$ and $\bm\nu\in C([0,1],\mathcal P(\mathfrak{S}(R_2)))$. Then, for every $\xi \in L^\infty(\Omega^1,\mathcal F^1,\mathbb P^1;\R^d)$ such that $\Vert \xi \Vert \le R_0$ $\mathbb P^1$-a.s., there exists a solution $(\bm Y^\xi,\bm p^\xi) \in\mathscr{C}^{1,2,2d}$ unique up to $\mathbb P^1$-a.s. equality, to the forward-backward system of McKean--Vlasov ODEs \eqref{eqn:MVODE1}-\eqref{eqn:MVODE2}. Furthermore, there exist constants $R_X=R_X(R_0,R_2)$ and $R_a=R_a(\beta,R_0,R_2,K)$ such that $Y^{\xi}_t\in \bar B(R_X)$ and \\$p^\xi_t\in \bar B(R_a)$  for every $t\in[0,1]$, $\mathbb P^1$-almost surely.
\end{lemma}
\begin{proof}
\textit{Step One: Existence and Uniqueness of $\bm{Y}^{\xi}$}\\
The mapping $(x,\mu,t)\mapsto \Gamma^R(x,\mu,\nu_t)$ satisfies Assumption \textbf{(MKV SDE)} in Section 4.2.1 of \cite{carmona2018probabilistic} Vol. I, with identically zero diffusion coefficient $\Sigma$. Therefore, Theorem 4.21 in \cite{carmona2018probabilistic} Vol. I applies and proves that there exists a unique solution $\bm{Y}^{\xi,R}=(Y^{\xi,R}_t)_{t\in[0,1]}\in \mathscr{C}^{1,2,d}$ to \eqref{eqn:MVODE1}, when $\Gamma$ is replaced by $\Gamma^R$. Then, applying Lemma \ref{lemma:BoundedGamma} to bound $\Gamma$ and the fact that $\Vert \xi\Vert \le R_0$ $\mathbb P^1$-a.s., we obtain that
\begin{align*}
    \Vert Y^{\xi,R}_t\Vert &\le \Vert \xi\Vert+ \int_0^t\Vert \Gamma(P_R(Y^{\xi,R}_s), P_R{}_\#\mathcal L^1(Y^{\xi,R}_s), \nu_s)\Vert ds
    \\&\le R_0+ R_2^2\int_0^tW_\infty(\delta_0, P_R{}_\#\mathcal L^1(Y^{\xi,R}_s))ds, \qquad \mathbb P^1-a.s.
\end{align*}
The right-hand side is deterministic, and so for every $q\ge 1$
\begin{align*}
      W_q(\delta_0,\mathcal L^1(Y^{\xi,R}_t))& \le \left(\mathbb E^1[\Vert Y^{\xi,R}_t\Vert^q]\right)^{1/q}\le R_0+ R_2^2\int_0^t W_\infty(\delta_0, P_R{}_\#\mathcal L^1(Y^{\xi,R}_s))ds.
\end{align*}
Since $\Vert P_R(x)\Vert \le \Vert x\Vert$ for every $x\in \R^d$, for every $s\in[0,1]$, we have that
\[
W_\infty(\delta_0,P_R{}_\#\mathcal L^1(Y^{\xi,R}_s)))\le W_\infty(\delta_0,\mathcal L^1(Y^{\xi,R}_s))).
\]
By substituting this into the above inequality and passing to the limit as $q \to \infty$ yields
\[
W_\infty(\delta_0,\mathcal L^1(Y^{\xi,R}_t))  
    \le R_0+ R_2^2\int_0^tW_\infty(\delta_0,\mathcal L^1(Y^{\xi,R}_s))ds.
\]
By applying Gr\"onwall's lemma, we get
\[
\sup_{t\in[0,1]}W_\infty(\delta_0,\mathcal L^1(Y^{\xi,R}_t)) \le R_0\exp(R_2^2)=:R_X(R_0,R_2). 
\]
 As a result, $Y^{\xi,R}_t\in\bar B(R_X)$ for every $t\in[0,1]$ $\mathbb P^1$-a.s., which implies that $P_R$ acts trivially on $\bm{Y}^{\xi,R}$, whenever $R>R_X$. Therefore, $\bm{Y}^{\xi,R}$ satisfies \eqref{eqn:MVODE1} $\mathbb P^1$-a.s., whenever $R>R_X$.\\

\textit{Step Two: Existence and Uniqueness of $\bm{p}^{\xi}$}\\
Let $\bm{Y}^{\xi}$ be the unique solution to \eqref{eqn:MVODE1}. Since $Y^{\xi}_1\in \bar B(R_X) $ $\mathbb P^1$-a.s., the terminal condition
\[
p_1^{\xi} = \partial_\mu\ell(\mathcal L^1(Y^{\xi}_1))(Y^{\xi}_1)\in L^2(\Omega^1,\mathcal F^1,\mathbb P^1;\R^d).
\]
 In fact, Assumption \ref{assm:Obj} yields that $p_1^\xi\in \bar B(K(R_X))$ $\mathbb P^1$-a.s. Furthermore, the mapping
 \[
 \Omega^1\times [0,1]\times \mathcal P_2(\R^{2d})\times \R^d\ni(\omega^1,t,\rho,p)\mapsto - \mathcal K^R(Y^{\xi}_t(\omega^1),\rho,\nu_t,p)
 \]
is Lipschitz in variables $(\rho,p)$ with respect to  the $W_2$-metric and the Euclidean metric, respectively. Furthermore, $\bm Y^\xi\in \mathscr{C}^{1,2,d}$ implies that $Y^\xi_t$ is continuous in $t$ for $\mathbb P^1$-a.e. $\omega^1$ and is $\mathcal F^1$-measurable for fixed $t$. After modifying $\bm Y^\xi$ on a $\mathbb P^1$-null set, if necessary, to make it continuous in $t$ for all $\omega^1\in\Omega^1$, $Y^\xi_t$ is a Carath\'eodory function, Definition 4.1 in \cite{AZHMYAKOV201987}, thus it is $(\mathcal F^1\otimes \mathcal B([0,1]))$-measurable by Theorem 4.1 in \cite{AZHMYAKOV201987}. Since $\mathcal K^R$ is continuous in its first and third variable, the above mapping is $(\mathcal F^1\otimes \mathcal B([0,1]))$-measurable, for fixed $(\rho,p)\in \mathcal P_2(\R^{2d})\times \R^d$.
 Accordingly, after interchanging time $t \leftrightarrow 1-t$, the terminal condition and velocity field $\mathcal K^R$ satisfy Assumption \textbf{(MKV SDE)} in Section 4.2.1 of \cite{carmona2018probabilistic} Vol. I, with identically zero diffusion coefficient.  Note that this time reversal is only valid when the forward process has no diffusion term. More care is required when diffusion is present; See Theorem 4.23 in \cite{carmona2018probabilistic} Vol. I. Hence, by Theorem 4.21 together with Remark 4.22 in \cite{carmona2018probabilistic} Vol. I., with $Y^\xi_t$ in place of their $\zeta_t$, we deduce that there exists a unique solution $\bm{p}^{\xi,R}$ to \eqref{eqn:MVODE2}, when $\mathcal K$ is replaced by $\mathcal K^R$. 

By applying $p_1^\xi\in \bar B(K(R_X))$, as discussed at the start of \textit{Step Two}, we have that
 \begin{align*}
      \Vert p_t^{\xi,R} \Vert &\le \Vert \partial_\mu\ell(\mathcal L(Y^\xi_1))(Y^\xi_1)\Vert + \int^1_t \Vert \mathcal K^R(Y^\xi_s,\mathcal L(Y^\xi_s,p^{\xi,R}_s),\nu_s, p_s^{\xi,R})ds \\& \le K(R_X)+\tilde{B}_\mathcal{K}\int_t^1 W_\infty(\delta_0,P_{R}{}_\#\mathcal L^1(p^{\xi,R}_s))ds, \qquad\mathbb P^1-a.s.,
 \end{align*}
where $\tilde{B}_\mathcal K$ in the final inequality is the constant from Lemma \ref{lemma:BoundedGamma}. By the same Gr\"onwall argument used in \textit{Step One}, we deduce that
\[
\sup_{t\in[0,1]}W_\infty(\delta_0,p^{\xi,R}_t)\le K(R_X)\exp(\tilde{B}_\mathcal{K})=:R_a.
\]
Therefore, $p^{\xi,R}_t\in \bar B(R_a)$ for all $t\in[0,1]$ $\mathbb  P^1$-a.s., which implies that $P_R$ acts trivially on $\bm{p}^{\xi,R}$ whenever $R>R_a$. As a result, the pair $(\bm{Y}^{\xi,R},\bm{p}^{\xi,R})$ satisfies \eqref{eqn:MVODE1}-\eqref{eqn:MVODE2} $\mathbb P^1$-a.s. when $R$ is taken greater than the maximum of $R_X$ and $R_a$. 
\end{proof}
 Lemma \ref{lemma:MVODE_WellDEfined} establishes the existence and uniqueness for the McKean--Vlasov ODEs associated with the transformer in the case where the mean-field parameters $\bm{\nu}_\tau$ are deterministic. However, $\bm{\nu}_\tau$ depends on the sampling of training data, and so we model $\bm{\nu}_\tau$ as a random curve of measures that takes values in $C([0,1]; \mathcal P_2(\R^{k\times 4d}))$.
    
    Let $p\in[1,\infty)$ and $(\widetilde{\Omega},\widetilde{\mathcal F})$ be a measurable space. Proposition 5.7 in \cite{carmona2018probabilistic} Vol. I states that $\mathcal B(\mathcal P_p(\R^d))$ is generated by the evaluation maps $(\mathcal P_p(\R^d)\ni \mu \mapsto \mu(D))$ for all $D\in \mathcal B(\R^d)$.  Therefore, a map $M:\widetilde{\Omega}\to \mathcal P_p(\R^d)$ is $\widetilde{\mathcal F}$-measurable if
    \[
    \{\tilde{\omega}\in \tilde{\Omega}\,\colon\,M(\tilde{\omega})(D)\in B \}=M^{-1}(\{\mu\in \mathcal P_p(\R^d)\colon \mu(D)\in B\})\in \widetilde{\mathcal F},
    \]
 for every $D\in \mathcal B(\R^d)$, and $B\in \mathcal B(\R)$. This occurs if the following criterion is satisfied.
    
\begin{criterion} \label{criterion:measure}
A map $M: \widetilde{\Omega}\mapsto\mathcal P_p(\R^d)$ is $(\widetilde{\mathcal F}, \mathcal B(\mathcal P_p(\R^d)))$-measurable if $\widetilde{\Omega}\ni\tilde{\omega}\mapsto M(\tilde\omega)(D)$ is $(\widetilde{\mathcal F},\mathcal B(\R))$-measurable for every $D \in \mathcal B(\R^d)$.
\end{criterion}
Let $\mathcal C=C([0,1],\mathcal P_2(\R^d))$ and write $e_t:\mathcal C\to \mathcal P_2(\R^d)$ for the evaluation map $e_t \bm{\mathfrak{m}}= \mathfrak{m}_t$. We equip $\mathcal C$ with the metric
\begin{equation}\label{eqn:UnifromMetric}
    \mathcal D(\bm{\mu}^1, \bm{\mu}^2)= \sup_{t\in[0,1]}W_2(e_t\bm\mu^1,e_t\bm\mu^2),
\end{equation}
which makes $(\mathcal C,\mathcal D)$
a separable metric space, Corollary 4.2.18 in \cite{Eng89}. Accordingly, there exists a countably dense subset $\widetilde{\mathcal C}\subseteq \mathcal C$ such that any open set in $(\mathcal C,\mathcal D)$ can be expressed as a countable union of sets of the form
\begin{align*}
\left\{\bm{\mathfrak{\mu}}\in \mathcal C\,|\mathcal D( \bm{\mathfrak{m}},\bm{\mathfrak{\mu}})<q\right\}  
& = \bigcap_{t\in \mathbb Q \cap [0,1]}\left\{\bm{\mu}\in \mathcal C \,|\,W_2(e_t\bm{\mathfrak{m}},e_t\bm{\mathfrak{\mu}})<q \right\},
\end{align*}
for $q\in \mathbb Q \cap [0,\infty)=: \mathbb Q_{\ge 0}$ and $\bm{\mathfrak{m}}\in \widetilde{\mathcal C}$, see Appendix M3 in \cite{billingsley2013convergence} for further details. The equality uses the continuity of the map $t\mapsto\mathfrak{m}_t$ to replace the supremum over $[0,1]$ by the supremum over any dense subset. Therefore, any open set of $(\mathcal C,\mathcal D)$ can be expressed as a countable union and intersection of sets of the form
\[
\left\{\bm{\mu}\in \mathcal C \,|\,W_2(e_t\bm{\mathfrak{m}},e_t\bm{\mathfrak{\mu}})<q \right\},
\]
for $t\in\mathbb Q\cap [0,1]$, $q\in \mathbb Q_{\ge0}$ and $\bm{\mathfrak{m}}\in \widetilde{\mathcal C}$. This implies that
\[
  \mathcal B(\mathcal C)\subseteq\sigma \left(\{\bm{\mathfrak{m}}\in\mathcal C\colon W_2(e_t\bm{\mathfrak{m}},e_t\bm\mu)<q \} : t\in \mathbb Q \cap [0,1], q\in \mathbb Q_{\ge 0}, \bm\mu\in \widetilde{\mathcal C} \right).
\] 
Take any $\widetilde{\bm{\mathfrak{m}}}: \widetilde{\Omega}\to \mathcal C$. The above demonstrates that $\widetilde{\bm{\mathfrak{m}}}$ is $(\widetilde{\mathcal F},\mathcal B(\mathcal C))$-measurable if
\[
\left\{ \tilde\omega \in \widetilde{\Omega}:W_2(\widetilde{\mathfrak{m}}_t(\tilde{\omega}),e_t\bm\mu)< q\right\}=\widetilde{\bm{\mathfrak{m}}}^{-1}(\{\bm{\mathfrak{m}}\in \mathcal C: W_2(e_t\bm{\mathfrak{m}},e_t\bm\mu)<q \})\in \widetilde{\mathcal F},
\]
for every $t\in[0,1]$, $q\in \mathbb Q_{\ge 0}$ and $\bm{\mu}\in \widetilde{\mathcal C}$. This occurs whenever $\tilde\omega\mapsto\widetilde{\mathfrak{m}_t}(\tilde\omega)$ is $(\widetilde{\mathcal F},\mathcal B(\mathcal P_2(\R^d)))$-measurable for every $t\in [0,1]$.
By combining this with Criterion \ref{criterion:measure} for measurability with respect to $\mathcal B(\mathcal P_2(\R^d))$, we deduce the following.
\begin{criterion}\label{criterion}
$\widetilde{\bm{\mathfrak{m}}}: \widetilde{\Omega}\mapsto \mathcal C$  is $(\widetilde{\mathcal F}, \mathcal B(\mathcal C))$-measurable if the map
\[
\widetilde{\Omega}\ni \tilde{\omega}\mapsto \tilde{\mathfrak{m}}_t(\tilde{\omega})(D)
\]
is $(\widetilde{\mathcal F},\mathcal B(\R))$-measurable for every $t\in[0,1]$ and $D \in \mathcal B(\R^d)$.
\end{criterion}
Let us generalise Lemma \ref{lemma:MVODE_WellDEfined} to include random mean-field parameters $\bm{\nu}$. We do this by working on the product probability space $(\Omega,\mathcal F,\mathbb P)$, defined in Section \ref{sec:MainTheorem}, so that the randomness due to $\bm\nu$ and initial condition $\xi$ are defined on separate spaces $(\Omega^0,\mathcal F^0)$ and $(\Omega^1,\mathcal F^1)$, respectively, which are then canonically embedded into $(\Omega,\mathcal F)$. As discussed in Section 2.1.3 in \cite{carmona2018probabilistic} Vol. II, the section $X(\omega^0,\cdot)$ of a random variable $X$ defined on $(\Omega,\mathcal F,\mathbb P)$, where $\mathcal F$ is the completion of $\mathcal F^0\otimes \mathcal F^1$, is only a random variable on $(\Omega^1,\mathcal F^1)$ for $\mathbb P^0$-a.e. $\omega^0\in \Omega^0$. Accordingly, the law of $X(\omega^0,\cdot)$ is only well-defined $\mathbb P^0$-a.s.. However, we choose the initial condition $\xi$ and $\bm\nu$ to be independent so that all variables that follow can be defined on $(\Omega, \mathcal F^0\otimes \mathcal F^1)$, which makes their section a random variable defined on $(\Omega^1,\mathcal F^1)$ for all $\omega^0 \in \Omega^0$. For a random variable that is $(\mathcal F^0\otimes \mathcal F^1)$-measurable, it is well-defined to denote the conditional law of $X$ on $(\Omega^1,\mathcal F^1,\mathbb P^1)$ by $\mathcal L^1(X)(\omega^0)= \mathbb P^1\circ X(\omega^0,\cdot)^{-1}$, Section 2.1.3 in \cite{carmona2018probabilistic} Vol. II.
\begin{lemma}\label{lemma:MVRODE_wd}
  Let $R_2>0$ and $\bm\nu=(\nu_t)_{t\in[0,1]}$ be a stochastic process on the probability space $(\Omega^0,\mathcal F^0,\mathbb P^0)$ such that:
\begin{itemize}
    \item The map $(\omega^0,t) \mapsto \nu_t(\omega^0)$ is $(\mathcal F^0\otimes \mathcal B([0,1]), \mathcal B(\mathcal P(\R^{k\times 4d}))$-measurable.
    \item There exists a subset $\Omega^0_\mathrm{w.d.}\subseteq\Omega^0$ with $\mathbb P^0(\Omega^0_\mathrm{w.d.})=1$ so that $\bm\nu(\omega^0)\in C([0,1],\mathcal P(\mathfrak{S}(R_2)) )$ for every $\omega^0\in \Omega^0_\mathrm{w.d.}$, where $\mathcal P(\mathfrak S(R_2))$ is endowed by the topology induced by the $W_2$ metric.
\end{itemize}
Then, for any initial condition $\xi\in L^\infty(\Omega^1,\mathcal F^1,\mathbb P^1;\R^d)$ with $\Vert \xi\Vert \le R_0$ $\mathbb P^0$-a.s., there exist a solution $(\bm{Y}^\xi,\bm{p}^\xi)\in \mathscr{C}^{2,2d}$, unique up to $\mathbb P$-a.s. equality, to the forward-backward system of McKean--Vlasov ODEs
\begin{align}
    \frac{dY^{\xi}_t}{dt} &= \Gamma(Y^\xi_t,\mathcal L^1(Y^\xi_t),\nu_t), \qquad &Y^\xi_0 =\xi,\label{eqn:MVRODE1}\\
    \frac{d p^{\xi}_t}{dt} &= - \mathcal K(Y^\xi_t,\mathcal L^1(Y^\xi_t,p^\xi_t),\nu_t,p^\xi_t), \qquad &p^\xi_1 = \partial_\mu\ell(\mathcal L^1(Y^\xi_1))(Y^\xi_1).\label{eqn:MVRODE2}
\end{align}
Furthermore, the constants $R_X,R_a$, defined in Lemma \ref{lemma:MVODE_WellDEfined}, are such that $Y_t^\xi\in \bar B(R_X)$ and $p^\xi_t\in \bar B(R_a)$ for every $t\in[0,1]$ $\mathbb P$-almost surely.
\end{lemma}
\begin{remark}
We require that $\xi$ and $(\nu_t)_{t\in[0,1]}$ are independent. This occurs in practice since we assume that the training data at step $\tau$ is sampled independently of the past training data. However,
    we take $\xi\in L^2(\Omega^1,\mathcal F^1,\mathbb P^1)$, despite having introduced $(\Omega^1,\mathcal F^1,\mathbb P^1)$ in Section \ref{sec:MainTheorem} as the source of randomness from parameter initialisation. This is done so that an additional probability space does not have to be introduced, whilst keeping  $\xi$ and $\bm\nu$ independent. 
\end{remark}
\begin{proof}
Throughout the proof, we shall take $R> R_X \vee R_a$, where $R_X,R_a$ are the constants defined in Lemma \ref{lemma:MVODE_WellDEfined}. This ensures that the solution $(\bm Y^\xi,\bm p^\xi)$ to \eqref{eqn:MVRODE1}-\eqref{eqn:MVRODE2} coincides $\mathbb P$-a.s. with the corresponding solution $(\bm Y^{\xi,R},\bm p^{\xi,R})$ when $\Gamma,\mathcal K$ are replaced by $\Gamma^R,\mathcal K^R$.\\

\textit{Step One: Existence of $(\bm{Y}^{\xi}, \bm{p}^\xi)$}\\

Given $\xi \in L^\infty(\Omega^1,\mathcal F^1,\mathbb P^1; \bar B(R_0))$ that is canonically embedded as a random variable on \\$(\Omega,\mathcal F^0\otimes \mathcal F^1)$, the assumptions of Lemma \ref{lemma:MVODE_WellDEfined} are satisfied by $\xi$ and the mean-field parameters $(\nu_t(\omega^0))_{t\in[0,1]}$ whenever $\omega^0\in \Omega^0_\mathrm{w.d.}$. Thus, there exists $(\bm{Y}^\xi,\bm{p}^\xi)(\omega^0,\cdot)\in \mathscr{C}^{1,2,2d}$ such that \eqref{eqn:MVRODE1}-\eqref{eqn:MVRODE2} is satisfied for $\mathbb P^0$-a.e. $\omega^0$. Since we cannot guarantee the existence of the solution to  \eqref{eqn:MVRODE1}-\eqref{eqn:MVRODE2} on $(\Omega^0_\mathrm{w.d.})^c$, define $$(\bm{Y}^\xi,\bm{p}^\xi)(\omega^0,\cdot)=\left((\xi,\partial_\mu\ell(\mathcal L^1(\xi))(\xi))\right)_{t\in[0,1]}, \qquad \omega^0\notin\Omega^0_\mathrm{w.d.}.$$ This guarantees that $(\bm{Y}^\xi,\bm{p}^\xi)(\omega^0,\cdot)\in\mathscr{C}^{1,2,2d}$ for every $\omega^0\in \Omega^0$, as Assumption \ref{assm:Obj} implies that  for $\partial_\mu\ell(\mathcal L^1(\xi))(\xi)\in L^\infty(\Omega^1)$.\\

The uniqueness and boundedness claims are immediate consequences of Lemma \ref{lemma:MVODE_WellDEfined}.

\textit{Step Two: Measurability of $(\bm{Y}^\xi,\bm{p}^\xi)$}

Fix $\bm{\mu}\in \mathcal C$, $x\in\R^d$, and consider the ODE
\begin{equation}\label{eqn:FlowMapY}
    \dot{Y}^{x,\bm{\mu},R}_t = \Gamma^R(Y^{x,\bm{\mu},R}_t,\mu_t,\nu_t), \qquad Y_0^{x,\bm{\mu},R}=x,
\end{equation}
which has a unique solution $\bm{Y}^{x,\bm{\mu},R}=(Y_t^{x,\bm{\mu},R})_{t\in[0,1]}$, by  Lemma 8.1.4 in \cite{Ambrosio2008} whenever \\$\omega^0\in \Omega^0_\mathrm{w.d.}$. To remain consistent with the convention for $\bm{Y}^\xi$ on $(\Omega^0_\mathrm{w.d.})^c$, we set $\bm{Y}^{x,\bm{\mu},R}=x$ for $\omega^0\notin \Omega^0_\mathrm{w.d.}$. Now, let us define the operator
\begin{align*}
&\mathcal T^{R}: \Omega^0\times L^\infty(\Omega^1; \bar B(R_0))\times \mathcal C\mapsto \mathcal C,\qquad\mathcal T^R(\omega^0,\xi,\bm{\mathfrak{m}})\colon= 
   (\mathcal L^1(Y^{\xi,\bm{\mathfrak{m}},R}_t)(\omega^0))_{t\in[0,1]},
\end{align*}
where $Y^{\xi,\bm{\mathfrak{m}},R}_t$ refers to $Y^{x,\bm{\mathfrak{m}},R}_t$ evaluated at $x=\xi$. Notice that $\mathcal T^R(\omega^0,\xi,\bm{\mathfrak{m}})= \mathcal L^1(\xi)$ for \\$\omega^0\notin \Omega^0_\mathrm{w.d.}$.  As shown in the proof of Theorem 4.21 in \cite{carmona2018probabilistic} Vol. I, $\mathcal L^1(\bm{Y}^{\xi}_t)(\omega^0)$ coincides with $\bm{\mathfrak{m}}_t^\infty(\omega^0)$  for every $\omega^0\in \Omega^0_\mathrm{w.d.}$ and $t\in[0,1]$, where $\bm{\mathfrak{m}}^\infty$ is obtained as the limit, with respect to metric $\mathcal D$, of the Picard iterates
\[
\bm{\mathfrak{m}}^{n+1} = \mathcal T^R(\omega^0,\xi, \bm{\mathfrak{m}^n}), \qquad \bm{\mathfrak{m}}^0= (\mathcal L^1(\xi))_{t\in[0,1]}.
\]
 On $(\Omega^{0}_\mathrm{w.d.})^c$, $\mathcal T^R_{\omega^0,\xi}$ is a constant map and $\bm{Y}^{\xi}(\omega^0)=\xi$. Therefore, $\bm{\mathfrak{m}}_t^n(\omega^0)$ converges to $\mathcal L^1(Y^\xi_t)(\omega^0)$ for every $t\in[0,1]$ and $\omega^0\in \Omega^0$.

Let us show by induction that $\bm{\mathfrak{m}}^n$ belongs to
\[
\mathfrak{M} = \{ \bm{\mathfrak{m}}: \Omega^0\mapsto \mathcal C \,| \,\bm{\mathfrak{m}} \text{ is } (\mathcal F^0, \mathcal B(\mathcal C))\text{-measurable} \},
\]
for every $n\ge 0$. The base case $\bm{\mathfrak{m}}^0$ is deterministic, and hence $\bm{\mathfrak{m}}^0\in \mathfrak{M}$. Since measurability is stable under pointwise limits, if $\mathcal T^R$ preserves measurability, then $\mathfrak{m}^\infty\in \mathfrak{M}$ by induction. Let us use the short hand notation $Y^{x,n,R}_t$ to denote $Y^{x,\bm{\mu},R}_t$ evaluated at $\bm{\mu}=\bm{\mathfrak{m}}^n$, and suppose that $Y^{\xi,n,R}_t$ is $(\mathcal F^0\otimes\mathcal F^1)$-measurable. Then, for any $D \in \mathcal B(\R^d)$ the function
\[
\Omega\ni \omega\mapsto \mathbbm{1}_{D}(Y^{\xi,n,R}_t(\omega))
\]
is $(\mathcal F^0\otimes \mathcal F^1)$-measurable. Therefore, applying Lemma 1.28 in \cite{Kallenberg2021}, yields that
\[
\Omega^0\ni\omega^0\mapsto\mathfrak{m}^{n+1}_t(D) = \mathbb P^1\left[ Y^{\xi,n,R}_t(\omega^0,\cdot)\in D\right]
\]
is $\mathcal F^0$-measurable for every $D\in\mathcal B(\R^d)$, $t\in[0,1]$. By Criterion \ref{criterion}, this is sufficient to conclude that $\bm{\mathfrak{m}}^{n+1}\in \mathfrak{M}$. Therefore, to complete the inductive step, we are left to show that $Y^{\xi,n,R}_t$ is $(\mathcal F^0\otimes \mathcal F^1)$-measurable given that $\bm{\mathfrak{m}}^n\in \mathfrak{M}$. Consider the $\sigma$-algebra defined by
\[
\mathcal F^0_\mathrm{w.d.}:=\{\Omega^{0}_\mathrm{w.d.}\cap A: A \in \mathcal F^0\}.
\]
Since $(\Omega^0,\mathcal F^0,\mathbb P^0)$ is a complete probability space and $(\Omega^{0}_{\mathrm{w.d.}})^c$ is a $\mathbb P^0$ null set, $(\Omega^{0}_{\mathrm{w.d.}})^c\in\mathcal F^0$. This implies that $A\cap \Omega^0_\mathrm{w.d.}\in \mathcal F^0$ for every $A\in \mathcal F^0$, hence $\mathcal F^0_\mathrm{w.d.}$ is a sub-$\sigma$-algebra of $\mathcal F^0$. Let $Z$ be any random variable  defined on $(\Omega^0,\mathcal F^0,\mathbb P^0)$ that takes values in $\R^d$, and write $Z_\mathrm{w.d.}$ for its restriction to $\Omega^0_\mathrm{w.d.}$. Then, for every $S\in \mathcal B(\R^d)$
\[
Z_\mathrm{w.d.}^{-1}(S) = \Omega^0_\mathrm{w.d.}\cap \{\omega^0\in \Omega^0\colon Z(\omega^0)\in S\}\in \mathcal F^0_\mathrm{w.d.}.
\]
Thus, the restriction of any random variable defined on $(\Omega^0,\mathcal F^0)$ to $\Omega^0_\mathrm{w.d.}$ that takes values in $\R^d$ is $\mathcal F^{0}_\mathrm{w.d.}$-measurable. In particular, if $\bm{\mathfrak{m}}^n \in \mathfrak{M}$, the velocity field
\[
\Omega^0_\mathrm{w.d.}\times \R^d\times[0,1]\ni(\omega^0,x,t)\mapsto\Gamma^R(x,\mathfrak{m}^n_t(\omega^0),\nu_t(\omega^0))
\]
is continuous in $(x,t)$ for fixed $\omega^0\in \Omega^0_\mathrm{w.d.}$, and $\mathcal F^0_\mathrm{w.d.}$-measurable for fixed $(x,t)$. Hence, Lemma 2.2 in \cite{Han2017} applies and yields that $ Y^{x,n,R}_t$ is $\mathcal F^0_\mathrm{w.d.}$-measurable for each $t\in[0,1]$.

Notice that for any $S \in \mathcal B(\R^d)$
\[
(Y^{x,n,R}_t)^{-1}(S) = \underbrace{\{ \omega^0\in \Omega^{0}_\mathrm{w.d.}\colon Y^{x,n,R}_t(\omega^0)\in S\}}_{=:S_1}\cup \underbrace{\{\omega^0\notin \Omega^0_\mathrm{w.d.}\colon Y^{x,n,R}_t(\omega^0)=x\in S\}}_{=:S_2}.
\]
The set $S_1$ is an element of $\mathcal F^0_{\mathrm{w.d.}}\subseteq \mathcal F^0$, since the restriction of $Y^{x,n,R}_t$ to $\Omega^0_\mathrm{w.d.}$ is $\mathcal F^0_\mathrm{w.d.}$-measurable. 
The second set $B_2\subseteq (\Omega^0_\mathrm{w.d.})^c$ is a $\mathbb P^0$-null set. Therefore, $B_2\in \mathcal F^0$ as $(\Omega^0,\mathcal F^0,\mathbb P^0)$ is complete. Hence, $B_1\cup B_2\in \mathcal F^0$, which implies that $\Omega^0\ni\omega^0\mapsto Y^{x,n,R}_t(\omega^0)$ is $\mathcal F^0$-measurable. Furthermore, Lemma 8.1.4 in \cite{Ambrosio2008} also yields that the mapping $x\mapsto \bm{Y}^{x,n,R}(\omega^0)$ is continuous, with respect to the supremum norm on $C([0,1];\R^d)$, for each fixed $\omega^0\in \Omega^0$. Therefore, \[\R^d\times \Omega^0\ni(x,\omega^0)\mapsto Y^{x,n,R}_t(\omega^0)\]
 is a Carath\'eodory function, Definition 4.1 in \cite{AZHMYAKOV201987}, and is thus jointly measurable in $(\omega^0,x)$ by Theorem 4.1 in \cite{AZHMYAKOV201987}. Note that the composition of measurable functions is measurable when the $\sigma$-algebra on the co-domain of the first map agrees with the $\sigma$-algebra on the domain of the second map. Accordingly, $Y^{\xi,n,R}_t$ is $(\mathcal F^0\otimes\mathcal F^1)$-measurable for every $t\in[0,1]$, having canonically embedded the random variable $Y^{\xi,n,R}_t$ defined on $(\Omega^0,\mathcal F^0)$ into $(\Omega,\mathcal F^0\otimes \mathcal F^1)$. This completes the inductive step.\\

Let $\bm Y^{x,\xi}$ denote $\bm Y^{x,\bm\mu,R}$ evaluated at $\bm\mu=\bm{\mathfrak{m}}^\infty$. The preceding analysis shows that $Y^{x,\xi}_t$ is $\mathcal F^0$-measurable and $Y^{x,\xi}_t|_{x=\xi}$ is $(\mathcal F^0\otimes \mathcal F^1)$-measurable for every $t\in[0,1]$. Furthermore, notice that $\bm Y^{x,\xi}|_{x=\xi}$ satisfies \eqref{eqn:MVRODE1} $\mathbb P$-a.s., and thus provides a version of the solution to \eqref{eqn:MVRODE1} that is $(\mathcal F^0\otimes \mathcal F^1)$-measurable for every $t\in[0,1]$. \\

\textit{Step Three: Measurability of $\bm{p}^{\xi,R}$} \\

By Assumption \ref{assm:Obj}, $\partial_\mu\ell$ is jointly continuous in each argument. By combining this assumption with the analysis of \textit{Step Two}, we find that the terminal condition $\partial_\mu\ell(\mathcal L^1(Y^\xi_1))(Y^{x,\xi}_1)$ is $\mathcal F^0$-measurable. Furthermore, as seen in \textit{Step Two} of the proof of Lemma \ref{lemma:MVODE_WellDEfined}, the terminal condition is bounded $\mathbb P^0$-a.s. by $K(R_X)$ for every $\omega^0\in \Omega^0$. Since we use the convention that $Y_1^{x,\xi}=x$ on $(\Omega^0_\mathrm{w.d.})^c$ and $R_0\le R_X$, it follows that $\partial_\mu\ell(\mathcal L^1(Y^\xi_1))(Y^{x,\xi}_1)\in \bar B(K(R_X))$ for all $\omega^0\in\Omega^0$. 

Let $\bm{q}\in \mathscr{C}^{2,d}$ such that $q_t$ is $(\mathcal F^0\otimes \mathcal F^1)$-measurable for every $t\in[0,1]$. Then, $(Y_t^\xi,q_t)$ is \\$(\mathcal F^0\otimes \mathcal F^1)$-measurable for every $t\in[0,1]$, which implies that $\mathcal L^1(Y^{\xi}_t,q_t)(\omega^0)$ is well-defined for every $\omega^0\in \Omega^0$. Consider, for
$\omega^0\in \Omega^0_\mathrm{w.d.}$, the velocity field given by
\[
(t,p,\rho)\mapsto- \mathcal K^R(Y^{x,\xi}_t(\omega^0),\rho, \nu_t(\omega^0),p),
\]
which is Lipschitz in variables $(p,\rho)$ and continuous in $t$ for fixed $(p,\rho)$. Therefore, the above velocity field satisfies the assumptions of Lemma 8.1.4 in \cite{Ambrosio2008}. Thus, by applying this Lemma for $\omega^0\in\Omega^0_\mathrm{w.d.}$, 
the ODE
\begin{equation}\label{eqn:FlowMapP}
    \dot{p}^{x,\xi,\bm{q}, R}_t = -\mathcal K^R(Y^{x,\xi}_t,\mathcal L^1(Y^{\xi}_t,q_t), \nu_t,p^{x,\xi,\bm{q}, R}_t), \qquad p^{x,\xi,\bm{q},R}_1 = \partial_\mu\ell(\mathcal L^1(Y^{\xi}_1))(Y^{x,\xi}_1).
\end{equation}
has a unique solution in $C([0,1];\R^d)$ for every $x\in \R^d$, $\xi \in L^2(\Omega^1)$. For $\omega\notin \Omega^0_\mathrm{w.d.}$, we define $p^{x,\xi,\bm{q},R}_t=\partial_\mu\ell(\mathcal L^1(\xi))(x)$ so that $t\mapsto p_t^{x,\xi,\bm{q},R}$ is continuous in $t$ for every $\omega^0\in \Omega^0$. Furthermore, Lemma 8.1.4 in \cite{Ambrosio2008} states that $p^{x,\xi,\bm{q},R}_t$ is continuous in its terminal value. In \textit{Step Two}, we deduced that $x\mapsto Y^{x,\xi}_1$ is continuous for every $\omega^0\in \Omega^0$. By Assumption \ref{assm:Obj}, the terminal condition for $p_1^{x,\xi,\bm q,R}$ is a continuous function of $Y^{x,\xi}_1$, and thus it is continuous in $x$. Accordingly, the map $x\mapsto p^{x,\xi,\bm{q},R}_t$ is continuous, which implies that $p_t^{x,\xi,\bm q,R}(\omega^0)|_{x=\xi}$ is $\mathcal F^1$-measurable for every $t\in[0,1]$, $\omega^0\in \Omega^0$ and $\xi \in L^\infty(\Omega^1;\bar B(R_0))$. Then, the continuity of $t\mapsto p_t^{x,\xi,\bm q,R}(\omega^0)|_{x=\xi}$ seen earlier yields that $\bm p^{x,\xi,\bm q,R}(\omega^0)|_{x=\xi}\in \mathscr{C}^{1,2,2d}$ for any $\xi \in L^\infty(\Omega^1; \bar B(R_0))$ and $\omega^0\in \Omega^0$. \\

Consider the operator,
\[
\widetilde{\mathcal T}^R:\Omega^0\times L^\infty(\Omega^1;\bar B(R_0))\times \mathscr{C}^{1,2,d}\to \mathscr{C}^{1,2,d}, \qquad \widetilde{\mathcal T}^R(\omega^0,\xi,\bm{q})\mapsto \bm{p}^{x,\xi,\bm{q},R}(\omega^0)|_{x=\xi}.
\]
Theorem 4.23 in \cite{carmona2018probabilistic} Vol. I shows that the Picard iterates defined by
\[
\bm{p}^{\xi,n+1,R} = \widetilde{\mathcal T}^R(\cdot,\xi,\bm{p}^{\xi,n,R}) , \qquad \bm{p}^{\xi,0,R}=\partial_\mu\ell(\mathcal L^1(\xi))(\xi).
\]
converge, with respect to the metric induced by $\Vert_\cdot\Vert_{\mathscr{C}^{1,2,d}}$, to $\bm p^\xi$ for every $\omega^0\in \Omega^0_\mathrm{w.d.}$. Since $\bm{p}^{x,\xi,\bm{q},R}$ is independent of $\bm{q}$ for $\omega^0\notin \Omega^0_\mathrm{w.d.}$, this convergence holds for every $\omega^0\in\Omega^0$. Note that Theorem 4.23 in \cite{carmona2018probabilistic} Vol. I concerns backward McKean--Vlasov SDEs, where the solution is a pair of adapted processes. A second process for the diffusion term in the BSDE is required to ensure that the first process is adapted to the forward filtration. In our setting, we can take the filtration to be independent of time $t$ and apply Remark 4.22 in \cite{carmona2018probabilistic} Vol. I, with $Y^\xi_t$ in place of $\zeta_t$, to show that the solution to the backward equation for $\bm{p}^{x,\xi,\bm{q},R}$ has identically zero diffusion coefficient. Thus, the convergence of the Picard iterates $\bm p^{\xi,n,R}$ follows from Theorem 4.23 in \cite{carmona2018probabilistic} Vol. I, by applying their argument with all diffusion terms set to zero.

 Since $Y^\xi_t$ is $(\mathcal F^0\otimes \mathcal F^1)$-measurable and $\xi $ is $\mathcal F^1$-measurable, the map 
\[
\Omega\ni (\omega^0,\omega^1)\mapsto \mathbbm{1}_{D_1}(Y^{\xi}_t(\omega^0,\omega^1))\mathbbm{1}_{D_2}(\xi(\omega^1))
\]
is $(\mathcal F^0\otimes \mathcal F^1)$-measurable for every $D_1,D_2\in \mathcal B(\R^d)$. By Lemma 1.28 in \cite{Kallenberg2021}, the map \[\Omega^0\ni\omega^0\mapsto \mathcal L^1(Y^\xi_t,p^{\xi,0,R}_t)(\omega^0)(D_1\times D_2)= \int \mathbbm{1}_{D_1}(Y^{\xi}_t(\omega^0,\omega^1))\cdot\mathbbm{1}_{D_2}(\xi(\omega^1))d\mathbb P^1(\omega^1)\] is $\mathcal F^0$-measurable. Hence, since $\mathcal B(\R^{2d})=\mathcal B(\R^d)\otimes \mathcal B(\R^d)$, Criterion \ref{criterion} applies to guarantee that $\bm{\mathfrak{r}}^0\in \mathfrak{M}$.
By following the same argument as \textit{Step Two}, if $p^{\xi,n,R}_t$ is $(\mathcal F^0\otimes\mathcal F^1)$-measurable for every $t\in[0,1]$, then $\bm{\mathfrak{r}}^n:=(\mathcal L^1(Y^\xi_t,p^{\xi,n,R}_t))_{t\in[0,1]}$ is well-defined and belongs to $\mathfrak{M}$. Then, with $\bm{\mathfrak{m}}^n$ and $\bm Y^{\xi,n,R}$ replaced by $\bm{\mathfrak{r}}^n$ and $\bm{p}^{\xi,n,R}$, respectively, in the argument of \textit{Step Two}, we deduce that $p_t^{\xi,n+1,R}$ is $(\mathcal F^0\otimes\mathcal F^1)$-measurable for every $t\in[0,1]$ whenever $\bm{\mathfrak{r}}^n\in\mathfrak{M}$. Thus, we have by induction that $p^{\xi,n}_t$ is $(\mathcal F^0\otimes \mathcal F^1)$-measurable for every $t\in[0,1], n\ge 0$. 

By taking the coupling $\mathcal L^1((Y^{\xi}_t,p^{\xi,n,R}_t), (Y^\xi_t,p^\xi_t)) $, we obtain
\[
W^2_2(\mathfrak{r}^n_t, \mathcal L^1(Y^\xi_t,p^\xi_t))\le \E^1\left[\Vert p^{\xi,n,R}_t-p^{\xi}_t\Vert^2\right]\le \Vert \bm p^{\xi,n,R}-\bm p^{\xi}\Vert^2_{\mathscr{C}^{1,2,d}}
\]
holds for every $t\in[0,1]$ and $\omega^0\in \Omega^0$. Hence $\bm{\mathfrak{r}}^n$ converges to $( \mathcal L^1(Y^\xi_t,p^\xi_t))_{t\in[0,1]}$ with respect to $\mathcal D$. The preceding analysis shows that $\bm{\mathfrak{r}}^n\in\mathfrak{M}$ for every $n\ge0$. Therefore, $(\mathcal L^1(Y^\xi_t,p^\xi_t))_{t\in[0,1]}\in \mathfrak{M}$ by the stability of measurability under pointwise limits.
\end{proof}

\begin{corollary}\label{cor:FlowMap_wd}
Let $R_2>0$ and $\bm{\nu}$ satisfy the same conditions as in Lemma \ref{lemma:MVRODE_wd} and denote the solution to \eqref{eqn:MVRODE1}-\eqref{eqn:MVRODE2} by $(\bm{Y}^\xi,\bm p^\xi)\in \mathscr{C}^{2,2d}$. Then, there exists a solution $(\bm Y^{x,\xi}, \bm p^{x,\xi})\in \mathscr{C}^{0,2,2d}$, unique up to $\mathbb P^0$-a.s. equality, to the following system of ODEs
    \begin{align}
    \dot{Y}^{x,\xi}_t &= \Gamma(Y^{x,\xi}_t, \mathcal L^1(Y^\xi_t),\nu_t), \qquad & Y_0^{x,\xi}=x,\label{eqn:FMY}\\
    \dot{p}^{x,\xi}_t &= -\mathcal K(Y^{x,\xi}_t, \mathcal L^1(Y^\xi_t, p^\xi_t),\nu_t, p^{x,\xi}_t), \qquad & p_1^{x,\xi}=\partial_\mu\ell(\mathcal L^1(Y^\xi_1))(Y^{x,\xi}_1).\label{eqn:FMp}
    \end{align}
Furthermore, there exists a constant $\Lambda_F=\Lambda_F(R_0,R_2,\beta)$ such that
\[
\Vert Y^{x_1,\xi_1}_t-Y^{x_2,\xi_2}_t\Vert + \Vert p^{x_1,\xi_1}_t-p^{x_2,\xi_2}_t\Vert \le \Lambda_F \left( \Vert x_1-x_2\Vert + W_2(\mathcal L^1(\xi_1),\mathcal L^1(\xi_2))\right),
\]
$\mathbb P^0$-almost surely.
\end{corollary}
\begin{remark}\label{remark:flow}
   The Lipschitz estimate of $(\bm Y^{x,\xi},\bm p^{x,\xi})$ in $\xi\in L^2(\Omega^1)$ of Corollary \ref{cor:FlowMap_wd} shows that $(\bm{Y}^{x,\xi}, \bm p^{x,\xi})$ depends on $\xi \in L^2(\Omega^1)$ only through its law. Therefore, with $\zeta = \mathcal L^1(\xi)\in \mathcal P(\bar B(R_0))$, the paths
$\bm{X}^{x,\zeta},\bm{a}^{x,\zeta}\in C([0,1]; \R^d)$ given by $(\bm{X}^{x,\zeta},\bm{a}^{x,\zeta})=(\bm{Y}^{x,\xi},\bm{p}^{x,\xi})$ are well-defined such that $(X_t^{x,\zeta},a_t^{x,\zeta})$ are $\mathbb P^0$-a.s. Lipschitz in $(t,x,\zeta)$ with respect to the Euclidean and $W_2$-metric, respectively. The Lipschitz continuity in $t$ results immediately from the boundedness of $(X^{x,\zeta}_t,a^{x,\zeta}_t)$ and the Linear growth of $\mathcal K$ and $\Gamma$, Lemma \ref{lemma:BoundedGamma}. Moreover, $(X_t^{x,\zeta},a_t^{x,\zeta})$ is $\mathcal F^0$-measurable for any $t\in[0,1], x\in \R^d$ and $\zeta \in \mathcal P(\bar B(R_0))$. 
\end{remark}
\begin{proof}
Notice that $\bm Y^{x,\xi},\bm p^{x,\xi}$ were introduced earlier in \eqref{eqn:FlowMapY} and \eqref{eqn:FlowMapP} with $\bm \mu$ and $\bm q$ evaluated at $\mathcal L^1(Y^\xi_t)$ and $\mathcal L^1(Y^\xi_t,p^{\xi}_t)$, respectively. The existence of $(\bm Y^\xi,\bm p^\xi)\in\mathscr{C}^{2,2d}$ was established in Lemma \ref{lemma:MVRODE_wd}. As a result, the existence, uniqueness and measurability of $\bm Y^{x,\xi},\bm p^{x,\xi}$ follows directly from \textit{Steps Two} and \textit{Three} of the proof of Lemma \ref{lemma:MVRODE_wd}. Therefore, only the Lipschitz continuity with respect to the initial condition requires further study. Lemma 3.1 in \cite{peng}, with identically zero diffusion terms, establishes that there exists a constant $\Lambda_Y$ depending only on the Lipschitz constant $\Lambda_\Gamma^R$ of $\gamma^R$ such that
\[
\sup_{s\in[0,1]}\Vert Y^{x_1,\xi_1}_t-Y^{x_2,\xi_2}_t\Vert \le \Lambda_Y\left(\Vert x_1-x_2\Vert+W_2(\mathcal L^1(\xi_1),\mathcal L^1(\xi_2)) \right),\qquad \mathbb P^0-a.s.
\]
Let us argue similarly to Lemma 3.1 in \cite{peng} to prove that the same holds for the adjoint variable. Suppose $\xi_1',\xi_2'\in L^\infty(\Omega^1;\R^d)$ such that $\mathcal L^1(\xi_i')=\mathcal L^1(\xi_i)$ for $i=1,2$. 
By the uniqueness of the solution to \eqref{eqn:MVRODE1}, the law of $(Y^{\xi_i',\xi_i}_s,p^{\xi_i',\xi_i}_s)$ conditioned on $\mathcal F^0$ coincides with $\mathcal L^1( Y^{\xi_i}_s,p^{\xi_i}_s)$ $\mathbb P^0$-a.s. for $i=1,2$ and $s\in[0,1]$. Then, the Lipschitz continuity of $\bm Y^{x,\xi}$ implies that 
\begin{align}\label{eqn:EYxixi}
    \sup_{s\in[0,1]} W^2_2(\mathcal L^1(Y^{\xi_1}_s),\mathcal L^1(Y^{\xi_2}_s))&\le \sup_{s\in[0,1]}\E^1[ \Vert Y^{\xi_1',\xi_1}_s-Y^{\xi_2',\xi_2}_s\Vert^2]\\& \le 2\Lambda_Y^2\left( W_2^2(\mathcal L^1(\xi_1),\mathcal L^1(\xi_2)) + \E^1\left[ \Vert \xi_1'-\xi_2'\Vert^2\right]\right).\notag
\end{align}
The final term becomes the squared $2$-Wasserstein distance between $\mathcal L^1(\xi_1)$ and $\mathcal L^1(\xi_2)$ after taking the infimum over all couplings between $\xi_1',\xi_2'$.
Accordingly, the difference in terminal conditions for $p$ can be bounded using Assumption \ref{assm:Obj} to get
\begin{align*}
    \Vert p^{x_1,\xi_1}_1-p_1^{x_2,\xi_2}\Vert^2 &\le \Lambda^2_\ell(R_X)\left( W_2(\mathcal L^1(Y^{\xi_1}_1), \mathcal L^1(Y^{\xi_2}_1))+\Vert Y^{x_1,\xi_1}_1-Y^{x_2,\xi_2}_1\Vert\right)^2\\
    &\le 4\Lambda^2_\ell(R_X)\Lambda^2_Y\left(\Vert x_1-x_2\Vert^2 + 3W^2_2(\mathcal L^1(\xi_1),\mathcal L^1(\xi_2))\right).
\end{align*}
Then, by applying the above estimate, Young's and Cauchy-Schwarz inequalities and the Lipschitz continuity of $\mathcal K^R$, we obtain, for $\omega^0\in \Omega^0_\mathrm{w.d.}$, that
\begin{align}\label{eqn:differeince_in_p}
    \E^1\left[\left\Vert p^{\xi_1',\xi_1}_t-p_t^{\xi_2',\xi_2} \right\Vert^2\right]&\le 24\Lambda_Y^2\Lambda_\ell^2(R_X)\left(\E^1\left[\Vert\xi_1'-\xi_2'\Vert^2\right]+W_2^2(\mathcal L^1(\xi_1),\mathcal L^1(\xi_2)) \right)\notag\\
    & + 6\Lambda^2_{\mathcal K}(1-t)\int_t^1\E^1\left[\Vert Y^{\xi_1',\xi_1}_s-Y^{\xi_2,\xi_2'}_s \Vert^2+\Vert p^{\xi_1',\xi_1}_s-p^{\xi_2',\xi_2}_s\Vert^2 \right]ds\notag\\
    &+ 6\Lambda^2_{\mathcal K}(1-t)\int_t^1 W_2^2(\mathcal L^1(Y^{\xi_1}_s,p^{\xi_1}_s),\mathcal L^1(Y^{\xi_2}_s,p^{\xi_2}_s))ds,
\end{align}
where $\Lambda_\mathcal K$ is the Lipschitz constant of $\mathcal K^R$ with $R=R_X\vee R_a$.
The final term can be bounded using the Lipschitz continuity of $(x,\xi)\mapsto Y_t^{x,\xi}$ according to
\begin{align*}
    W_2^2(\mathcal L^1(Y^{\xi_1}_s,p^{\xi_1}_s),&\mathcal L^1(Y^{\xi_2}_s,p^{\xi_2}_s)) \le \E^1 \left[ \Vert Y_s^{\xi_1',\xi_1}-Y^{\xi_2',\xi_2}_s\Vert^2+\Vert p^{\xi_1',\xi_1}_s-p^{\xi_2',\xi_2}_s \Vert^2\right]\\&\le2\Lambda_Y^2W^2_2(\mathcal L^1(\xi_1),\mathcal L^1(\xi_2))+\E^1 \left[ 2\Lambda_Y^2\Vert\xi_1'-\xi_2' \Vert^2+\Vert p^{\xi_1',\xi_2}_s-p^{\xi_2',\xi_2}_s \Vert^2\right].
\end{align*}
Combining this with the preceding bound  and applying the Lipschitz continuity of $(x,\xi)\mapsto\bm Y^{x,\xi}$ implies that there exists some constant $C'$ depending only on $\Lambda^{R_X\vee R_a}_{\mathcal K},\Lambda_Y,\Lambda_\ell(R_X)$ such that 
\begin{align*}
    \E^1\left[\left\Vert p^{\xi_1',\xi_1}_t-p_t^{\xi_2',\xi_2} \right\Vert^2\right]&\le C' \left( W^2_2(\mathcal L^1(\xi_1),\mathcal L^1(\xi_2))+ \E^1[\Vert \xi_1'-\xi_2' \Vert^2]+\int_t^1 \E^1\left[\left\Vert p^{\xi_1',\xi_1}_s-p_s^{\xi_2',\xi_2} \right\Vert^2 \right]ds\right).
\end{align*}
By applying Gr\"onwall's lemma, we get
\[
\sup_{t\in[0,1]}\E^1\left[\left\Vert p^{\xi_1',\xi_1}_t-p_t^{\xi_2',\xi_2} \right\Vert^2 \right]\le C'e^{C'}\left( W_2^2(\mathcal L^1(\xi_1),\mathcal L^1(\xi_2))+ \E^1\left[\Vert \xi_1'-\xi_2'\Vert^2\right])\right).
\]
Combining this with the bound on $\E^1\left[\Vert Y^{\xi_1',\xi_1}_t-Y^{\xi_2',\xi_2}\Vert^2\right] $ seen in \eqref{eqn:EYxixi}, we obtain
\begin{align*}
   W_2^2\left(\mathcal L^1(Y^{\xi_1}_t,p^{\xi_2}_t), \mathcal L^1(Y^{\xi_2}_t,p^{\xi_2}_t) \right)&\le  \E^1\left[ \Vert Y^{\xi_1',\xi_1}_t-Y^{\xi_2',\xi_2}_t \Vert^2 +\Vert p^{\xi_1',\xi_1}_t-p^{\xi_2',\xi_2}_t \Vert^2 \right] \\& \le (2\Lambda_Y^2+C'e^{C'})\left(W_2^2(\mathcal L^1(\xi_1),\mathcal L^1(\xi_2))+ \E^1\left[\Vert \xi_1'-\xi_2'\Vert^2\right] \right).
\end{align*} Again, by taking the infimum over the couplings $\xi_1',\xi_2'\in L^2(\Omega^1)$ with $\mathcal L^1(\xi_i')=\mathcal L^1(\xi_i)$, $i=1,2$, we find that 
\[
\sup_{t\in[0,1]}W_2\left(\mathcal L^1(Y^{\xi_1}_t,p^{\xi_2}_t), \mathcal L^1(Y^{\xi_2}_t,p^{\xi_2}_t) \right)\le (4\Lambda_Y^2+2C'e^{C'})^{1/2}W_2(\mathcal L^1(\xi_1),\mathcal L^1(\xi_2)),
\]
holds $\mathbb P^0$-a.s. for any $\xi_1,\xi_2\in L^\infty(\Omega^1;\bar B(R_0))$. Note that \eqref{eqn:differeince_in_p} holds with $x_1,x_2$ in place of $\xi_1',\xi_2'$, without requiring the expectation with respect to $\mathbb P^1$. Therefore, the above estimates yield that there exists some constant $C''$  depending only on $\Lambda^{R_X\vee R_a}_{\mathcal K},\Lambda_Y,\Lambda_\ell(R_X)$ such that
\[
\left\Vert p^{x_1,\xi_1}_t-p^{x_2,\xi_2}_t\right\Vert^2\le C'' \left( \Vert x_1-x_2\Vert^2+ W^2_2(\mathcal L^1(\xi_1),\mathcal L^1(\xi_2))+\int_t^1\Vert p^{x_1,\xi_1}_s-p^{x_2,\xi_2}\Vert^2ds\right).
\]
The desired result now holds after applying Gr\"onwall's lemma.
\end{proof}
Let $\bm{y}=(y_1,\ldots,y_N)$ with  $y_i\in \bar B(R_0)$ for every $1\le i \le N$.
Suppose that $\xi^N\in L^\infty(\Omega^1;\R^d)$ is a random variable that takes the values $y_1,\ldots, y_N$ each with probability $N^{-1}$, so that its law is given by
\[
\zeta^N:=\mathcal L^1(\xi^N) = \frac{1}{N}\sum_{i=1}^N\delta_{y_i}.
\]
Accordingly, the law of $(X^{x,\zeta^N}_t,a^{x,\zeta^N}_t)_{x=\xi^N}$ conditioned on $\mathcal F^0$ is the empirical measure
\[
\mathfrak{r}^N_t:= \frac{1}{N}\sum_{i=1}^N \delta_{\left(X^{y_i,\zeta^N}_t,\,a^{y_i,\zeta^N}_t\right)}.\]
Then, by the uniqueness of the solutions to \eqref{eqn:MVODE1}-\eqref{eqn:MVODE2} established in Lemma \ref{lemma:MVRODE_wd}, $(\bm X^{x,\zeta^N},\bm a^{x,\zeta^N})|_{x=\xi^N}$ equals to $(\bm Y^{\xi^N},\bm p^{\xi^N})$ $\mathbb P$-almost surely. By Lemma 2.4 in \cite{carmona2018probabilistic} Vol. II, the conditional laws $\mathfrak{r}^N_t$ and $\mathcal L^1(Y^{\xi^N}_t,p^{\xi^N}_t)$
 coincide  $\mathbb P^0$-a.s., and so we can replace $\mathcal L^1(Y^{\xi^N}_t,p^{\xi^N}_t)$ by $\mathfrak{r}^N_t$ in \eqref{eqn:FMY}-\eqref{eqn:FMp}. This yields our continuous-time model \eqref{eqn:ansatz}-\eqref{eqn:anstaz_IC}, when $\bm\nu$ in Corollary \eqref{cor:FlowMap_wd} is taken to be the mean-field parameters $\bm\nu_\tau$. Hence, there exists a solution to \eqref{eqn:ansatz}-\eqref{eqn:anstaz_IC} provided that $\bm\nu_\tau$ satisfies the assumptions of Corollary \eqref{cor:FlowMap_wd}.

 Furthermore, the uniqueness of the solution to \eqref{eqn:ansatz}-\eqref{eqn:anstaz_IC} follows by a standard Gr\"onwall argument from the local Lipchitz continuity and linear growth of $\Gamma,\mathcal K$, Lemmas \ref{lemma:Lipderivs} and \ref{lemma:BoundedGamma}, and the local Lipschitz continuity of $\partial_\mu\ell$, Assumption \ref{assm:Obj}. Therefore, we obtain that
\begin{equation}\label{eqn:IPSvsFM}
\left(X^{i,N}_{t,\tau},a^{i,N}_{t,\tau}\right) = \left(X^{y_i,\zeta^N}_{t}, a^{y_i,\zeta^N}_t \right), \qquad \mathbb P^0-a.s.,
\end{equation}
for every $t\in[0,1]$ when $\bm{\nu}$ in \eqref{eqn:MVRODE1}-\eqref{eqn:MVODE2} is taken as $\bm{\nu}_\tau$.\\

This is made precise and extended as follows.
\begin{lemma}\label{lemma:InductiveStep}
    For training step $\tau \ge 0$, suppose that mean-field parameters $\bm{\nu}_\tau$ satisfy the conditions of Lemma \ref{lemma:MVRODE_wd}, for some $R_2>0$. Then, for any initial condition $\bm{y}\in (\bar B(R_0))^N$, there exists a solution $(X^{i,\bm{y}}_{t,\tau},a^{i,\bm{y}}_{t,\tau})$ to \eqref{eqn:ansatz}-\eqref{eqn:anstaz_IC}, unique up to $\mathbb P^0$-a.s. equality, such that $(X^{i,\bm{y}}_{t,\tau},a^{i,\bm{y}}_{t,\tau})$ are $\mathcal F^0$-measurable for every $t\in[0,1]$. Furthermore, these solutions satisfy
    $X^{i,\bm{y}}_{t,\tau}\in\bar B(R_X)$ and $a^{i,\bm{y}}_{t,\tau}\in \bar B(R_a)$ $\mathbb P^0$-a.s., where $R_X,R_a$ are the constants from Lemma \ref{lemma:MVODE_WellDEfined}. Furthermore, given training data $(\bm{y}_{b,\tau})_{b=1}^B$ that is $\mathcal F^0$-measurable, the gradient map $g_{t,\tau+1}[\bm\nu_\tau](\theta)$, given by \eqref{eqn:UpdateG}, is $\mathcal F^0$-measurable for any $\theta \in \R^{k\times 4d}$ and $t\in[0,1]$.
\end{lemma}
\begin{proof}
Let $(X_{t,\tau}^{x,\zeta},a^{x,\zeta}_{t,\tau})$ be the solution to \eqref{eqn:FMY}-\eqref{eqn:FMp} with mean-field parameters $\bm\nu_\tau$. This solution exists by Remark \ref{remark:flow} and Corollary \ref{cor:FlowMap_wd}, whose conditions are satisfied by $\bm\nu_\tau$ by assumption. 
Let $\zeta^N$ be the empirical measure on $(y^i)_{i=1}^N$
Then, \eqref{eqn:IPSvsFM} identifies $X^{i,\bm{y}}_{t,\tau},a^{i,\bm{y}}_{t,\tau}$ with $X_t^{x,\zeta},a^{x,\zeta}_t$ evaluated at $x=y^i$ and $\zeta=\zeta^N$. Therefore, the claim on the boundedness, uniqueness and measurability of $X^{i,\bm{y}}_{t,\tau},a^{i,\bm{y}}_{t,\tau}$ now follows immediately from Corollary \ref{cor:FlowMap_wd}. \\

Let $\zeta^N_{b,\tau}$ be the empirical measure on the training data $y_{b,\tau}^1,\ldots,y_{b,\tau}^N$. Since $(x,\zeta)\mapsto (X_t^{x,\zeta}, a^{x,\zeta}_t)$ is jointly continuous, after modifying it on a $\mathbb P^0$-null set if necessary, and $(y^i_{b,\tau})_{i=1}^N$ and $ \zeta^N_{b,\tau}$ are $\mathcal F^0$-measurable, it follows by the composition of continuous and measurable maps that $X^{i,b}_{t,\tau},a^{i,b}_{t,\tau}$ are $\mathcal F^0$-measurable for every $t\in[0,1]$.
Similarly, we find that, for any deterministic $\theta\in \R^{k\times 4d}$, $\partial_\nu\mathcal H(X^{i,b}_{t,\tau},\mu^{N,b}_{t,\tau},\nu_{t,\tau},a^{i,b}_{t,\tau})(\theta)$ is $\mathcal F^0$-measurable as $\partial_\nu\mathcal H$ is jointly continuous in each argument, see Lemma \ref{lemma:LipmathcalHnu}, and $X^{i,b}_{t,\tau},\mu^{N,b}_{t,\tau},a^{i,b}_{t,\tau}$ and $\nu_{t,\tau}$ are $\mathcal F^0$-measurable. The gradient $g_{t,\tau+1}[\bm\nu_\tau](\theta)$, given in equation \eqref{eqn:UpdateG}, is the average of $\partial_\nu\mathcal H(X^{i,b}_{t,\tau},\mu^{N,b}_{t,\tau},\nu_{t,\tau},a^{i,b}_{t,\tau})(\theta)$ over particles $1\le i\le N$ and batches $1\le b\le B$. Hence, $g_{t,\tau+1}[\bm\nu_\tau](\theta)$ is $\mathcal F^0$-measurable for any $\theta\in \R^{k\times 4d}$ and $t\in[0,1]$.

\end{proof}

 Lemma \ref{lemma:BoundDelta} states that the infinity norm on $\mathcal R$ applied to the updates of Algorithm \ref{alg:AdamW} are bounded by some constant $B_\beta=B_\beta(\beta_1,\beta_2)$. As a result, when $\Vert \mathcal R(\theta^h_{r,\tau})\Vert_\infty$ exceeds $B_\beta\lambda^{-1}$, the weight-decay term dominates and forces $\Vert \mathcal R(\theta^{h}_{r,\tau+1})\Vert_\infty$ to be smaller than $\Vert \mathcal R(\theta^{h}_{r,\tau})\Vert_\infty$. In this way, we can find a compact set with which the support of $\nu_{t,\tau}$ is common to all $t\in[0,1], \tau\ge 0$. To compensate for the different choice of $\mathcal R$ in Algorithm \ref{alg:AdamW}, introduce
\[
c_\mathcal R=\begin{cases}1 &\mathcal R=\mathcal R_1,\\
\sqrt{dk} &\mathcal R=\mathrm{Id},
\end{cases}
\]
where $\mathrm{Id}$ is the identity map. This constant is chosen so that any $\theta \in \R^{k\times 4d}$ that satisfies \\$\Vert \mathcal R(\theta)\Vert_\infty\le R$ is an element of $\mathfrak{S}(c_\mathcal RR)$ for each of our choices of $\mathcal R$.
\begin{lemma}\label{lemma:GlobalWellDefined}
Suppose that Assumptions \ref{assm:Init}-\ref{assm:TD2} hold and that the parameters $\nu_{t,\tau},\bar\nu^{N,L}_{t,\tau}$ are updated according to Algorithm \ref{alg:AdamW} with $\mathcal R$ either the identity map or given by equation \eqref{eqn:R}. Suppose that Algorithm \ref{alg:AdamW} is performed with weight decay $\lambda>0$, and step sizes satisfying $\eta_k \in (0,\lambda^{-1})$. Then, there exists a constant $R_\theta>0$ such that the supports of $\nu_{t,\tau}$ and $\bar\nu^{N,L}_{t,\tau}$  are contained in $\mathfrak{S}(R_\theta)$ for any training step $\tau$, $\mathbb P^0$-a.s. Specifically $R_\theta=c_\mathcal{R} B_\beta\lambda^{-1}$, where $B_\beta$ is defined in Lemma \ref{lemma:BoundDelta}. Also, the map $(\omega^0,t)\mapsto\nu_{t,\tau}(\omega^0)$ is continuous for $\mathbb P^0$-a.e. $\omega^0$, and $\mathcal F^0$-measurable for $t\in[0,1]$. Similarly, $\Omega\ni\omega\mapsto \bar \nu^{N,L}_{t}$ is $\mathcal F$-measurable for every $t\in[0,1]$.

Furthermore, there exists a solution to the system of equations \eqref{eqn:ansatz}-\eqref{eqn:anstaz_IC}, unique up to $\mathbb P^0$-a.s. equality, given any initial condition $\bm y\in (\bar B(R_0))^N$. For every $t\in[0,1]$ and $\tau\ge 0$, these solutions $(X^{i,\bm{y}}_{t,\tau},a^{i,\bm{y}}_{t,\tau})$ are $\mathcal F^0$-measurable and $\mathbb P^0$-a.s. satisfy $X^{i,\bm{y}}_{t,\tau}\in \bar B(R_X)$, $a^{i,\bm{y}}_{t,\tau}\in \bar B(R_a)$ for every $1\le i \le N$, where $R_X,R_a$ are the constants of Lemma \ref{lemma:MVODE_WellDEfined} with $R_2=R_\theta$. Finally, the solutions to the discrete-time IPS \eqref{eqn:CTE1}-\eqref{eqn:CTE2}, denoted by $\bar X^{i,L,\bm{y}}_{t,\tau}, \bar a^{i,L,\bm{y}}_{t,\tau}$ are $\mathcal F$-measurable and satisfy the same bounds $\mathbb P$-a.s.
\end{lemma}
\begin{proof}
Let us first consider the continuous-time model.\\

\underline{Base Case}: The support of $\nu_{t,0}=\pi$ is contained in $\{\theta: \Vert \mathcal R(\theta)\Vert_\infty\le \lambda^{-1}\}$ by Assumption \ref{assm:Init}. The constant $c_\mathcal{R}$ is chosen so that $\Vert \mathcal R(\theta)\Vert_\infty\le \lambda^{-1}$ implies $\theta\in \mathfrak{S}(c_\mathcal R\lambda^{-1})$. As a result,  $supp\, \pi\subseteq\mathfrak{S}(c_\mathcal R\lambda^{-1})\subseteq \mathfrak{S}(R_\theta)$, where $R_\theta =B_\beta c_\mathcal R\lambda^{-1}$, since $B_\beta>1$, as seen in Lemma \ref{lemma:BoundDelta}.  Also, $\nu_{t,0}=\pi$ is deterministic and constant in $t$, and so satisfies the measurability and continuity requirements. Finally, $\Phi_{t,0}[\bm\nu](\theta)=\theta$ is deterministic and thus measurable.

\underline{Inductive Hypothesis}: Suppose for some training step $\tau\ge 0$ that $\bm\nu_j$ satisfies the Assumptions of Lemma \ref{lemma:MVRODE_wd} with $R_2=R_\theta$ and that $ \Phi_{t,j}[\bm\nu](\theta)$ is $\mathcal F^0$-measurable for every $j\le\tau, \theta \in \R^{k\times 4d}$ and $t\in[0,1]$. \\

Lemma \ref{lemma:InductiveStep} applies, with $R_2=R_\theta$, and guarantees that the desired results for $\bm X^{i,\bm{y}}_{\tau},\bm a^{i,\bm{y}}_{\tau}$ hold. Furthermore, it states that $g_{t,j+1}[\bm{\nu}_j](\theta)$ is $\mathcal F^0$-measurable for every $\theta \in \R^{k\times 4d}$ and $j\le \tau$. Since the map $\theta\mapsto g_{t,j+1}[\bm{\nu}_j](\theta)$ is continuous, by replacing $\bm{\bar\nu}^{N,L}$ by $\bm\nu$ in Lemma \ref{lemma:GradCont},  $g_{t,j+1}[\bm{\nu}_j](\vartheta)$ is $\mathcal F^0$-measurable for every $\mathcal F^0$-measurable $\vartheta$. Specifically, $g_{t,j+1}[\bm{\nu}_j](\Phi_{t,j}[\bm\nu](\theta))$ is $\mathcal F^0$-measurable for every $j\le \tau$ and $t\in[0,1]$. The variance and momentum accumulators are linear combinations of these gradients and are therefore $\mathcal F^0$-measurable. The denominators of the Adam updates are bounded away from zero, and so  $\hat m_j/(\sqrt{\hat v_j}+\varepsilon)$ are $\mathcal F^0$-measurable for every $j\le\tau+1$. As a result, $\Phi_{t,\tau+1}[\bm\nu](\theta)$ is $\mathcal F^0$-measurable. Moreover, Lemma \ref{lemma:AdamLip}, whose conditions are satisfied according to Lemma \ref{lemma:GradCont}, implies that $\theta\mapsto\Phi_{t,\tau+1}[\bm\nu](\theta)$ is continuous, and thus $(\omega^0,\theta)\mapsto \Phi_{t,\tau+1}[\bm\nu](\omega^0)(\theta)$ is $(\mathcal F^0\otimes \mathcal B(\R^{k\times 4d}))$-measurable by Theorem 4.1 in \cite{AZHMYAKOV201987}. Then, consider the map given by
\[
\omega^0\mapsto \nu_{t,\tau+1}(\omega^0)(D) = \int \mathbbm{1}_{D}(\Phi_{t,\tau+1}[\bm\nu](\omega^0)(\theta))\,d\pi(\theta).
\]
It is $\mathcal F^0$-measurable for every $D \in \mathcal B(\R^{k\times 4d})$ by Lemma 1.28 in \cite{Kallenberg2021}. Therefore, $\omega^0\mapsto \nu_{t,\tau+1}(\omega^0)$ is $\mathcal F^0$-measurable for every $t\in[0,1]$ by Criterion \ref{criterion:measure}.
 Lemma \ref{lemma:BoundDelta} in Appendix \ref{App:K2} states that the infinity norm on $\mathcal R$ applied to the Adam updates is bounded by the constant $B_\beta$, which only depends on $\beta_1,\beta_2$. Furthermore, if $\mathcal R$ is the identity map or given by equation \eqref{eqn:R}, then $\mathcal R$ is $1$-homogeneous and sub-additive in each component. Applying these properties gives
\begin{align*}
   \Vert \mathcal R( \Phi_{t,\tau+1}[\nu_{<\tau+1}](\theta))\Vert_\infty&\le \alpha_{1,\tau}\Vert \mathcal R(\theta)\Vert_\infty +\sum_{i=1}^{k+1} \eta_{i}\alpha_{i+1,\tau+1}\left\Vert \mathcal R\left( \frac{\hat m^\mathcal{R}_i}{\sqrt{\hat v^\mathcal R_i}+\varepsilon}\right)\right\Vert_\infty\\
   & \le\alpha_{1,\tau}\Vert \mathcal R(\theta)\Vert_\infty +B_\beta\sum_{i=1}^{k+1}\eta_{i}\alpha_{i+1,\tau+1}.
\end{align*}
Notice that $(1-\eta_i\lambda)\alpha_{i+1,\tau+1}=\alpha_{i,\tau+1}$, which applied to the above bound produces the telescoping sum 
\begin{align*}
   \Vert \mathcal R( \Phi_{t,\tau+1}[\bm\nu_{<\tau+1}](\theta))\Vert_\infty& \le \alpha_{1,\tau}\Vert \mathcal R(\theta)\Vert_\infty +\lambda^{-1}B_\beta\sum_{i=1}^{k+1} \alpha_{i+1,\tau+1}-\alpha_{i,\tau+1}\\
   & \le\alpha_{1,\tau}\left(\Vert \mathcal R(\theta)\Vert_\infty -\lambda^{-1}B_\beta\right)+\lambda^{-1}B_\beta.
\end{align*}
As a result, if $\theta$ belongs to the closed set $S$ defined by
\[
S=\{\theta \in \R^{k\times 4d}:\Vert \mathcal R(\theta)\Vert_\infty \le B_\beta\lambda^{-1}\},
\]
then $\Phi_{t,\tau+1}[\bm\nu](\theta)\in S$. This implies that $\Phi^{-1}_{t,\tau+1}[\bm\nu](S^c)\subseteq S^c$. Therefore, $S^c$ is an open set such that $$\nu_{t,\tau+1}(S^c)=\nu_{t,0}(\Phi^{-1}_{t,\tau+1}[\bm\nu_{<\tau+1}](S^c))\le \nu_{t,0}(S^c)=0, $$ since the support of $\nu_{t,0}$ is contained within $S$ by Assumption \ref{assm:Init}. Hence, it must be that the support of $\nu_{t,\tau+1}$ is contained within $S \subseteq\mathfrak{S}(R_\theta)$. The inductive step is completed by applying Lemma \ref{lemma:Lipt} to show that $t\mapsto \nu_{t,\tau+1}$ is continuous $\mathbb P^0$-a.s.\\

Note that the above analysis bounding $\Phi[\bm\nu]$ requires only that the sequence of mean-field parameters $\bm\nu_\tau$
is generated according to Algorithm \ref{alg:AdamW}. Accordingly, $supp\,\bar \nu^{N,L}_{t,\tau}\subseteq \mathfrak{S}(R_\theta)$ for every $t\in[0,1]$, and $\tau\ge 0$. Furthermore, $(\bar X^{i,L,\bm{y}}_{t,\tau},\bar a^{i,L,\bm{y}}_{t,\tau})\in \bar B(R_X)\times \bar B(R_a)$ for every $t\in[0,1],\tau\ge 0$ and $\bm{y}\in (\bar B(R_0))^N$, $\mathbb P$-almost surely, by adapting the Gr\"onwall argument in the proof of Lemma \ref{lemma:MVODE_WellDEfined} to the discrete-time setting. If the parameters $\bm{\bar\nu}^{N,L}_j$ are $\mathcal F$-measurable for every $j \le \tau$, then the $(\bar X^{i,L,b}_{t,j},\bar a^{i,L,b}_{t,j})$ are $\mathcal F$-measurable for every $t\in[0,1]$ and $j\le \tau$. This is because the updates defining $(\bar X^{i,L,b}_{t,j},\bar a^{i,L,b}_{t,j})$ are continuous functions of the $\mathcal F$-measurable parameters and the previous states. Accordingly, $g_{j+1}[\bm{\nu}^{N,L}_j](\theta)$ is $\mathcal F$-measurable for every $\theta\in \R^{k\times 4d}$ and $j\le \tau$, by the same argument as Lemma \ref{lemma:InductiveStep}. By Lemma \ref{lemma:LipmathcalHnu}, the gradient map $g_{j+1}[\bm{\nu}^{N,L}_j](\theta)$  is continuous in $\theta$, as it is a linear combination of $\partial_\nu\mathcal H$. Therefore, by applying the same argument as above, $\bar{\nu}^{N,L}_{t,\tau+1}$ is $\mathcal F$-measurable for every $t\in[0,1]$. Hence, by induction, $\bar{\nu}^{N,L}_{t,\tau}$ is $\mathcal F$-measurable  for every $t\in[0,1]$ and $\tau\ge 0$. The base case holds since the initialised parameters are $\mathcal F$-measurable.
\end{proof}
\begin{remark}\label{remark:GlobalLip}
Lemma \ref{lemma:GlobalWellDefined} establishes that the trajectories of $X^{i,\bm{y}}_{t,\tau}, a^{i,\bm{y}}_{t,\tau}$, and the parameters $\nu_{t,\tau}$ remain confined to a compact set that is uniform across all training steps. The same is true for the corresponding discrete-time variables. Consequently, the local Lipschitz constants and bounds obtained in Appendix \ref{App:D} apply globally on the relevant domain when evaluated at $R_1=R_X, R_2=R_\theta$ and $R_3=R_a$. For the remainder of the paper, we take all of these constants to be evaluated at these radii and suppress their explicit dependence on $R_X,R_\theta,R_a$ for notational convenience. Similarly, we shall write $\mathfrak{S}$ to denote $\mathfrak{S}(R_\theta)$.
\end{remark}
    \section{Regularity of the AdamW Flow Map $\Phi[\nu]$}\label{App:F}
\begin{lemma}\label{lemma:AdamLip}
Suppose that there exists a constant $D>0$ such that for any sequence of mean-field parameters $\bm\nu=(\nu_j)_{j=0}^T$ and $\bm{\bar \nu}=(\bar\nu_j)_{j=0}^T$, and  $\theta, \vartheta\in \mathfrak{S}$, the gradient map defined in \eqref{eqn:UpdateG} satisfies 
\[
 \Vert g_{t,\tau+1}[\nu_\tau](\theta) -  g_{t,\tau+1}[\bar \nu_\tau]( \vartheta)\Vert^2_\mathrm{F}  \le D\left(A_{t,\tau}+ \Vert \theta-\vartheta \Vert^2_\mathrm{F} \right),
\]
 for some non-negative $A_{t,\tau}$ at every training step $\tau \in[0:T-1]$. Then, there exists a constant $D_1=D_1(D,\beta_1,\beta_2,\varepsilon,\lambda,T)$ such that for any $\tau\in [1: T]$
        \[
 \Vert\Phi_{t,\tau}[\bm\nu](\theta)-\Phi_{t,\tau}[\bm{\bar\nu}](\vartheta)\Vert^2_\mathrm{F} \le  D_1\left(\Vert \theta-\vartheta\Vert^2_\mathrm{F}+\sum_{j=0}^{T-1}\kappa^0_{j,T}A_{t,j}\right),
    \]
where $\kappa^0_{i,T}>0$ is given by
\begin{align}\label{eqn:kappa}
    \kappa^0_{i,T} &=\sum_{\tau=i+1}^T\frac{(1-\beta_1)\eta_\tau}{1-\beta_1^\tau}(\beta_1^{\tau-i-1}+2\beta_2^{\tau-i-1}).
\end{align}
\end{lemma}
\begin{proof}
For ease of notation, write
\[
\theta_{t,T} = \Phi_{t,T}[\bm\nu](\theta), \qquad  \vartheta_{t,T}=\Phi_{t,T}[\bm{\bar\nu}](\vartheta).
\]
By applying the Cauchy-Schwarz inequality on the definition of the flow map \eqref{eqn:AdamWFlowMap}, we get
\begin{align*}
     \Vert\theta_{t,T}-\vartheta_{t,T}\Vert^2_\mathrm{F}&\le2\alpha^2_{1,T}\Vert \theta-\vartheta\Vert^2_\mathrm{F} \\& +2\left(\sum_{\tau=1}^T \eta_\tau \alpha_{\tau+1,T} \right)\left(\sum_{\tau=1}^T \eta_\tau \alpha_{\tau+1,T} \left\Vert\frac{\hat m^\mathcal{R}_\tau(G_t[\bm\nu])}{\sqrt{\hat v^\mathcal{R}_\tau(G_t[\bm\nu])}+\varepsilon}-\frac{\hat m^\mathcal{R}_\tau(G_t[\bm{\bar\nu}])}{\sqrt{\hat v^\mathcal{R}_\tau(G_t[\bm{\bar\nu}])}+\varepsilon}\right\Vert^2_\mathrm{F}\right).
\end{align*}
Notice that $\lambda \eta_\tau \alpha_{\tau+1,T}=\alpha_{\tau+1,T}-\alpha_{\tau,T}$, and so the first term becomes the telescoping sum
\[
\sum_{\tau=1}^T \eta_\tau \alpha_{\tau+1,T} = \lambda^{-1}\sum_{\tau=1}^T\alpha_{\tau+1,T}-\alpha_{\tau,T}\le \lambda^{-1}.
\]
Lemma \ref{lemma:UpdateStability} establishes the Lipschitz continuity of the Adam updates with respect to the gradient sequence $G_t$, which states that
\begin{align*}
    \left\Vert\frac{\hat m^{\mathcal R}_\tau(G_t[\nu])}{\sqrt{\hat v^{\mathcal R}_\tau(G_t[\nu])}+\varepsilon}-\frac{\hat m^{\mathcal R}_\tau(G_t[\bar\nu])}{\sqrt{\hat v^{\mathcal R}_\tau(G_t[\bar\nu])}+\varepsilon}\right\Vert_F^2 
\le \frac{2(1-\beta_1)}{\varepsilon^2(1-\beta_1^\tau)}\sum_{j=1}^\tau(\beta_1^{\tau-j}+2\beta_2^{\tau-j})\Vert g_{t,j}-\bar g_{t,j}\Vert_\mathrm{F}^2,
\end{align*}
where we use the short hand notation $g_{t,j}=g_{t,j}[\nu_{j-1}](\theta_{t,j-1})$ and $\bar g_{i,j}=g_{t,j}[\bar\nu_{j-1}](\vartheta_{t,j-1})$. Therefore, after interchanging the order of summation, we obtain
\begin{align*}
     \Vert\theta_{t,T}-\vartheta_{t,T}\Vert^2_\mathrm{F}  &\le 2\alpha^2_{1,T}\Vert \theta-\vartheta\Vert^2_\mathrm{F}+\sum_{\tau=1}^T \frac{4\eta_\tau\alpha_{\tau+1,T} (1-\beta_1)}{\lambda \varepsilon^2 (1-\beta_1^\tau)}\sum_{i=1}^\tau (\beta_1^{\tau-i}+2\beta_2^{\tau-i})\Vert g_{t,i}-\bar g_{t,i}\Vert^2_\mathrm{F}\\
     & \le 2\alpha^2_{1,T}\Vert \theta-\vartheta\Vert^2_\mathrm{F}+\frac{4}{\lambda \varepsilon^2}\sum_{i=1}^T\left(\sum_{\tau=i}^T  \frac{\eta_\tau\alpha_{\tau+1,T} (1-\beta_1)}{ (1-\beta_1^\tau)}(\beta_1^{\tau-i}+2\beta_2^{\tau-i})\right) \Vert g_{t,i}-\bar g_{t,i}\Vert^2_\mathrm{F}.
\end{align*}
Accordingly, define $\kappa^\lambda_{i,T}$, for $0\le i < T$, by
\begin{align}\label{eqn:kappa_}
    \kappa^\lambda_{i,T} &=\sum_{\tau=i+1}^T\frac{(1-\beta_1)\eta_\tau\alpha_{\tau+1,T}}{1-\beta_1^\tau}(\beta_1^{\tau-i-1}+2\beta_2^{\tau-i-1}),
\end{align}
so that 
\begin{align}\label{eqn:DeltaTheta}
     \Vert\theta_{t,T}-\vartheta_{t,T}\Vert^2_\mathrm{F}  &\le 2\alpha^2_{1,T}\Vert \theta-\vartheta\Vert^2_\mathrm{F}+\frac{4}{\lambda\varepsilon^2} \sum_{i=1}^{T} \kappa^\lambda_{i-1,T} \Vert g_{t,i}[\nu_{i-1}](\theta_{t,i-1})- g_{t,i}[\bar \nu_{i-1}](\vartheta_{t,i-1})\Vert^2_\mathrm{F}.
\end{align}
Then, the assumed continuity of $g$ implies that
\begin{align*}
     \Vert\Phi_{t,T}[\nu](\theta)-\Phi_{t,T}[\bar\nu](\vartheta)\Vert^2_\mathrm{F} &\le 2\alpha_{1,T}\Vert \theta-\vartheta\Vert^2_\mathrm{F}+ \frac{4D}{\lambda\varepsilon^2}\sum_{j=0}^{T-1}\kappa^\lambda_{j,T}\left( A_{t,j} + \Vert\Phi_{t,j}[\bm\nu](\theta)-\Phi_{t,j}[\bm{\bar\nu}](\theta)\Vert^2_\mathrm{F}\right),
\end{align*}
 where we have applied $\alpha_{1,T}^2\le \alpha_{1,T}$ as each $(1-\eta_j\lambda)\in(0,1)$. Using the convention $\sum_{j=0}^{-1}\kappa^\lambda_{0,0}A_{t,0}=0$, let us define the sequence
\[
c_T = \alpha_{1,T}\Vert \theta-\vartheta\Vert^2_\mathrm{F}+\sum_{j=0}^{T-1}\kappa^\lambda_{j,T} A_{t,j},
\]
which is non-negative and $c_0=\Vert \theta-\vartheta\Vert_\mathrm{F}^2$. Since $\kappa^\lambda_{t,T}$ satisfies 
\[
\kappa^\lambda_{i,T}=(1-\eta_T\lambda)\kappa^\lambda_{i,T-1}+ \frac{(1-\beta_1)\eta_T}{1-\beta_1^T}(\beta_1^{T-i-1}+2\beta_2^{T-i-1}), \qquad \kappa^\lambda_{T-1,T}=\frac{3(1-\beta_1)\eta_T}{1-\beta_1^T},
\]
for $0\le i \le T-2$, the sequence $c_T$ satisfies
\[
c_T=(1-\eta_{T}\lambda)c_{T-1}+\sum_{j=0}^{T-1}\frac{1-\beta_1}{1-\beta_1^T}\eta_T(\beta_1^{T-j-1}+2\beta_2^{T-j-1}) A_{t,j}.
\]
As a result, the conditions for Lemma \ref{lemma:RecursiveSum} hold since $c_j\ge (1-\eta_T\lambda)c_{j-1}$ for $j\ge 1$, and so there exists a constant $D_1=D_1(\lambda,\beta_1,\beta_2,T)$ such that
\[
\Vert\Phi_{t,T}[\nu](\theta)-\Phi_{t,T}[\bar\nu](\vartheta)\Vert^2_\mathrm{F}\le D_1\left( \Vert \theta-\vartheta\Vert_\mathrm{F}^2+\sum_{i=1}^T\sum_{j=0}^{i-1}\frac{1-\beta_1}{1-\beta_1^i}\eta_i(\beta_1^{i-j-1}+2\beta_2^{i-j-1}) A_{t,j}\right).
\]
The Lemma concludes by interchanging the order of summation and substituting in the definition of $\kappa^{0}_{j,T}$ given in \eqref{eqn:kappa}.
\end{proof}
\begin{lemma}\label{lemma:Lipt}
\textbf{Continuity of $\nu_t$ in time}\\
Let $\nu_{t,\tau}$ be the parameters of \eqref{eqn:nu} at training step $\tau$.
Then, there exists a constant \\$D_2=D_2(R_X,R_a,R_\theta,\varepsilon,\beta_1,\beta_2,\lambda,T)$ such that for every training step $\tau \in[0:T]$ and $s,t\in[0,1]$
\[
 W_2(\nu_{t,\tau}, \nu_{s,\tau})\le D_2|t-s|, \qquad \mathbb P^0-a.s.
\]
\end{lemma}
\begin{proof}
For any $\theta ,\vartheta\in \R^{k\times 4d}$, expanding the gradient map of  \eqref{eqn:UpdateG} by the triangle inequality yields
\begin{align*}
    \Vert g_{t,\tau+1}&[\bm\nu_\tau](\theta)-g_{s,\tau+1}[\bm{\nu}_\tau](\vartheta)\Vert_\mathrm{F} \\
    & \le \frac{1}{B}\sum_{b=1}^B\int \left\Vert \partial_\nu\mathcal H(x,\mu^{N,b}_{t,\tau},\nu_{t,\tau},a)( \theta)-\partial_\nu\mathcal H(x,\mu^{N,b}_{s,\tau},\nu_{s,\tau},a)( \vartheta)\right\Vert_\mathrm{F}d\rho^{N,b}_{t,\tau}(x,a)\\
    & + \frac{1}{B}\sum_{b=1}^B\left\Vert \int \partial_\nu\mathcal H(x,\mu^{N,b}_{s,\tau},\nu_{s,\tau},a)(\vartheta)d(\rho^{N,b}_{t,\tau}-\rho^{N,b}_{s,\tau})(x,a)\right\Vert_\mathrm{F}, \qquad \mathbb P^0-a.s.
\end{align*}
Lemma \ref{lemma:LipmathcalHnu} states that $\partial_\nu\mathcal H$ is $\Lambda_{\mathcal H\nu,\infty}$-Lipschitz continuous in $\theta$ and $\mu$, with respect to the Frobenius norm and the $W_\infty$-metric, respectively. Since $\mathcal H$ is linear in $\nu$, its derivative $\partial_\nu\mathcal H$ with respect to $\nu$ is independent of $\nu$. Hence, the first term is bounded by \[
\Lambda_{\mathcal H\nu,\infty} \left(\Vert \theta-\vartheta\Vert_\mathrm{F}+W_\infty(\mu^{N,b}_{t,\tau},\mu^{N,b}_{s,\tau})\right).\]
Note that for any fixed $\mu,\theta,\nu$ with the support of $\nu$ and $\mu$ contained in $\mathfrak{S},\bar B(R_X)$ respectively and $\theta \in \mathfrak{S}$, the map $(x,a)\to \partial_\nu\mathcal H(x,\mu,\nu,a)(\theta)$ is $\Lambda_{\mathcal H\nu,\infty}$-Lipschitz continuous on $\bar B(R_X)\times \bar B(R_a)$, Lemma \ref{lemma:LipmathcalHnu}. Therefore, applying the Kantorovich-Rubinstein Theorem \cite{villani2008optimal} to the second term is bounded above by $\Lambda_{\mathcal H\nu,\infty}W_1(\rho^{N,b}_{t,\tau},\rho^{N,b}_{s,\tau})$. Combining the contributions from each term yields that
\begin{align*}
    \Vert g_{t,\tau+1}[\bm\nu_\tau](\theta)-g_{s,\tau+1}[\bm\nu_\tau](\vartheta)\Vert_\mathrm{F} 
    \le\frac{\Lambda_{\mathcal H\nu,\infty}}{B}\sum_{b=1}^B  W_\infty(\mu^{N,b}_{t,\tau}, \mu^{N,b}_{s,\tau})+\Vert \theta- \vartheta\Vert_\mathrm{F} +  W_1(\rho^{N,b}_{t,\tau},\rho^{N,b}_{s,\tau}).
\end{align*}
Immediately from the boundedness of $\Gamma$ and its derivatives, see Remark \ref{remark:GlobalLip} and Lemma \ref{lemma:BoundedGamma}, we get
\[
W_\infty(\mu^{N,b}_{t,\tau}, \mu^{N,b}_{s,\tau}) \le |t-s|R_XR_\theta^2, \qquad  W_1(\rho^{N,b}_{t,\tau},\rho^{N,b}_{s,\tau}) \le |t-s| (R_XR_\theta^2+R_a \tilde{B}_\mathcal{K}),
\]
where $\tilde{B}_\mathcal{K}$ is the constant from Lemma \ref{lemma:BoundedGamma} with $R_1,R_2$ evaluated at $R_X,R_\theta$, respectively. Therefore, for every $j<T$, the gradient map $\mathbb P^0$-a.s. satisfies the Lipschitz estimate
\begin{equation}\label{eqn:LipGt}
    \Vert g_{t,j+1}[\bm\nu_j](\theta)-g_{s,j+1}[\bm\nu_j](\bar\theta)\Vert_\mathrm{F}\le\Lambda_{\mathcal H\nu,\infty}\left(\Vert \theta-\vartheta\Vert_F+(R_a\tilde{B}_\mathcal{K}+2R_XR_\theta^2)||t-s|\right).
\end{equation}
Recall the notation $\Phi_{t,T}[\bm\nu](\theta)=\theta_{t,T}$ from Lemma \ref{lemma:AdamLip}. Then, applying the same argument as seen in the proof  of Lemma \ref{lemma:AdamLip} until equation \eqref{eqn:DeltaTheta}, with $\theta=\vartheta$, results in
    \begin{align*}
            \Vert\theta_{t,T}-\theta_{s,T}\Vert^2_\mathrm{F} &\le \frac{2}{\lambda\varepsilon^2}\sum_{j=1}^T \kappa^\lambda_{j-1,T} \Vert g_{t,j}[\nu](\theta_{t,j-1})-g_{s,j}[\nu](\theta_{s,j-1})\Vert^2_\mathrm{F}\\
 &\le \frac{4\Lambda_{\mathcal H\nu,\infty}^2}{\lambda\varepsilon^2}\sum_{j=0}^{T-1}\kappa^\lambda_{j,T} \left( (\tilde{B}_\mathcal{K} R_a+2R_XR_\theta^2)^2|t-s|^2+\Vert \theta_{t,j}-\theta_{s,j}\Vert^2_\mathrm{F}\right),
    \end{align*}
where $\kappa^\lambda_{j,T}$ is defined in \eqref{eqn:kappa_}.
The final step is deduced by Young's inequality and the bound given in \eqref{eqn:LipGt}. By applying Lemma \ref{prop:KappaSum} to bound the summation of $\kappa^\lambda_{j,T}$ in the first and then applying Lemma \ref{lemma:RecursiveSum}, there exists a constant $D_2$ such that
\[
\Vert \Phi_{t,T}[\bm\nu](\theta)-\Phi_{s,T}[\bm\nu](\theta)\Vert^2_\mathrm{F} \le D^2_2|t-s|^{2}.
\]
Hence, the result holds by integrating over $\theta$ with respect to $\pi$.
\end{proof}

\begin{lemma}\label{lemma:GradCont}
Let $g:\R^{k\times 4d}\to \R^{k\times 4d}$ be the gradient map defined in \eqref{eqn:UpdateG}. Then, there exists a constant $ D= D(\Lambda,R_a,R_X, R_\theta)$ such that for any $\tau \ge 0$ and $\theta ,\vartheta \in \mathfrak S$
\begin{align*}
    \Vert g_{t,\tau+1}[\bm{\bar\nu}^{N,L}_\tau](\vartheta)-g_{t,\tau+1}[\bm\nu_\tau](\theta)\Vert^2_\mathrm{F} \le  D\left(\Vert \theta- \vartheta\Vert^2_\mathrm{F} +\frac{1}{B}\sum_{b=1}^B W^2_{\infty}(\bar\rho^{N,L,b}_{t,\tau},\rho^{N,b}_{t,\tau})\right),
\end{align*}
$\mathbb P$-almost surely.
\end{lemma}
\begin{proof}
Take $\theta, \vartheta\in \mathfrak S$. By expanding the expression for the gradient map $g$ defined in \eqref{eqn:UpdateG} via Young's Inequality, we obtain
\begin{align}\label{eqn:Deltag2}
    &\Vert \bar g_{t,\tau+1}[\bm{\bar\nu}^{N,L}_\tau](\vartheta)-g_{t,\tau+1}[\bm\nu_\tau](\theta)\Vert^2_\mathrm{F} \\ &\le \frac{2}{B}\sum_{b=1}^B \underbrace{ \left\Vert\int \partial_\nu\mathcal H(x,\bar \mu^{N,L,b}_{t,\tau}, \bar\nu^{N,L}_{t,\tau},a)(\vartheta) - \partial_\nu\mathcal H(x,\mu^{N,b}_{t,\tau}, \nu_{t,\tau},a)(\theta) d \bar\rho^{N,L,b}_{t,\tau}(x,a) \right\Vert^2_\mathrm{F}}_{(a)}\notag\\
    &+\frac{2}{B}\sum_{b=1}^B \underbrace{\left\Vert\int \partial_\nu\mathcal H(x,\mu^{N,b}_{t,\tau}, \nu_{t,\tau},a)(\theta) d (\bar\rho^{N,L,b}_{t,\tau}-\rho^{N,b}_{t,\tau})(x,a) \right\Vert^2_\mathrm{F}}_{(b)}.\notag
\end{align} 
Notice that $\partial_\nu\mathcal H$ is independent of $\nu$, since $\mathcal H$ is linear in $\nu$.
Then, the $\Lambda_{\mathcal H\nu,\infty}$-Lipchitz continuity of $\partial_\nu\mathcal H$ deduced in Lemma \ref{lemma:LipmathcalHnu} and Remark \ref{remark:GlobalLip} implies that 
\[
(a) \le 2\Lambda_{\mathcal H\nu,\infty}^2 \left( W_\infty^2(\bar \mu^{N,L,b}_{t,\tau}, \mu^{N,b}_{t,\tau})) + \Vert \theta- \vartheta\Vert^2_\mathrm{F}\right)\le 2\Lambda_{\mathcal H\nu,\infty}^2 \left( W_\infty^2(\bar \rho^{N,L,b}_{t,\tau}, \rho^{N,b}_{t,\tau})) + \Vert \theta- \vartheta\Vert^2_\mathrm{F}\right).
\]
The second inequality comes from the fact that the projection map onto the $x$-coordinate is $1$-Lipschitz.
Furthermore,
 the Kantorovich-Rubinstein Theorem, which applies since $\partial_\nu\mathcal H$ is $\Lambda_{\mathcal H\nu,\infty}$- Lipschitz in $(x,a)$, implies that term $(b)$ is bounded according to
\[
(b) \le\Lambda_{\mathcal H\nu,\infty}^2 W_1^2(\bar\rho^{N,L,b}_{t,\tau},\rho^{N,b}_{t,\tau}) \le \Lambda_{\mathcal H\nu,\infty}^2  W_\infty^2(\bar\rho^{N,L,b}_{t,\tau},\rho^{N,b}_{t,\tau}).
\]
The Lemma results from combining the above two bounds according to \eqref{eqn:Deltag2}.
\end{proof}
    \section{Framework \& Notation}\label{App:G}
\textit{The assumptions introduced in this section are not required for Theorem \ref{theorem:WTS}, but are stated for intermediate Lemmas. We later justify that these assumptions hold for the transformer.}\\

Our goal is to prove uniform convergence of the transformer to the flow maps defined in Corollary \ref{cor:FlowMap_wd}, which take an initial point $x\in\R^d$ and an initial distribution $\zeta\in \mathcal P_c(\R^d)$ as input. Accordingly, we introduce the domain 
\begin{equation}\label{eqn:Domain}
    \mathscr{D} = \bar B(R_0)\times \mathcal P(\bar B(R_0)).
\end{equation}
With $(X^{x,\zeta}_t,a^{x,\zeta}_t)$ from Remark \ref{remark:flow} in mind, consider the map
$$Z:\Omega_0\times [0,1]\times\R^{d}\times \mathcal P_c(\R^d)\mapsto \R^{2d},$$
that satisfies
\begin{assumption}\label{assm:cont}
For $t\in[0,1],x \in \R^d$ and $\zeta\in \mathcal P(\bar B(R_0))$, the map $\Omega^0\ni\omega^0\mapsto Z(\omega^0,t,x,\zeta)$ is $\mathcal F_0$-measurable. Furthermore, for every $\omega^0\in \Omega^0$, the map $(t,x,\zeta)\mapsto Z(\omega^0,t,x,\zeta)$ is Lipschitz continuous on $[0,1]\times \R^d\times \mathcal P(\bar B(R_0))$.
\end{assumption}
This Assumption requires continuity in $(t,x,\zeta)$ for every $\omega^0\in\Omega^0$, rather than $\mathbb P^0$-almost surely. However, the version of the solution to \eqref{eqn:MVRODE1}-\eqref{eqn:MVRODE2} used in the proof of Lemma \ref{lemma:MVRODE_wd} is constructed such that it is continuous in $(t,x,\zeta)$ for every $\omega^0\in\Omega^0$.

We extend this map canonically to $(\Omega,\mathcal F,\mathbb P)$ via
\[
\widetilde{Z}((\omega_0,\omega_1),t,x,\zeta) = Z(\omega_0,t,x,\zeta).\]
 For notational ease, we drop the dependence on $\omega$ and write
\begin{equation}\label{eqn:mathfrak_r}
Z^{x,\zeta}_t:= \widetilde{Z}(\omega,t,x,\zeta), \qquad\mathfrak{r}^\zeta_t=(Z^{\cdot,\zeta}_t)_\#\zeta.
\end{equation}

\begin{assumption}\label{assm:gevrey}
    Suppose that, for some $n \ge 1$, the map $Z^{x,\zeta}_t\in C^n([0,1],\R^{d_1})$
    for every $x\in\R^d$, $\zeta\in\mathcal P(\bar B(R_0))$. Assume that there exist constants $C_Z,R_Z>0$ such that its derivatives are bounded uniformly in $(x,\zeta)\in \mathscr D$ and $t\in(0,1)$ according to
\[
\sup_{t\in[0,1]}\sup_{(x,\zeta)\in \mathscr{D}}\Vert \partial_t^j Z^{x,\zeta}_{t}\Vert \le C_Z R_Z^j j!,\qquad 0\le j \le n.
\]
\end{assumption}
We saw in equation \eqref{eqn:Dzgamma2} that $D_z \gamma(z,\mu)$ could be expressed as the ratio of two functions that are linear in the product measure $\mu\otimes\mu$, which we denote by $\mu^{\otimes 2}$. Notice that the product measure $((Z^{\cdot,\zeta}_t)_\#\zeta)^{\otimes 2}$ can be expressed as the pushforward of $\zeta^{\otimes 2}$ by the map
\begin{equation}\label{eqn:prodmeasurefm}
        \Psi^{\zeta}_t: \R^{2d}\mapsto \R^{4d}, \qquad \Psi_t^{\zeta}(y_1,y_2)=\left(Z^{y_1,\zeta}_t,Z^{y_2,\zeta}_t \right).
\end{equation}
Similarly to $Z^{x,\zeta}_t$, we shall write $\Psi^{w,\zeta}_t$ for $w\in \R^{2d}$. Notice that $\Gamma$ and both terms in $\mathcal K$ are linear in $\nu$. To reflect this in our framework, we consider functions of the following form. Let $ f: \R^{2d}\times \mathcal P_2(\R^{2d})\times \mathcal P(\mathfrak{S})\mapsto \R^d$ be continuous in each variable and define $F$ according to
\begin{align}\label{eqn:F}
    F: \R^{2d}\times \mathcal P_2(\R^{2d})\times\mathfrak{S}\mapsto \R^d, \qquad F(x,m,\nu) = \int f(x,m,\theta)d\nu(\theta),
\end{align}
for any $\nu \in \mathcal P(\mathfrak{S})$. The functions $\Gamma$ and $\nabla_x\mathcal H$ have a very specific structure, which we encapsulate by the following assumptions.
\begin{assumption}\label{assm:StructuredF}
For each $j=1,3$, let
    $$f_j:\R^{2d}\times \R^{4d}\times \mathfrak{S}\mapsto \R, \qquad f_2:\R^{2d}\times \R^{4d}\times \mathfrak{S}\mapsto \R^d$$
 be functions that are jointly continuous in each variable such that for any $z\in \R^{2d}$, $\rho\in \mathcal P_c(\R^{2d})$ and $\theta \in \mathfrak{S}$, we have
\begin{align}\label{eqn:f}
    f(z,\rho,\theta) &= \frac{\int  \exp(f_1(z,w,\theta)) f_2(z,w,\theta) d\rho^{\otimes 2}(w)}{F_2(z,\rho,\theta)},\\
F_2(z,\rho,\theta) &= \int \exp(f_{3}(z,w,\theta))d\rho^{\otimes 2}(w).
\end{align}
Let $C_Z$ be as in Assumption \ref{assm:gevrey}. Further assume that there exists a constant $f_\mathrm{max}>0$ such that for any $z\in \bar B_{2d}(C_Z),w\in (\bar B_{2d}(C_Z))^2$, $u\in \mathbb S^{d-1}$ and $\theta \in \mathfrak{S}$, we have
    \[
    \Vert f_j(z,w,\theta)\Vert \le f_\mathrm{max}, \qquad j=1,2,3.
    \]
\end{assumption}
\begin{assumption}\label{assm:LipCont}
Fix $R>0$. Then, there exists constants $\Lambda_F,B_F>0$, that depend only on $R,R_\theta$, such that for any $z_1,z_2\in \bar B_{2d}(R)$, $m_1,m_2\in \mathcal P(\bar B_{2d}(R))$, $\theta\in \mathfrak{S}$ and $\mathfrak n_1,\mathfrak n_2\in \mathcal P(\mathfrak{S})$, we have
    \begin{align*}
    \Vert f(z_1,m_1,\theta)\Vert &\le B_F, \\
         \Vert F(z_1,m_1,\mathfrak{n}_1)- F(z_2,m_2,\mathfrak{n}_2)\Vert &\le \Lambda_F\left(\Vert z_1-z_2\Vert +  W_2(m_1,m_2)+  W_2(\mathfrak n_1,\mathfrak n_2) \right).
    \end{align*}
\end{assumption}
    \section{A Maximal $L^2$ Inequality Uniform in the Initial Condition}\label{App:H}
Let $M^{x,\zeta}_{0,\tau}=0$ and, for $r\in[0:L]$, define
\begin{equation}\label{eqn:Martingale}
M^{x,\zeta}_{l,\tau}=\int_0^{\frac{l}{L}} F\left(Z^{x,\zeta}_t,\mathfrak{r}_t^\zeta,\hat \nu^{N,L}_{r_t,\tau}\right)-F\left(Z^{x,\zeta}_t,\mathfrak{r}^\zeta_t,\nu_{r_t,\tau}\right)dt,
\end{equation}
 where $\mathfrak{r}^\zeta_t$ is given by \eqref{eqn:mathfrak_r}. The key idea is that the expected value, under $\mathbb P^1$, of a test function integrated against the random measure $\hat \nu^{N,L}_{r/L,\tau}$, defined in \eqref{eqn:hatnu}, is the same test function integrated against $\nu_{r/L,\tau}$. Furthermore, $M_{l,\tau}^{x,\zeta}$ is adapted to the filtration defined by
\begin{equation}\label{eqn:FiltrationG}
        \mathcal G_{r}=\sigma\left(\mathcal F^0 \cup \sigma\left(\mathcal \theta^h_{r',0}:0\le r'<r,\, 1\le h\le H\right)\right), \qquad \mathcal G_{0}=\mathcal F^0.
\end{equation}
 As a result, $M^{x,\zeta}_{l,\tau}$ is a $\mathcal G_{r}$-martingale for every $x,\zeta$. Before undertaking any further analysis, we must ensure that the supremum of $M^{x,\zeta}_{t,\tau}$ over the initial data $\mathscr{D}$ is measurable so that the expectations in the following analysis  are well-defined.
\begin{lemma}\label{lemma:MeasurableSupremum}
For every $r\in[0:L]$ and $\tau \ge 0$, the random variable
    $$
    \sup_{(x,\zeta)\in\mathscr{D}} M^{x,\zeta}_{r,\tau}
    $$ is defined on $(\Omega,\mathcal F)$ and is $\mathcal G_{r}$-measurable.
\end{lemma}
\begin{proof}
By Assumption \ref{assm:cont} and applying Theorem 4.1 in \cite{AZHMYAKOV201987}, $(\omega^0,x)\mapsto Z^{x,\zeta}_t(\omega^0)$ is $(\mathcal F^0\otimes \mathcal B(\R^d))$-measurable and continuous in $x$ for each $\zeta\in \mathcal P(\bar B(R_0)$ and $t\in[0,1]$. As seen in the proof of Lemma, $(\omega^0,\theta) \mapsto \Phi_{t,\tau}[\bm\nu](\omega^0)(\theta)$ is $(\mathcal F^0\otimes \mathcal B(\R^{k\times 4d}))$-measurable. Then, Criterion \ref{criterion:measure} can be verified by applying Lemma 1.28 in \cite{Kallenberg2021} to the jointly measurable maps, as in the proof of Lemma \ref{lemma:GlobalWellDefined}. As a result, $\nu_{t,\tau}$ and $\mathfrak{r}^\zeta_t$ are $\mathcal F^0$-measurable and $\hat \nu^{N,L}_{\frac{r}{L},\tau}$ is $\mathcal G_{r}$-measurable, for every $t\in[0,1]$ and $0\le r <L$. By the Lipschitz continuity of $F$, Assumption \ref{assm:LipCont}, the integrand of $M^{x,\zeta}_{l,\tau}$ in \eqref{eqn:Martingale} is $\mathcal G_l$-measurable for every $0\le l<L$. Furthermore, for each $(\omega^0,\omega^1)\in \Omega$, the continuity in $t$ of $Z^{x,\zeta}_{t}(\omega^0)$ and $\mathfrak{r}^\zeta_t(\omega^0)$, Assumption \ref{assm:cont}, combined with the Lipschitz continuity of $F$, Assumption \ref{assm:LipCont}, implies that the map
\[
t\mapsto F(Z_t^{x,\zeta},\mathfrak{r}^\zeta_t,\hat\nu^{N,L}_{r_t,\tau})-F(Z_t^{x,\zeta},\mathfrak{r}^\zeta_t,\nu_{r_t,\tau})
\]
 is continuous on $(r/L,{r+1}/L)$. Therefore, the integrand of $M^{x,\zeta}_{r,\tau}$ is a Carath\'eodory function in $(\omega,t)$, and thus is $(\mathcal G_r\otimes \mathcal B([0,1]))$-measurable by Theorem 4.1 in \cite{AZHMYAKOV201987}. By invoking Lemma 1.28 in \cite{Kallenberg2021}, $M^{x,\zeta}_{r,\tau}$ is $\mathcal G_{r}-$measurable for every $x\in \bar B(R_0),\zeta\in \mathcal P(\bar B(R_0))$. 
 
 The metric spaces $(\bar B(R_0),\Vert.\Vert)$, where $\Vert.\Vert$ denotes the metric induced by the Euclidean norm, and $(\mathcal P(\bar B(R_0)), W_2)$ are separable. The latter is due to Theorem 6.18 in \cite{villani2008optimal}. Therefore, let $\tilde B(R_0), \tilde{\mathcal P}$ be countable dense subsets of $ \bar B(R_0), \mathcal P(\bar B(R_0))$, respectively.  Since $(x,\zeta)\mapsto Z^{x,\zeta}_t$ is Lipschitz by Assumption \ref{assm:cont}, for every $\zeta_1,\zeta_2\in \mathcal P(\bar B(R_0))$, the map $\zeta \mapsto \mathfrak{r}^\zeta$ satisfies the estimate
 \begin{align*}
W_2(\mathfrak{r}^{\zeta_1}_t,\mathfrak{r}^{\zeta_2}_t)&\le W_2((Z^{\cdot,\zeta_1}_t)_\#\zeta_1,(Z^{\cdot,\zeta_1}_t)_\#\zeta_2)+ \left(\int \Vert Z^{x,\zeta_1}_t-Z^{x,\zeta_2}_t\Vert^2 d\zeta_2(x) \right)^{1/2}\\
& \le 2 \Lambda_Z W_2(\zeta_1,\zeta_2),
 \end{align*}
 where $\Lambda_Z$ is the Lipschitz constant of $(x,\zeta)\mapsto Z^{x,\zeta}_t$ on $\mathscr{D}$. Therefore,
 $\zeta\mapsto\mathfrak{r}^\zeta_t= (Z^{\cdot,\zeta}_t)_\#\zeta$ is continuous.
 Combining this with the boundedness and Lipschitz continuity of $F$ in each argument, Assumption \ref{assm:LipCont}, the map $(x,\zeta)\mapsto M^{x,\zeta}_{r,\tau}$ is continuous, by the Dominated Convergence Theorem. As a result, the supremum
  \[
   \sup_{x\in \bar B(R_0)}\sup_{\zeta\in \mathcal P(\bar B(R_0))} M^{x,\zeta}_{r,\tau}= \sup_{x\in \tilde{B}(R_0)}\sup_{\zeta\in \tilde{\mathcal P}} M^{x,\zeta}_{r,\tau}
  \]
  can be taken over the countably dense subset instead. The pointwise supremum over a countable number of functions is measurable, which concludes the claim.
\end{proof}
The martingale $M^{x,\zeta}_{l,\tau}$ is the difference between two models that use parameters $\bm\nu_\tau$ and $\bm{\hat \nu}^{N,L}_\tau$, respectively. Therefore, we want to show that the supremum of $\Vert M^{x,\zeta}_{l,\tau}\Vert$ over $\mathscr{D}$ tends to zero in expectation. One technique we shall use to do this is Rademacher symmetrisation, Lemma $2.3.6$ in \cite{vanderVaart1996}, which requires taking a supremum over a class of mean-zero random variables. We express the Euclidean norm as the supremum over $u\in \mathbb S^{d-1}$ of $u^TM^{x,\zeta}_{l,\tau}$, which allows us to exploit the mean-zero property of $u^TM^{x,\zeta}_{l,\tau}$. Hence, we are interested in the function $J_2(z,w,u,\theta)=u^Tf_2(z,w,\theta)$, where $f_2$ is a function from the decomposition of $F$ in Assumption \ref{assm:StructuredF}.
\begin{definition}\label{def:J}
Let $f_1,f_2,f_3$ be as in Assumption \ref{assm:StructuredF}. Then, for j=1,2,3, define 
\[
    J_j: \R^{2d}\times \R^{4d}\times \mathbb S^{d-1}\times\mathfrak{S}\mapsto \R
\]
by 
\[
J_1(z,w,u,\theta) = f_1(z,w,\theta), \quad J_2(z,w,u,\theta)=\langle u,f_2(z,w,\theta)\rangle, \quad J_3(z,w,u,\theta)=f_3(z,w,\theta).
\]
\end{definition}
Although $J_1,J_3$ are independent of $u$, we include this dependence on $u$ so that each function is defined on the same domain.
\begin{lemma}\label{lemma:L1Mstar}
Suppose that the $f$ in equation \eqref{eqn:F} satisfies Assumptions \ref{assm:StructuredF}, \ref{assm:LipCont}. Further suppose that the map $Z^{x,\zeta}_t$ satisfies Assumptions \ref{assm:cont} and  \ref{assm:gevrey} for some $n\ge 1$, and take $\Psi^{w,\zeta}_t$ defined by equation \eqref{eqn:prodmeasurefm}. 
Then, there exists a constant $C_3=C_3(\Lambda_F,C_Z,R_Z,f_\mathrm{max})$ such that for any $l\in[1:L]$
\begin{align*}
        \mathbb E^1\Bigg[\sup_{(x,\zeta)\in\mathscr D}&\left\Vert M^{x,\zeta}_{l,\tau} \right\Vert \Bigg]\le C_3\sum_{j=1}^3\mathbb E^1 \left[\sup_{(x,\zeta,u,w) \in \mathscr D_2}\sum_{r=0}^{l-1}\sum_{h=1}^H  \frac{\xi_{r,h,j}}{LH}\;J_j(Z^{x,\zeta}_\frac{r}{L},\Psi^{\zeta}_\frac{r}{L}(w),u,\theta^h_{r,\tau})\right]+\frac{C_3 l}{L^2}, 
\end{align*}
$\mathbb P^0$-almost surely, where
\[
\mathscr D_2 =\bar B(R_0)\times \mathcal P(\bar B(R_0))\times \mathbb S^{d-1}\times (\bar B(R_0))^2.
\]
\end{lemma}
\begin{proof}
The $\Lambda_F-$Lipschitz continuity of $F$, Assumption \ref{assm:LipCont}, and structure of $F$ given in \eqref{eqn:F} implies that
\begin{align*}
    \Vert M^{x,\zeta}_{l,\tau}\Vert &\le \frac{1}{LH}\left\Vert \sum_{r=0}^{l-1} \sum_{h=1}^H \left(f\left(Z^{x,\zeta}_\frac{r}{L},\mathfrak{r}^\zeta_\frac{r}{L},\Phi_{\frac{r}{L},\tau}[\bm\nu](\theta_{r,0}^h)\right)-F\left(Z^{x,\zeta}_\frac{r}{L},\mathfrak{r}^\zeta_\frac{r}{L},\nu_{\frac{r}{L},\tau}\right)\right)\right\Vert\notag\\
    &+2\Lambda_F\sum_{r=0}^{l-1}\int_\frac{r}{L}^\frac{r+1}{L}\left\Vert Z^{x,\zeta}_t-Z^{x,\zeta}_\frac{r}{L} \right\Vert+ W_2\left(\mathfrak{r}^\zeta_t,\mathfrak{r}^\zeta_\frac{r}{L}\right)dt.
\end{align*}
Then, the $ C_Z R_Z-$Lipschitz continuity of the flow map $Z^\zeta_t$ in $t$,  Assumption \ref{assm:gevrey}, applied to the final term yields
\[
W_2\left(\mathfrak{r}^\zeta_t,\mathfrak{r}^\zeta_\frac{r}{L}\right)\le\sqrt{\int \left\Vert Z^{x,\zeta}_t-Z_\frac{r}{L}^{x,\zeta}\right\Vert^2d\zeta(x)} \le \frac{C_Z R_Z}{L}.
\]
The same bound holds on $\left\Vert Z^{x,\zeta}_t-Z^{x,\zeta}_\frac{r}{L} \right\Vert $.
After applying this time discretisation error, we get
\begin{align}\label{eqn:ToDoRadSym}
       \Vert M^{x,\zeta}_{l,\tau}\Vert &\le \sup_{u\in \mathbb S^{d-1}}\frac{1}{LH} \sum_{r=0}^{l-1} \sum_{h=1}^H u^T\left(f\left(Z^{x,\zeta}_\frac{r}{L},\mathfrak{r}^\zeta_\frac{r}{L},\Phi_{\frac{r}{L},\tau}[\bm\nu](\theta_{r,0}^h)\right)-F\left(Z^{\xi,\zeta}_\frac{r}{L},\mathfrak{r}^\zeta_\frac{r}{L},\nu_{\frac{r}{L},\tau}\right)\right)\\
    &+\frac{4\Lambda_FC_Z R_Z l}{L^2}. \notag
\end{align} 
Recall the product decomposition of $(\Omega,\mathcal F,\mathbb P)$ that was outlined in Section \ref{sec:MainTheorem}. As seen in the proof of Lemma \ref{lemma:GlobalWellDefined}, the AdamW flow map $\Phi_{t,\tau}[\bm\nu](\theta)$ is $\mathcal{F}^0
$-measurable for any deterministic $\theta$. Furthermore, the maps $Z^{x,\zeta}_t,\Psi^{w,\zeta}_t$ are $\mathcal F^0$-measurable by Assumption \ref{assm:cont}. In fact, $(\omega^0,x)\mapsto Z^{x,\zeta}_t(\omega^0)$ is a Carath\'eodory function, Definition 4.1 in \cite{AZHMYAKOV201987}, and so is $(\mathcal F^0\otimes \mathcal B(\R^{d}))$-measurable by Theorem 4.1 in \cite{AZHMYAKOV201987}. Then, for any $D \in \mathcal B(\R^{2d})$, Lemma 1.28 in \cite{Kallenberg2021} yeilds that the map
\[
\omega^0\mapsto \mathfrak{r}^\zeta_{t}(\omega^0)(D) = \int_{\R^d} \mathbbm{1}_{D}(Z^{x,\zeta}_t(\omega^0)) d\zeta(x)
\]
is $\mathcal F^0$-measurable. Therefore, Criterion \ref{criterion} states that $\mathfrak{r}^\zeta_{t}$ is $(\mathcal F^0,\mathcal B(\mathcal P_2(\R^{2d}))$-measurable. Accordingly, each of $Z^{x,\zeta}_\frac{r}{L},\mathfrak{r}^\zeta_\frac{r}{L},\Phi_{\frac{r}{L},\tau}[\bm\nu](\theta)$ is independent of $\mathcal F^1$. Therefore, for any $r\in[0:l]$, training step $\tau \in[0:T]$, and $\mathbb P^0$-a.e. $\omega^0$, the expectation with respect to $\mathbb P^1$ acts only on $\theta^h_r$. Hence, we obtain
\begin{align*}
    \mathbb E^1 \left[f\left(Z^{x,\zeta}_\frac{r}{L},\mathfrak{r}^\zeta_\frac{r}{L},\Phi_{\frac{r}{L},\tau}[\bm\nu](\theta_{r,0}^h)\right)
\right](\omega^0)&=\int f\left(Z^{x,\zeta}_\frac{r}{L}(\omega^0),\mathfrak{r}^\zeta_\frac{r}{L}(\omega^0),\Phi_{\frac{r}{L},\tau}[\bm\nu(\omega^0)](\theta_{r,0}^h(\omega^1))\right) d\mathbb P^1(\omega^1)\\&= \int f\left(Z^{x,\zeta}_\frac{r}{L}(\omega^0),\mathfrak{r}^\zeta_\frac{r}{L}(\omega^0),\Phi_{\frac{r}{L},\tau}[\bm\nu(\omega^0)](\theta)\right)d\pi(\theta)\\&=F\left(Z^{x,\zeta}_\frac{r}{L}(\omega^0),\mathfrak{r}^\zeta_\frac{r}{L}(\omega^0),\nu_{\frac{r}{L},\tau}(\omega^0)\right), 
\end{align*}
where $\pi = \mathbb P^1\circ (\theta^h_{r,0})^{-1}$ is the law of the parameters $\theta^h_{r,0}$ at initialisation, Assumption \ref{assm:Init}. From now on, we omit the dependence on $\omega^0$, and all equalities are understood to hold $\mathbb P^0$-a.s. The above expression implies that the first term on the right-hand side of equation \eqref{eqn:ToDoRadSym} is the average over $LH$ independent mean-zero random variables, when conditioned on $\mathcal F^0$. As a result, we may apply the Rademacher symmetrization given in Lemma $2.3.6$ in \cite{vanderVaart1996}, that only requires independence of the stochastic processes, and not identical distribution. Therefore, take $i.i.d.$ random variables $\epsilon_{0,1},\ldots,\epsilon_{l-1,H}$, defined on $(\Omega^1,\mathcal F^1,\mathbb P^1)$, and independent of $((\theta^h_r)_{r=0}^{L-1})_{h=1}^H$ that take the values $\pm 1$ with probability one half. By applying Rademacher symmetrization, Lemma $2.3.6$ in \cite{vanderVaart1996}, under $\mathbb P^1$, we deduce that 
    \begin{align}
            \mathbb E^1\bigg[\sup_{(x,\zeta)\in \mathscr D}\Vert M^{x,\zeta}_{l,\tau}\Vert \bigg]\le 2\mathbb E^1\left[ \sup_{(x,\zeta,u)\in \mathscr D_1}\frac{1}{LH}\sum_{h=1}^H\sum_{r=0}^{l-1}\epsilon_{r,h}\;u^Tf\left(Z^{x,\zeta}_\frac{r}{L},\mathfrak{r}^\zeta_\frac{r}{L},\theta_{r,\tau}^h\right)\right]+ \frac{4\Lambda_FC_Z R_Z l}{L^2},
    \end{align}
where the supremum is over the set $\mathscr D_1=\mathscr D\times \mathbb S^{d-1}$. Then, writing $f$ in the form given in Assumption \ref{assm:StructuredF}, we obtain
    \begin{align*}
    \mathbb E^1\bigg[ &\sup_{(x,\zeta,u)\in \mathscr D_1}\frac{1}{LH}\sum_{h=1}^H\sum_{r=0}^{l-1}\epsilon_{r,h}\;u^Tf\left(Z^{x,\zeta}_\frac{r}{L},\mathfrak{r}^{\zeta}_\frac{r}{L},\theta_{r,\tau}^h\right)\bigg]\\
    &=  \mathbb E^1\bigg[ \sup_{(x,\zeta,u)\in \mathscr D_1}\int\sum_{h=1}^H\sum_{r=0}^{l-1}\frac{\epsilon_{r,h}}{LH}\frac{ \exp\left(f_1(Z^{x,\zeta}_\frac{r}{L},\Psi^{w,\zeta}_\frac{r}{L},\theta^h_{r,\tau}))\right)J_2(Z^{x,\zeta}_\frac{r}{L},\Psi^{w,\zeta}_\frac{r}{L},u,\theta^h_{r,\tau})}{F_2(Z^{x,\zeta}_\frac{r}{L},\mathfrak{r}^\zeta_\frac{r}{L},\theta^h_{r,\tau})}d\zeta^{\otimes 2}(w)\bigg].\notag
    \end{align*}
Let us introduce the notation $\psi=(x,\zeta,u,w)\in \mathscr{D}_2$, where $\mathscr D_2=\mathscr D_1\times (\bar B(R_0))^2$, and define
\begin{equation}\label{eqn:shorthand}
            J^{\psi}_{j,r}(\theta^h_{r,\tau}) =J_j(Z^{x,\zeta}_\frac{r}{L},\Psi^{w,\zeta}_\frac{r}{L},u,\theta^h_{r,\tau}), \qquad j=1,2,3.
\end{equation}
     Since an integral with respect to the measure $\zeta^{\otimes 2}$ is bounded above by the supremum of the integrand over the support of $\zeta^{\otimes 2}$, it follows that
    \begin{align}\label{eqn:EnormK}
\mathbb E^1\bigg[ \sup_{(x,\zeta,u)\in \mathscr D_1}\frac{1}{LH}&\sum_{h=1}^H\sum_{r=0}^{l-1}\epsilon_{r,h}\;u^Tf\left(Z^{x,\zeta}_\frac{r}{L},\mathfrak{r}^{\zeta}_\frac{r}{L},\theta_{r,\tau}^h\right)\bigg]
\\&\le \mathbb E^1\left[ \sup_{(x,\zeta,u,w)\in \mathscr D_2}\frac{1}{LH}\sum_{h=1}^{H}\sum_{r=0}^{l-1}\epsilon_{r,h}\frac{ \exp\left(J^{x,\zeta,u,w}_{1,r}(\theta^h_{r,\tau})\right)J^{x,\zeta,u,w}_{2,r}(\theta^h_{r,\tau})}{F_2(Z^{x,\zeta}_\frac{r}{L},\mathfrak{r}^{\zeta}_\frac{r}{L},\theta^h_{r,\tau})}\right].\notag
\end{align}
 Define the map $\mathcal{Q}:\R^3\mapsto \R$ by
\begin{align*}
&C'=\left(f_{\mathrm{max}}\vee 1 \right)e^{3f_\mathrm{max}},\\ &\mathcal{Q}(z_1,z_2,z_3)= \frac{1}{C'}\begin{cases}
    \frac{ exp(z_1)z_2}{z_3} &\text{if } z_3>\exp(-f_\mathrm{max})\\
    \exp(z_1) z_2z_3 &\text{otherwise}
\end{cases},
\end{align*}
which is zero when $z_2$ is zero and $1$-Lipchitz, with respect to the $1$-norm, on $[-f_\mathrm{max},f_\mathrm{max}]\times[-f_\mathrm{max},f_\mathrm{max}]\times[0,\infty)$. Accordingly, the map given by 
\[
\widetilde{\mathcal Q}(z_1,z_2,z_3) =\mathcal Q\left((z_1\wedge f_\mathrm{max})\vee(-f_\mathrm{max}),\ (z_2\wedge f_\mathrm{max})\vee(-f_\mathrm{max}), \ z_3\vee 0\right)
\]
is 1-Lipschitz continuous with respect to the $1$-norm on $\R^3$, by composition of Lipschitz functions. Notice that $\vert J_1^\psi(\theta)\vert$ and $\vert J_2^\psi(\theta)\vert$ are bounded by $f_\mathrm{max}$ for all $\psi\in \mathscr{D}_2$ and $\theta \in\mathfrak{S}$ by Assumption \ref{assm:StructuredF}, which applies due to the boundedness of $Z^{x,\zeta}_t$ from Assumption \ref{assm:gevrey}. Similarly, \\$F_2(Z^{x,\zeta}_\frac{r}{L},\mathfrak{r}^\zeta_\frac{r}{L}, \theta)>\exp(-f_{\mathrm{max}})$ for all $(x,\zeta)\in \mathscr{D}$ and $\theta \in \mathfrak{S}$. Then, the term inside the summation on the right-hand side of \eqref{eqn:EnormK} coincides with $\widetilde{\mathcal  Q}$ with $z_1,z_2$ evaluated at $J_1^{\psi}(\theta^h_{r,\tau}),J_2^\psi(\theta^h_{r,\tau})$ and $z_3$ at $F_2$. Therefore, the multivariate Talagrand-Ledoux contraction inequality, Theorem 16.3 in \cite{vandeGeer2016}, applies to equation \eqref{eqn:EnormK}. We take the reference function $f^*$ in Theorem 16.3 in \cite{vandeGeer2016} to be identically zero as $\widetilde{\mathcal Q}(0)=0$. As a result, there exists a constant $C''>0$ such that, for any family of $i.i.d.$ random variables $\{\xi_{r,h,i}\}$, independent of $((\theta_{r,0}^h)_{r=0}^{L-1})_{h=1}^H$ defined on $(\Omega^1,\mathcal F^1)$ with $\mathcal \xi_{r,h,i}\sim \mathcal N(0,1)$ for $i\in\{ 1,2,3\},\,0\le r <L$ and $1\le h\le H$, we have 
    \begin{align}
           \mathbb E^1\Bigg[& \sup_{(x,\zeta,u,w)\in \mathscr D_2}\frac{1}{LH}\sum_{h=1}^H\sum_{r=0}^{l-1}\epsilon_{r,h}\frac{ \exp\left(J^{x,\zeta,u,w}_{1,r}(\theta^h_{r,\tau})\right) J^{x,\zeta,u,w}_{2,r}(\theta^h_{r,\tau})}{F_2(Z^{x,\zeta}_\frac{r}{L},\mathfrak{r}^{\zeta}_\frac{r}{L},\theta^h_{r,\tau})}\Bigg]\notag\\&=  C' \mathbb E^1\left[ \sup_{(x,\zeta,u,w)\in \mathscr D_2}\sum_{h=1}^H\sum_{r=0}^{l-1}\frac{\epsilon_{r,h}}{LH}\; \widetilde{\mathcal{Q}}\left(J^{x,\zeta,u,y}_{1,r}(\theta^h_{r,\tau}), J^{x,\zeta,u,y}_{2,r}(\theta^h_{r,\tau}),F_2(Z^{\zeta}_\frac{r}{L},\mathfrak{r}^{\zeta}_\frac{r}{L},\theta^h_{r,\tau})\right)\right]\notag\\&\le C''\mathbb E^1\left[\sup_{\psi \in \mathscr D_2}\frac{1}{LH}\sum_{h=1}^H\sum_{r=0}^{l-1} \xi_{r,h,1} J^{\psi}_{1,r}(\theta^h_{r,\tau}) +  \xi_{r,h,2}J^{\psi}_{2,r}(\theta^h_{r,\tau}) \right]\notag\\& +C''\mathbb E^1\left[\sup_{(x,\zeta)\in \mathscr D}\frac{1}{LH}\sum_{h=1}^H\sum_{r=0}^{l-1}\xi_{r,h,3} F_2(Z^{x,\zeta}_\frac{r}{L},\mathfrak{r}^{\zeta}_\frac{r}{L},\theta^h_{r,\tau})\right].\label{eqn:J4}
    \end{align}
 Although Theorem 16.3 in \cite{vandeGeer2016} is for deterministic $\theta^h_{r,\tau}$, we proceed by taking the expectation conditioned on $\theta^h_r$ first, then applying Theorem 16.3 in \cite{vandeGeer2016} and finally taking the expectation over the parameters. A more rigorous version of this argument is given at the end of this proof. Again, $F_2$ is linear in its measure argument, and so
\begin{align}\label{eqn:J5}
     \mathbb E^1\bigg[\sup_{(x,\zeta)\in \mathscr D}\sum_{h=1}^H\sum_{r=0}^{l-1}\frac{\xi_{r,h,3}}{LH}\,& F_2(Z^{x,\zeta}_\frac{r}{L},\mathfrak{r}^{\zeta}_\frac{r}{L},\theta^h_{r,\tau})\bigg]\\&= \mathbb E^1\left[\sup_{(x,\zeta,u)\in \mathcal D_1}\int\sum_{h=1}^H\sum_{r=0}^{l-1}\frac{\xi_{r,h,3}}{LH}\, \exp\left( J_{3,r}^{x,\zeta,u,w}(\theta^h_{r,\tau})\right)d\zeta^{\otimes 2}(w)\right]
     \notag\\&\le \mathbb E^1\left[\sup_{\psi\in \mathscr D_2}\frac{1}{LH}\sum_{h=1}^H\sum_{r=0}^{l-1}\xi_{r,h,3}\, \exp\left( J_{3,r}^{\psi}(\theta^h_{r,\tau})\right)\right].\notag
\end{align}
We would like to apply the Ledoux-Talagrand contraction inequality to the above expression. However, this only applies when the noise is Rademacher. Instead, we construct two collections of Gaussian random variables and apply the Sudakov-Fernique Theorem to compare their maximum, Corollary 3.14 in \cite{Ledoux1991}. Let $\varphi\in\mathfrak{S}^{l\times H}$ and define $\mathcal J^\psi,\mathscr{J}^\psi:\mathfrak{S}^{l\times H}\to \R $ by
\[
\mathcal J^\psi(\phi) = \sum_{h=1}
^H\sum_{r=0}^{l-1}\frac{\xi_{r,h,3}}{LH}\, \exp\left( J^\psi_{3,r}(\phi_{r,h})\right), \qquad \mathscr J^\psi(\phi) = \exp(f_\mathrm{max})\sum_{h=1}
^H\sum_{r=0}^{l-1} \frac{\xi_{r,h,3}}{LH}\,  J^\psi_{3,r}(\phi_{r,h}). \]
For any deterministic $\phi\in \mathfrak{S}^{l\times H}$ and $\psi_1,\psi_2\in \mathscr D_2$, by applying the independence and zero mean of $\xi_{r,h,3}$, we get 
\begin{align*}
    \E^1[|\mathcal J^{\psi_1}(\varphi)- \mathcal J^{\psi_2}(\varphi)|^2]&=\mathbb E^1\Bigg[\bigg(\sum_{h=1}^H\sum_{r=0}^{l-1}\frac{\xi_{r,h,3}}{LH}\;\left\{ \exp\left( J^{\psi_1}_{3,r}(\varphi_{r,h})\right)- \exp\left( J^{\psi_2}_{3,r}(\varphi_{r,h})\right)\right\}
\bigg)^2\Bigg]\\ & = \mathbb E^1\Bigg[\frac{1}{L^2H^2}\sum_{h=1}^H\sum_{r=0}^{l-1}\xi^2_{r,h,3}\;\left\{ \exp\left( J^{\psi_1}_{3,r}(\varphi_{r,h})\right)- \exp\left(J^{\psi_2}_{3,r}(\varphi_{r,h})\right)\right\}
^2\Bigg]
\end{align*}
As a result of Assumptions \ref{assm:StructuredF} and the boundedness of $Z^{x,\zeta}_t$ from Assumption \ref{assm:gevrey}, $|J^\psi_3(\theta)|\le f_{\mathrm{max}}$ for every $\psi\in \mathscr{D}_2$ and $\theta\in \mathfrak{S}$. Therefore, the function $\exp$ is $\exp(f_\mathrm{max})$-Lipschitz on the range of $J^\psi_3(\theta)$. This yields
\begin{align*}
\mathbb E^1\Bigg[\bigg(\sum_{h=1}^H\sum_{r=0}^{l-1}&\frac{\xi_{r,h,3}}{LH}\;\left\{ \exp\left( J^{\psi_1}_{3,r}(\varphi_{r,h})\right)- \exp\left( J^{\psi_2}_{3,r}(\varphi_{r,h})\right)\right\}
\bigg)^2\Bigg]
    \\& \le e^{2f_{\mathrm{max}}}\mathbb E^1\Bigg[\frac{1}{L^2H^2}\sum_{h=1}^H\sum_{r=0}^{l-1}\xi^2_{r,h,3}\;\left\{J^{\psi_1}_{3,r}(\varphi_{r,h})-J^{\psi_2}_{3,r}(\varphi_{r,h})\right\}
^2\Bigg]\\
& = \mathbb E^1\Bigg[\bigg( \mathscr{J}^{\psi_1}(\varphi)-\mathscr{J}^{\psi_2}(\varphi)\bigg)
^2\Bigg].
\end{align*}
Since $\mathscr{D}_2$, equipped with the product metric induced by the Euclidean and $W_2$ metrics, is separable, there exists a countable dense subset $\widetilde{\mathscr{D}}_2\subseteq \mathscr{D}_2$. The separability of $\mathcal P(\bar B(R_0))$ is due to Theorem 6.18 in \cite{villani2008optimal}. Let $\widetilde{\mathscr{D}}^n_2$ be an increasing sequence of subsets of $\mathscr{D}_2$ with $|\widetilde{\mathscr{D}}^n_2|=n$, such that
\[
\bigcup_{n\ge 0}\widetilde{\mathscr{D}}^n_2= \widetilde{\mathscr{D}}_2.
\]
Then, for $\mathbb P^0$-a.e. $\omega^0\in \Omega^0$, and any deterministic $\varphi\in \mathfrak{S}^{l\times H}$, and for $\bm\psi=(\psi_1,\ldots\psi_n)$ an enumeration of $\mathscr{D}^n_2$, the vectors in $\R^n$ given by
\[
\underline{\mathcal J}^{\bm{\psi}}(\varphi):=(\mathcal J^{\psi_1}(\varphi),\ldots,\mathcal J^{\psi_n}(\varphi)),\qquad \underline{\mathscr{J}}^{\bm{\psi}}(\varphi):= (\mathscr J^{\psi_1}(\varphi),\ldots,\mathscr J^{\psi_n}(\varphi)),
\]
are mean-zero Gaussians on $(\Omega^1,\mathcal F^1,\mathbb P^1)$, since each component is a linear combination of independent Gaussian random variables. The above calculation shows that $\underline{\mathcal J}^{\bm{\psi}}(\varphi), \underline{\mathscr J}^{\bm{\psi}}(\varphi)$ satisfies the conditions for the Corollary 3.14 of \cite{Ledoux1991}, which is a Corollary to the Sudakov-Fernique Theorem. By applying this Corollary, we obtain
\[
\E^1\left[\max_{i\in[1:n]}\mathcal J^{\psi_i}(\varphi)\right]\le2 \E^1\left[ \max_{i\in[1:n]}\mathscr J^{\psi_i}(\varphi)\right]\le 2\E^1\left[\sup_{\psi\in \widetilde{\mathscr{D}}_2}\mathscr J^{\psi}(\varphi)\right].
\]
By applying the Monotone Convergence Theorem to the left-hand side, we get
\[
\E^1\left[\sup_{\psi\in \widetilde{\mathscr{D}}_2}\mathcal J^{\psi}(\varphi)\right]\le2 \E^1\left[\sup_{\psi\in \widetilde{\mathscr{D}}_2}\mathscr J^{\psi}(\varphi)\right].
\]
Recall that $f_3=J_3$ is jointly continuous in each variable, Assumption \ref{assm:StructuredF}, and the maps \\$(t,x,\zeta)\mapsto Z_t^{x,\zeta}$ and $(t,w,\zeta)\mapsto \Psi_t^{w,\zeta}$ are jointly continuous by Assumption \ref{assm:cont} for every $\omega^0\in \Omega^0$. Therefore, $\mathcal J^{\psi}(\varphi)$ and $\mathscr J^{\psi}(\varphi)$ are continuous in $\psi$ on $\mathscr{D}$. This implies that the supremum of $\mathcal J^{\psi}(\varphi)$ and $\mathscr J^{\psi}(\varphi)$ over the countable dense subset $\widetilde{\mathscr D}_2$ coincides with that over $\mathscr{D}_2$. Since $\Theta =\{\theta_{0}^1,\ldots,\theta^H_{L-1} \}$is $\mathcal G_l$-measurable, while $((\xi_{r,h,3})_{r=0}^{L-1})_{h=1}^H$ are independent of $\mathcal G_l$, it follows from Example 4.1.7 in \cite{Durrett_2019} and the tower property for conditional expectation that
\begin{align*}
    \mathbb E^1\left[\sup_{\psi \in \mathscr D_2}\sum_{h=1}^H\sum_{r=0}^{l-1}\frac{\xi_{r,h,3}}{LH}\,\exp\left( J^{\psi}_{3,r}(\theta^h_{r,\tau})\right)\right]&= \E^1\left[\E^1\left[\sup_{\psi\in \widetilde{\mathscr{D}}_2}\mathcal J^\psi(\Theta)\middle|\mathcal G_l\right]\right]\\
    &= \E^1\left[\E^1\left[\sup_{\psi\in \widetilde{\mathscr{D}}_2}\mathcal J^\psi(\phi)\right]\middle|_{\phi=\Theta}\right].
\end{align*}
By applying the bound on the expected surpemum of $\mathcal J^\psi$ by that on $\mathscr{J}^\psi$ derived above, we get
\begin{align*}
    \mathbb E^1\left[\sup_{\psi \in \mathscr D_2}\sum_{h=1}^H\sum_{r=0}^{l-1}\frac{\xi_{r,h,3}}{LH}\,\exp\left( J^{\psi}_{3,r}(\theta^h_{r,\tau})\right)\right]&\le 2\E^1\left[\E^1\left[\sup_{\psi\in \widetilde{\mathscr{D}}_2}\mathscr J^\psi(\Theta)\middle|\mathcal G_l\right]\right]
    \\&=2 \mathbb E^1\left[\sup_{\psi \in \mathscr D_2}\frac{e^{f_{\mathrm{max}}}}{LH}\sum_{h=1}^H\sum_{r=0}^{l-1}\xi_{r,h,3}\;  J^{\psi}_{3,r}(\theta^h_{r,\tau})\right].
\end{align*}
Substituting this bound back into equation \eqref{eqn:J5} and then into \eqref{eqn:J4} concludes the Lemma.
\end{proof}
Example choices for $J_2$, which we will make in Appendix \ref{app:J}, are
\[
J_2^{\gamma}(z,w,u,\theta) = \langle\theta_O u,\theta_V w_1\rangle, \qquad J_2^{\gamma z}(z,w,u,\theta)  = \beta \langle\theta_Q u,\theta_K(w_1-w_3)\rangle\langle\theta_O z_2,\theta_V w_1\rangle,
\]
 corresponding to $F=\Gamma$ and $\nabla_x \mathcal H$, respectively, where the subscript $j$ indexes the $j$-th block of $d$ components of  $z\in \R^{2d}$ and $w\in \R^{4d}$. The structures of $J_1,J_3$ are similar. Note that, for any $u\in \R^d$, $w\in \R^{4d}$ and $\phi^1,\ldots,\phi^H\in \mathfrak{S}$,
\[
\sum_{h=1}^H J_2^\gamma(z,w,u,\phi^h) = u^T\left(\sum_{h=1}^H(\phi_O^h)^T\phi^h_V \right)w_1.
\]
Hence, when $(\phi^h_O)^T\phi_V^h$ are $i.i.d.$ with zero mean, the expected Frobenius norm of the sum scales like $\sqrt H$. However, such an interchange between  $J_2^{\gamma z}$ and summation is no longer true. Instead, we informally have that
\[
J_2^{\gamma z}(z,w,u,\theta) \;\text{``=''}\;\beta (u\otimes z_2)^T\left(\theta_Q^T\theta_K\otimes \theta_O^T\theta_V \right)((w_1-w_3) \otimes w_1).
\] 
As a result, a higher-order tensor analog occurs rather than a matrix sum, which we summarise by Assumption \ref{assm:MultiLinear}. This structure allows us to compute explicit bounds on the suprema rather than rely on Theorems, such as Dudley's Metric Integral Theorem (5.22 in \cite{Wainwright_2019}), to bound the expected suprema. Therefore, the argument will not require any coverings of the space the suprema is taken over, which makes the bounds independent of the number of tokens and dimension.
\begin{assumption}\label{assm:MultiLinear}
For some $\bm q\in\{0,1,2\}^3$, there exists a tensor-valued map
\[
\widetilde J:\mathfrak S \mapsto  (\mathbb R^{2d})^{\otimes q_1}\otimes(\mathbb R^{4d})^{\otimes q_2}\otimes (\R^d)^{\otimes q_3}
\]
 such that for all $(z,w,u,\theta)\in \mathbb R^{2d}\times \mathbb R^{4d}\times \mathbb R^d\times \mathfrak S$, the map $J$ takes the form
\begin{align*}
J(z,w,u,\theta)
=
    \left\langle \widetilde J(\theta),\, z^{\otimes q_1} \otimes w^{\otimes q_2}\otimes u^{\otimes q_3}\right\rangle_{\mathrm{HS}},
\end{align*}
where $\langle .,.\rangle_\mathrm{HS}$ is the Hilbert-Schmidt inner product. Here, to account for the cases when $J$ is independent of $z$ or $u$, we use the convention $(\R^d)^{\otimes 0}=\R$ and for any $u \in \R^d$ we define $u^{\otimes 0}=1$. Furthermore, there exists $ C_4=C_4(\beta,R_\theta)$ such that for any $\theta\in \mathfrak{S}$
\[
\left\Vert\widetilde J(\theta) \right\Vert_\mathrm{HS}\le C_4 .
\]
\end{assumption}
\begin{corollary}\label{corr:MultiLinearBound} 
Given $l,n\ge 1$ and $M,R>0$, take $u \in \mathbb S^{d-1},z_j\in \R^{2d},w_j\in\R^{4d}$, $\theta_s\in \mathfrak{S}$ and $e_s, c_{s} \in \R$ such that
\[
\Vert z_j\Vert,\Vert w_j\Vert \le MR^j
\]
for $j\in[0:n-1], s\in[1:l]$. Suppose that $J$ satisfies Assumption \ref{assm:MultiLinear} for some $\bm q\in \{0,1,2\}^3$. Then, with $q'=q_1+q_2$, we have
\begin{align*}
    \left|\sum_{s=1}^{l} e_s J\left(\sum_{j=0}^{n-1} z_j c^j_{s} ,\ \sum_{j=0}^{n-1} w_jc^j_{s},\  u,\theta_s\right) \right|\le M^{q'} \sum_{j_{1},\ldots, j_{q'}=0}^{n-1} \left\Vert \sum_{s=1}^l (Rc_s)^{j_1+\ldots +j_{q'}} \,e_s \widetilde J(\theta_s) \right\Vert_\mathrm{HS},
\end{align*}
where the superscript $j$ on $c_s$ denote the $j$-th power of $c_s$.
\end{corollary}
\begin{proof}
For $v_0,\ldots, v_{n-1}\in \R^d$, and $q\in \mathbb Z_{
>0}$, we have the following identity
\[
\left(\sum_{j=0}^{n-1}v_j\right)^{\otimes q} =\sum_{j_1,\ldots, j_q=0}^{n-1} \bigotimes_{i=1}^qv_{j_i}.
\]
Since $J$ satisfies Assumption \ref{assm:MultiLinear}, and using the above identity, we get
\begin{align*}
     \bigg|\sum_{s=1}^{l} e_s J&\bigg(\sum_{j=0}^{n-1} z_j (c_{s})^j,\sum_{j=0}^{n-1} w_j(c_{s})^j, u,\theta_s\bigg) \bigg|\\&= \left|\sum_{s=1}^{l} e_s \left\langle \widetilde J(\theta_s),\left(\sum_{j=0}^{n-1} z_j (c_{s})^j\right)^{\otimes q_1}\otimes\left(\sum_{j=0}^{n-1} w_j(c_{s})^j\right)^{\otimes q_2}\otimes u^{\otimes q_3}\right\rangle_\mathrm{HS} \right|\\&=
     \left|\sum_{s=1}^{l} e_s \left\langle \widetilde J(\theta_s),\left(\sum_{j_1,\ldots ,j_{q_1}=0}^{n-1} \bigotimes_{i=1}^{q_1} z_{j_i} (c_{s})^{j_i}\right)\otimes\left(\sum_{j_1,\ldots, j_{q_2}=0}^{n-1} \bigotimes_{i=1}^{q_2} z_{j_i} (c_{s})^{j_i}\right)\otimes u^{\otimes q_3}\right\rangle_\mathrm{HS} \right|
     \\
     & \le\sum_{j_1,\ldots,j_{q'}=0}^{n-1} \left| \left\langle\sum_{s=1}^{l} e_s c^{j_1+\ldots+j_{q'}}_{s}\widetilde J(\theta_s),\bigotimes_{i=1}^{q_1}z_{j_i}\otimes\bigotimes_{i=q_1+1}^{q'}w_{j_i}\otimes u^{\otimes q_3}\right\rangle_\mathrm{HS} \right|,
\end{align*}
with the convention that $\bigotimes_{i=1}^0z_{j_i}=1$.
Then, by the Cauchy-Schwarz inequality, the fact that the Hilbert-Schmidt norm is multiplicative under tensor products and the assumption that the norms of $w_j,z_j$ scale with $R^j$, we have
\begin{align*}
     \bigg|\sum_{s=1}^{l} e_s J\bigg(\sum_{j=0}^{n-1} z_j c^j_{s}&,\sum_{j=0}^{n-1} w_jc^j_{s}, u,\theta_s\bigg) \bigg| \le M^{q'}\sum_{j_1,\ldots,j_{q'}=0}^{n-1}  \left\Vert\sum_{s=1}^{l} e_s (Rc_s)^{j_1+\ldots+j_{q'}}\widetilde J(\theta_s)\right\Vert_\mathrm{HS} .
\end{align*}
\end{proof}
\begin{corollary}\label{corr:LipJ}
    Suppose that Assumption \ref{assm:MultiLinear} holds for some $\bm q\in \{0,1,2\}^3$ and fix $R>0$. Then, there exists a constant $\Lambda_J=\Lambda_J(\bm q,R)$ such that for any $z_1,z_2\in \bar B_{2d}(R), w_1,w_2\in (\bar B_{2d}(R))^2$, $u \in \mathbb S^{d-1}$, $ e_1,\ldots ,e_l \in \R$ and $\theta_1 ,\ldots,\theta_l\in\ \mathfrak{S}$, we have 
\begin{equation*}
    \left|\sum_{s=1}^l e_s \left( J(z_1,w_1,u,\theta_s)-J(z_2,w_2,u,\theta_s)\right)\right|\le\Lambda_J\left( \Vert z_1-z_2 \Vert + \Vert w_1-w_2 \Vert \right)\left\Vert \sum_{s=1}^l e_s \widetilde J(\theta_s)\right\Vert_\mathrm{HS}.
\end{equation*}
\end{corollary}
\begin{proof}
Since $J$ satisfies \ref{assm:MultiLinear}, we get that
\begin{align*}
  \bigg|\sum_{s=1}^l &e_s \bigg( J(z_1,w_1,u,\theta_s)-J(z_2,w_2,u,\theta_s)\bigg)\bigg|\\&\le  \left\vert \left\langle \sum_{s=1}^m  e_s \widetilde J(\theta_s), \left(\left(z_1^{\otimes q_1}-z_2^{\otimes q_1}\right)\otimes w_1^{q_2}+z_2^{\otimes q_1}\otimes \left(w_1^{q_2}-w_2^{\otimes q_2}\right)\right)\otimes u^{\otimes q_3}\right\rangle_\mathrm{HS}\right\vert. 
\end{align*}
Therefore, by the Cauchy-Schwarz inequality and using that the Hilbert-Schmidt norm is multiplicative under tensor products, we get
\begin{align*}
  \bigg|\sum_{s=1}^l &e_s \bigg( J(z_1,w_1,u,\theta_s)-J(z_2,w_2,u,\theta_s)\bigg)\bigg|\\&\le   \left((\sqrt{2}R)^{q_2}\left\Vert z_1^{\otimes q_1}-z_2^{\otimes q_1}\right\Vert_{\mathrm{HS}}+R^{q_1} \Vert w_1^{\otimes  q_2}-w_2^{\otimes q_2}\Vert_{\mathrm{HS}}\right)\left\Vert  \sum_{s=1}^m  e_s \widetilde J(\theta_s)\right\Vert_\mathrm{HS}. 
\end{align*}
The Corollary follows from the fact that, for any $v_1,v_2\in \bar B(R)$ and $q\ge1 $,
\[
\Vert v_1^{\otimes q}-v_2^{\otimes q}\Vert_\mathrm{HS}\le qR^{q-1}\Vert v_1-v_2\Vert.
\]
If $q= 0$, we use the convention that  $v_1^{\otimes 0}=v_2^{\otimes 0}=1$, so that $\Vert v_1^{\otimes 0}-v_2^{\otimes 0}\Vert_\mathrm{HS}=0\le \Vert v_1-v_2\Vert$.
\end{proof}

\begin{lemma}\label{lemma:TaylorApproxE}
Suppose that the flow map $Z_t^{x,\zeta}$, with $ \Psi_t^{w,\zeta}$ defined in \eqref{eqn:prodmeasurefm}, satisfies Assumptions \ref{assm:cont} and \ref{assm:gevrey} for some integer $n\ge 1$ and that $J$ satisfies Assumption \ref{assm:MultiLinear} for some $\bm q\in\{0,1,2\}^3$.
Let $\xi_{r,h}$ be $i.i.d.$ samples from $\mathcal N(0,1)$, independent of $((\theta^h_{r,0})_{r=0}^{L-1})_{h=1}^H$ defined on $(\Omega^1,\mathcal F^1,\mathbb P^1)$, for $r \in[0:l-1]$ and $h \in[1:H]$. Then, there exists a constant $C_{M,J}= C_{M,J}(\bm q,\beta,C_Z,R_Z,n)$ such that for any training step $\tau \in[0:T]$ and layer number $r'\in[1:L]$, we have
    \begin{align}
        \mathbb E^1\left[ \sup_{(x,\zeta,u,w)\in\mathscr D_2}\frac{1}{LH}\sum_{r=0}^{r'-1}\sum_{h=1}^H \xi_{r,h} J(Z^{x,\zeta}_\frac{r}{L}, \Psi^{w,\zeta}_\frac{r}{L} ,u, \theta^h_{r,\tau})\right]&\le C_{M,J}\left(\frac{1}{L}+ \frac{1}{L^{\frac{n}{2n+1}}H^{1/2}} \right),\label{eqn:EsupMathcalJ}
    \end{align}
$\mathbb P^0$-almost surely. 
\end{lemma}
\begin{remark}\label{remark:subadditivity}
By the subadditivity of the supremum, Lemma \ref{lemma:TaylorApproxE} extends to the case where $J$ is a linear combination of functions that satisfy Assumption \ref{assm:MultiLinear} with possibly different $\bm{q}$.
\end{remark}
\begin{proof}
Take $m\in \mathbb Z_{>0}$ such that $m\le \sqrt{L}$.
 Partition the layers into $\lfloor r'/m\rfloor$ contiguous blocks, and let $\hat Z^{x,\zeta,(n-1)}_t$ denote the $(n-1)^{th}$ order Taylor approximation to $Z^{x,\zeta}_t$ centered at $m\lfloor Lt/m\rfloor/L$ so that
\begin{equation}\label{eqn:TaylowApprox}
    \hat Z^{x,\zeta,(n-1)}_\frac{r}{L}: =  \sum_{j=0}^{n-1}\partial_t^j Z^{x,\zeta}_t\big|_{t=\frac{lm}{L}}\frac{(r-lm)^j}{L^jj!}, \qquad  \text{for}\;\;r \in [lm:(l+1)m-1].
\end{equation}
Define $\hat \Psi^{w,\zeta,(n-1)}_t$ similarly.
Then, approximating $Z^{x,\zeta}_\frac{r}{L}, \Psi^{w,\zeta}_\frac{r}{L}$ using $ \hat Z^{x,\zeta,(n-1)}_{\frac{r}{L}}, \hat \Psi^{w,\zeta,(n-1)}_{\frac{r}{L}}$, the expected supremum of $J$ is bounded by two parts corresponding to 
\begin{align}  \sum_{r=0}^{r'-1}\frac{\xi_{r,h}}{LH}J^{\psi}_r(\theta^h_{r,\tau})\notag&=
     \sum_{l=0}^{\lfloor\frac{r'}{m} \rfloor }\sum_{r=lm}^{(l+1)m-1} \mathbbm{1}_{r<r'}\frac{\xi_{r,h}}{LH}\left(J^{\psi}_r(\theta^h_{r,\tau})-J(\hat Z^{x,\zeta,(n-1)}_\frac{r}{L},\hat \Psi^{w,\zeta,(n-1)}_\frac{r}{L},u,\theta^h_{r,\tau})\right)\\
    &+  \sum_{l=0}^{\lfloor\frac{r'}{m} \rfloor }\sum_{r=lm}^{(l+1)m-1}\mathbbm{1}_{r<r'} \frac{\xi_{r,h}}{LH}J(\hat Z^{x,\zeta,(n-1)}_\frac{r}{L},\hat \Psi^{w,\zeta,(n-1)}_\frac{r}{L},u,\theta^h_{r,\tau}),\label{eqn:mathcalJ1}
\end{align}
where for ease of notation, similarly to equation \eqref{eqn:shorthand}, we have introduced $\psi = (x,\zeta,u,w)\in \mathscr{D}_2$ and
\[
J^\psi_r(\theta^h_{r,\tau})= J(Z^{x,\zeta}_\frac{r}{L}, \Psi^{w,\zeta}_\frac{r}{L} ,u, \theta^h_{r,\tau}).
\]
The indicator function has been introduced to account for the case when $r'$ is not divisible by $m$.
Then, substituting in the definition of $\hat Z^{x,\zeta,(n-1)}_\frac{r}{L}, \Psi^{w,\zeta,(n-1)}_\frac{r}{L}$, the second term on the right-hand side of equation \eqref{eqn:mathcalJ1} satisfies
\begin{align*}
    \bigg|\sum_{r=lm}^{(l+1)m-1}\sum_{h=1}^H &\xi_{r,h}J(\hat Z^{x,\zeta,(n-1)}_\frac{r}{L},\hat \Psi^{w,\zeta,(n-1)}_\frac{r}{L},u,\theta^h_{r,\tau})\bigg| \\&= 
    \left| \sum_{r=lm}^{(l+1)m-1}\sum_{h=1}^H \xi_{r,h}J\left(\sum_{j=0}^{n-1}\partial_t^j Z^{x,\zeta}_\frac{lm}{L}\frac{(r-lm)^j}{L^jj!},\sum_{j=0}^{n-1}\partial_t^j \Psi^{w,\zeta}_\frac{lm}{L}\frac{(r-lm)^j}{L^jj!},u,\theta^h_{r,\tau}\right)\right|.
\end{align*}
Then, since Assumptions \ref{assm:gevrey} and \ref{assm:MultiLinear} hold, we apply Corollary \ref{corr:MultiLinearBound} with $q=q_1+q_2$ to get
\begin{align}
\bigg|\sum_{r=lm}^{(l+1)m-1}\sum_{h=1}^H &\xi_{r,h}J(\hat Z^{x,\zeta,(n-1)}_\frac{r}{L},\hat \Psi^{w,\zeta,(n-1)}_\frac{r}{L},u,\theta^h_{r,\tau})\bigg| \label{eqn:mathcalJ1linearity}\\&\le C_Z^{q}\sum_{j_{1},\ldots, j_q=0}^{n-1} R_Z^{j_1+ \ldots +j_q} \left\Vert \sum_{r=lm}^{(l+1)m-1} \sum_{h=1}^H \xi_{r,h}\left(\frac{r-lm}{L} \right)^{j_1+ \ldots + j_q}\widetilde J\left( \theta^h_{r,\tau}\right) \right\Vert_\mathrm{HS}.\notag
\end{align}
Notice that the right-hand side of equation \eqref{eqn:mathcalJ1linearity} is independent of $x,\zeta,u,w$. Therefore, the expectation with respect to $\mathbb P^1$ of the supremum over $\mathscr D_2$ of the left-hand side of equation \eqref{eqn:mathcalJ1} is $\mathbb P^0$-a.s. controlled by 
\begin{align*}
    \mathbb E^1\bigg[&\sup_{(x,\zeta,u ,y )\in \mathscr D_2} \frac{1}{LH}\sum_{l=0}^{\lfloor \frac{r'}{m}\rfloor}\sum_{r=lm}^{(l+1)m-1}\sum_{h=1}^H \mathbbm{1}_{r<r'}\xi_{r,h}J(\hat Z^{x,\zeta,(n-1)}_\frac{r}{L},\hat \Psi^{w,\zeta,(n-1)}_\frac{r}{L},u,\theta^h_{r,\tau})\bigg]\\&\le  \mathbb E^1 \left[ \frac{C_Z^q}{LH}\sum_{l=0}^{\lfloor \frac{r'}{m}\rfloor}\sum_{j_1,\ldots,j_q=0}^{n-1}R_Z^{j_1+\ldots+j_q}\left\Vert\sum_{r=lm}^{(l+1)m-1}\sum_{h=1}^H \mathbbm{1}_{r<r'}\xi_{r,h}\widetilde J\left(\theta^h_{r,\tau}\right)\left(\frac{r-lm}{L}\right)^{j_1+\ldots+j_q}\right\Vert_\mathrm{HS}\right].
\end{align*}
Then, by applying Jensen's inequality followed by the independence and zero mean of the random variables $\xi_{r,h}$, we obtain
\begin{align}
\mathbb E^1 \bigg[&\sup_{(x,\zeta,u ,y )\in \mathscr D_2} \frac{1}{LH}\sum_{l=0}^{\lfloor \frac{r'}{m} \rfloor}\sum_{r=lm}^{(l+1)m-1}\sum_{h=1}^H \mathbbm{1}_{r<r'}\xi_{r,h}J(\hat Z^{x,\zeta,(n-1)}_\frac{r}{L},\hat\Psi^{w,\zeta,(n-1)}_\frac{r}{L},u,\theta^h_{r,\tau})\bigg]
\notag\\& \le\frac{C_Z^q}{LH}\sum_{l=0}^{\lfloor \frac{r'}{m}\rfloor}\sum_{j_1,\ldots,j_q=0}^{n-1}R_Z^{j_1+\ldots+j_q}\sqrt{\sum_{r=lm}^{(l+1)m-1}\sum_{h=1}^H \mathbbm{1}_{r<r'}\mathbb E^1\left[\xi^2_{r,h}\left\Vert \widetilde J(\theta^h_{r,\tau})\right\Vert^{2}_\mathrm{HS}\right]\left(\frac{r-lm}{L}\right)^{2(j_1+\ldots+j_q)}}\notag\\
    & \le\frac{C_4 C_Z^{q}}{\sqrt{LH}}\left(\left\lfloor\frac{r'}{m}\right\rfloor+1\right)\sum_{j_1,\ldots,j_q=0}^{n-1}R_Z^{j_1+\ldots+j_q}\sqrt{ \frac{1}{L}\sum_{i=0}^{m-1}\left(\frac{i}{L} \right)^{2(j_1+\ldots+j_q)} }\notag.
\end{align}
The final inequality uses the boundedness assumption on $\widetilde J$, see Assumption \ref{assm:MultiLinear}. Note that for any $c>0$
\begin{equation}\label{eqn:intapproxsum}
      \frac{1}{L}\sum_{i=0}^{m-1}\left(\frac{i}{L} \right)^{c}\le \int_0^{\frac{m}{L}}x^{c}dx = \frac{1}{c+1}\left(\frac{m}{L}\right)^{c+1}.
\end{equation}
Therefore, by applying bound \eqref{eqn:intapproxsum} and simplifying the resulting geometric series, we get
\begin{align}
\mathbb E^1 \bigg[\sup_{(x,\zeta,u ,y )\in \mathscr D_2} \frac{1}{LH}&\sum_{l=0}^{\lfloor \frac{r'}{m} \rfloor}\sum_{r=lm}^{(l+1)m-1}\sum_{h=1}^H \mathbbm{1}_{r<r'}\xi_{r,h}J(\hat Z^{x,\zeta,(n-1)}_\frac{r}{L},\hat\Psi^{w,\zeta,(n-1)}_\frac{r}{L},u,\theta^h_{r,\tau})\bigg]
\notag\\
&\le\frac{C_4 C_Z^{q}}{\sqrt{LH}}\left(\left\lfloor\frac{r'}{m}\right\rfloor+1\right)\sum_{j_1,\ldots,j_q=0}^{n-1}\left(\frac{R_Z m}{L}\right)^{j_1+\ldots+j_q}\sqrt{ \frac{m}{L} }\notag
\\
    & \le \frac{2 C_4 C_Z^{q}}{\sqrt{mH}} \left(\frac{1-\left(\frac{R_Z m}{L} \right)^{n}}{1-\frac{R_Z m}{L}} \right)^{q}.\label{eqn:bound_mathcalJ1_1}
\end{align}
The final inequality uses the fact $\lfloor r'/m\rfloor+1\le 2L/m$. For $0\le r< L$ with $l = \lfloor r/m\rfloor$, by Taylor's Theorem and Assumption \ref{assm:gevrey} on $\Vert \partial_t^j Z_t^{x,\zeta}\Vert$,  for any $(x,\zeta ) \in \mathscr D $ there exists some $s\in (lm/L,r/L)$ such that
\begin{align}\label{eqn:TA1}
    \Vert Z^{x,\zeta}_{\frac{r}{L}}-\hat Z^{x,\zeta,(n-1)}_{\frac{r}{L}} \Vert  &= \left\Vert \partial_t^{n}Z_t^{x,\zeta}|_{t=s}\frac{(r-lm)^n}{n!L^n}\right\Vert\le C_Z \left(R_Z \frac{r-lm}{L} \right)^n, \\ \Vert \hat Z^{x,\zeta ,(n-1)}_\frac{r}{L}\Vert &\le C_Z\sum_{i=0}^{n-1}\left(R_Z \frac{r-lm}{L} \right)^i\le C_Z\frac{1- \left(  \frac{R_Z}{\sqrt{L}}\right)^{n}}{1- \frac{R_Z}{\sqrt{L}}} ,\label{eqn:TA2}
\end{align}
since $(1-x^n)/(1-x)$ is increasing in $x>0$ for any $n\ge 1$ and $R_Zm/L\le R_Z/\sqrt{L}$. Since $\Psi$ is given by \eqref{eqn:prodmeasurefm}, identical bounds hold for $\Psi$ with $C_Z$ replaced by $\sqrt{2}C_Z$. Since Assumption \ref{assm:MultiLinear} holds and $ Z^{x,\zeta}_t ,\Psi^{w,\zeta}_t,\hat Z^{x,\zeta,(n-1)}_t$,and $\hat \Psi^{w,\zeta,(n-1)}_t$ are bounded $\mathbb P^0$-a.s., Corollary \ref{corr:LipJ} applies, giving a constant $\Lambda_J= \Lambda_J(R_Z,C_Z,n,\bm q)$ such that
\begin{align*}
    \Bigg\vert\sum_{h=1}^H \xi_{r,h}&\left(J(Z^{x,\zeta}_{\frac{r}{L}},\Psi^{w,\zeta}_{\frac{r}{L}},u,\theta^h_{r,\tau})-J(\hat Z^{x,\zeta,(n-1)}_\frac{r}{L},\hat \Psi^{w,\zeta, (n-1)}_\frac{r}{L},u,\theta^h_{r,\tau})\right)\Bigg\vert\\&\le \Lambda_J \left(\Vert Z^{x,\zeta}_{\frac{r}{L}}-\hat Z^{x,\zeta,(n-1)}_{\frac{r}{L}}\Vert + \Vert \Psi^{w,\zeta}_{\frac{r}{L}}-\Psi^{w,\zeta,(n-1)}_{\frac{r}{L}}\Vert\right)\Bigg\Vert \sum_{h=1}^H \xi_{r,h}\widetilde J\left(\theta^h_{r,\tau}\right)\Bigg\Vert_\mathrm{HS}.
\end{align*}
Therefore, by the Lipschitz continuity of $J$, given above, and the Taylor approximation error \eqref{eqn:TA1}, the first term on the right-hand side of equation \eqref{eqn:mathcalJ1} satisfies
\begin{align*}
        \mathbb E^1 \bigg[\sup_{(x,\zeta,u ,w )\in \mathscr D_2}& \frac{1}{LH}\sum_{l=0}^{\lfloor \frac{r'}{m}\rfloor}\sum_{r=lm}^{(l+1)m-1}\sum_{h=1}^H \mathbbm{1}_{r<r'}\xi_{r,h}\left(J^{x,\zeta,u,w}_{r}(\theta^h_{r,\tau})-J(\hat Z^{x,\zeta,(n-1)}_\frac{r}{L},\hat \Psi^{w,\zeta,(n-1)}_\frac{r}{L},\theta^h_{r,\tau})\right)\bigg]\notag\\
        & \le \frac{3\Lambda_J}{LH}\mathbb E^1 \left[\sum_{l=0}^{\lfloor \frac{r'}{m}\rfloor}\sum_{r=lm}^{(l+1)m-1} \mathbbm{1}_{r<r'}C_Z \left(R_Z \frac{r-lm}{L} \right)^n\Bigg\Vert \sum_{h=1}^H \xi_{r,h}\widetilde J\left(\theta^h_{r,\tau}\right)\Bigg\Vert_\mathrm{HS} \right].
\end{align*}
Since $\{\xi_{0,1},\ldots \xi_{L-1,H}\}$ are independent $\mathcal N(0,1)$ random variables, Jensen's inequality together with Assumption \ref{assm:MultiLinear} imply that
\begin{align}
        \mathbb E^1 \bigg[\sup_{(x,\zeta,u ,w )\in \mathscr D_2}& \frac{1}{LH}\sum_{l=0}^{\lfloor \frac{r'}{m}\rfloor}\sum_{r=lm}^{(l+1)m-1}\sum_{h=1}^H \mathbbm{1}_{r<r'}\xi_{r,h}\left(J^{x,\zeta,u,w}_{r}(\theta^h_{r,\tau})-J(\hat Z^{x,\zeta,(n-1)}_\frac{r}{L},\hat \Psi^{w,\zeta,(n-1)}_\frac{r}{L},\theta^h_{r,\tau})\right)\bigg]\notag\\
        &  \le \frac{3C_Z \Lambda_J}{LH}\sum_{l=0}^{\lfloor \frac{r'}{m}\rfloor}\sum_{r=lm}^{(l+1)m-1}\mathbbm{1}_{r<r'}\left(R_Z\frac{r-lm}{L} \right)^{n} \sqrt{\mathbb E^1\left[  \sum_{h=1}^H \xi_{r,h}^2\left\Vert \widetilde J\left(\theta^h_{r,\tau}\right)\right\Vert_\mathrm{HS}^{2} \right]}\notag\\
        & \le \frac{3\Lambda_J C_4 C_Z R_Z^{n}}{(n+1)\sqrt{H}}\left(\left\lfloor \frac{r'}{m}\right\rfloor +1\right)\left(\frac{m}{L} \right)^{n+1}.\label{eqn:bound_mathcalJ1_2}
\end{align} 
The last step follows from \eqref{eqn:intapproxsum}. Then, by combining equations \eqref{eqn:bound_mathcalJ1_1} and \eqref{eqn:bound_mathcalJ1_2},
there exists a constant $C_{M,J}'$ such that
\begin{align}\label{eqn:EsupJ1final}
    \mathbb E^1 \bigg[\sup_{(x,\zeta,u ,w)\in \mathscr D_2}&\frac{1}{LH}\sum_{h=1}^H \sum_{r=0}^{L-1}\xi_{r,h}J^{x,\zeta,u,w}_r(\theta^h_{r,\tau})\bigg] \le C_{M,J}'\left( \frac{1}{L}+\frac{1}{ \sqrt {mH}}+ \frac{1}{n\sqrt{H}}\left(\frac{m}{L} \right)^{n}\right).
\end{align}
Notice that the expression
\[
\frac{1}{\sqrt{m}}+ \frac{1}{n}\left( \frac{m}{L}\right)^{n}
\]
is maximised at $m= L^{\frac{n}{n+1/2}}2^{-\frac{1}{n+1/2}}$, over $m\le \sqrt{L}$. At this value of $m$, after absorbing any constants into $C'_{M,J}$, we have
\[
\mathbb E^1 \bigg[\sup_{(x,\zeta,u, w) \in \mathscr D_2}\frac{1}{LH}\sum_{h=1}^H \sum_{r=0}^{L-1}\xi_{r,h}J(Z^{x,\zeta}_\frac{r}{L},\Psi^{w,\zeta}_\frac{r}{L},u,\theta^h_{r,\tau})\bigg] \le C'_{M,\Gamma} \left( \frac{1}{L}+\frac{1}{H^{1/2}L^{\frac{n}{2n+1}}}\right).
\]
\end{proof}
\begin{lemma}\label{lemma:doobL2}
Suppose that Assumptions  \ref{assm:StructuredF},\ref{assm:LipCont} and \ref{assm:MultiLinear} hold for $F$ and its associated $J_1,J_2,J_3$ for some $\bm q\in \{0,1,2\}^3$, and that the map $Z^{x,\zeta}_t$ satisfies Assumptions \ref{assm:cont} and \ref{assm:gevrey} for some $n\in \mathbb Z_{>0}$, with $\Psi^{w,\zeta}_t$ defined in \eqref{eqn:prodmeasurefm}. Then, for $M^{x,\zeta}_{l,\tau}$ defined in \eqref{eqn:Martingale},
there exists a constant $C_M^*=C_M^* (B_F,R_Z,C_Z,\bm q,\beta,n,\Lambda_F)$ such that for any $l\in[0:L]$, we have
    \[
    \mathbb E^1 \left[\max_{r\in[0:l]}\sup_{(x,\zeta)\in \mathscr D} \Vert M^{x,\zeta}_{r,\tau}\Vert^2\right]\le C^*_{M}\left(L^{-2} + L^{-\frac{2n}{2n+1}}H^{-1}\right), \qquad \mathbb P^0-a.s.
    \]
\end{lemma}
\begin{proof}
Recall the filtration $\mathcal G_{r}$ defined in equation \eqref{eqn:FiltrationG}.
    Then, $M^{x,\zeta}_{r,\tau}$ is $\mathcal G_{r}$-measurable, as seen in the proof of Lemma \ref{lemma:MeasurableSupremum}, and it is bounded by $2B_F$ by Assumption \ref{assm:LipCont}. Since, $\theta^h_{r,0}$ is independent of $\mathcal G_{r',\tau}$ for any $r'\le r$, the martingale differences satisfy
    \begin{align*}
    \mathbb E \big[ &M^{x,\zeta}_{r+1,\tau}-M^{x,\zeta}_{r,\tau}\big| \mathcal G_{r,\tau}\big]\\&= \frac{1}{H}\sum_{h=1}^H\int_{\frac{r}{L}}^\frac{r+1}{L} \mathbb E\left[ f(Z^{\xi,\zeta}_t,\mathfrak{r}^{\zeta}_t,\Psi_{\frac{r}{L},\tau}[\bm\nu](\theta^h_{r,0}))- F(Z^{\xi,\zeta}_t,\mathfrak{r}^{\zeta}_t,\nu_{\frac{r}{L},\tau})\middle|\mathcal G_{r}\right]dt\\&=0, 
    \end{align*}
after following the same argument as seen in the proof of Lemma \ref{lemma:L1Mstar} to compute the expectation. Therefore, $M_{r,\tau}^{x,\zeta}$ is a $\mathcal G_{r}-$martingale for every $(x,\zeta)\in \mathscr D$. Let us consider
\begin{equation}\label{eqn:Mstar}
    M^*_{r,\tau}:=\sup_{(x,\zeta)\in \mathscr D} \Vert M_{r,\tau}^{x,\zeta}\Vert,
\end{equation}
which is bounded almost surely by $2B_F$ according to Assumption \ref{assm:LipCont}. Furthermore, $M^{*}_{r,\tau}$ is $\mathcal G_{r}$-measurable by Lemma \ref{lemma:MeasurableSupremum}.
As a result, $M^{*}_{r,\tau}$
is a non-negative $\mathcal G_{r}-$sub-martingale, and so applying Doob's $L^2$ inequality, Chapter II Theorem 1.7 in \cite{revuz2013continuous}, yields
\[
\mathbb E^1 \left[\max_{r\in[0:l]}
 \left( M^*_{r,\tau}\right)^2
 \right]\le  4\mathbb E^1 \left[
 \left( M^*_{l,\tau}\right)^2
 \right], \qquad\mathbb P^0-a.s.
\]
Since Assumption \ref{assm:StructuredF} holds, Lemma \ref{lemma:L1Mstar} deduces an $L^1$ upper bound on $M^*_{L,\tau}$ in terms of the functions $J_1, J_2,J_3$. Therefore, applying Lemma \ref{lemma:L1Mstar} followed by Lemma \ref{lemma:TaylorApproxE} on each of the resulting terms, we see that there exists a constant $C_{M,1}$ such that
\begin{equation}\label{eqn:EMstar1}
    \mathbb E^1 \left[ \Vert M^*_{l,\tau}\Vert\right]\le C_{M,1}\left(L^{-1}+ L^{-\frac{n}{2n+1}}H^{-1/2} \right),
\end{equation}
for every $l\in[0:L]$.
However, our claim is an $L^2$ bound on $M^*_{L-1,\tau}$ rather than $L^1$. Therefore, for a fixed $l\in[0:L-1]$, consider $\Upsilon^{x,\zeta}:(\R^{k \times 4d})^{(l+1)H}\to \R$ defined by
\begin{align*}
    \Upsilon^{x,\zeta}( \phi_{0,1},\ldots,\phi_{l,H}) &= \frac{1}{H}\sum_{r=0}^{l}\sum_{h=1}^H \int_\frac{r}{L}^\frac{r+1}{L} \left(F_{3,t,\tau}^{x,\zeta}( \phi_{r,h})-\int F_{3,t}^{x,\zeta}( \theta)\;d\pi(\theta)\right) dt,\\
    F_{3,t,\tau}^{x,\zeta}(\theta) &= F(Z^{x,\zeta}_t,\mathfrak{r}^\zeta_t, \Phi_{t,\tau}[\bm\nu](\theta)),
\end{align*}
and denote the supremum of the norm of $\Upsilon ^{x,\zeta}$ over $(x,\zeta)\in \mathscr D$ by
\begin{align*}
\Upsilon(\phi_{0,1},\ldots,\phi_{l,H})& = \sup_{(x,\zeta)\in \mathscr D} \Vert \Upsilon^{x,\zeta}( \phi_{0,1},\ldots,\phi_{l,H})\Vert.
\end{align*}
 Then, Assumption \ref{assm:LipCont} states that $F$ is bounded by $B_F$, and so
\begin{align*}
     \Vert F_{3,t}^{x,\zeta}(\phi_{r,h})-F_{3,t}^{x,\zeta}(\phi'_{r,h}) \Vert\le 2B_F,
\end{align*}
 for every $\phi_{r,h},\phi_{r,h}'\in \mathfrak{S}$,
which immediately gives
\[
\Vert\Upsilon^{x,\zeta}(\phi_{0,1},\ldots,\phi_{r,h},\ldots, \phi_{l,H}) -\Upsilon^{x,\zeta}(\phi_{0,1},\ldots,\phi'_{r,h},\ldots, \phi_{l,H})\Vert \le \frac{2B_F}{LH}
.\]
Therefore, the triangle inequality implies that for any $(x,\zeta)\in \mathscr D$
\begin{align*}
    \Vert \Upsilon^{x,\zeta}(\theta_{0,1},\ldots,\theta_{r,h},\ldots, \theta_{l,H})\Vert &\le \Vert \Upsilon^{x,\zeta}(\theta_{0,1},\ldots,\theta'_{r,h},\ldots, \theta_{l,H})\Vert+ \frac{2B_F}{LH}\\
    &\le \Upsilon(\theta_{0,1},\ldots,\theta'_{r,h},\ldots, \theta_{l,H})+ \frac{2B_F}{LH}.
\end{align*}
 Since the above expression holds for every $(x,\zeta)\in \mathscr D$, by the definition of the supremum we have that
\[
\Upsilon(\phi_{0,1},\ldots,\phi_{r,h},\ldots, \phi_{l,H}) \le  \Upsilon(\phi_{0,1},\ldots,\phi'_{r,h},\ldots, \phi_{l,H})+ \frac{2B_F}{LH}.
\]
The same bound holds with $\phi_{r,h}, \phi_{r,h}'$ interchanged, and so $\Upsilon$ satisfies \textit{the bounded differences property}
\begin{align*}
    |\Upsilon( \phi_{0,1},\ldots ,\phi_{r,h},\ldots,\phi_{l,H})-  \Upsilon( \phi_{0,1},\ldots,\phi'_{r,h},\ldots,\phi_{l,H})| \le \frac{2B_F}{LH}.
\end{align*}
Notice that
\[
\Upsilon(\theta^1_{0,0}, \ldots , \theta^H_{l,0}) = M^*_{l+1,\tau},
\]
where $\theta^1_{0,0}, \ldots , \theta^H_{l,0}$ are the parameters of the discrete-time model at initialisation and so are $i.i.d.$.
Therefore, McDiarmid's inequality, Lemma 1.2 in \cite{McDiarmid_1989} , applies
and yields the tail bound
\begin{equation}\label{eqn:tailbound}
   \mathbb P^1\left( \left|M^*_{l+1,\tau}-\mathbb E^1 M_{l+1,\tau}^*\right|\ge t \right)\le 2\ exp\left(-\frac{LHt^2}{2B_F^2} \right),
\end{equation}
for any $t>0$. Denote the variance of $M^*_{l+1,\tau}$ under $\mathbb P^1$ by $\mathbb{V}^1(M^*_{l+1,\tau})$. This can be bounded above by integrating the tail bound over $t\in(0,\infty)$ to obtain
\begin{align}\label{eqn:Varbound}
    \mathbb{V}^1(M^*_{l+1,\tau}) &= \int_{0}^\infty \mathbb P^1\left(|M^*_{l+1,\tau} -\mathbb E^1 M^*_{l+1,\tau}|^2\ge t\right) dt
    \le \int_0^\infty 2e^{-\frac{LH t}{2 B^2_F}} dt 
     = \frac{4 B_{F}^2}{LH} .
\end{align}
The Lemma concludes by substituting bounds \eqref{eqn:EMstar1} and \eqref{eqn:Varbound} into
\[
\mathbb E^1 \left[\max_{r\in[0:l]}
 \left( M^*_{r,\tau}\right)^2
 \right]\le  4
 \left(\mathbb E^1\left[ M^*_{l,\tau}\right]\right)^2+ 4 \mathbb{V}^1(M^*_{l,\tau}).
\]
\end{proof}
\begin{remark}\label{remark:mathcalK}
The second term in $\mathcal K$ is not of the form given in equation \eqref{eqn:F}. Instead, it can be written as
\begin{equation}\label{eqn:Fstar} f(z,\rho,\theta)= \int \frac{ \exp(f_1(z,w,\theta)) \; f_2(z,w,\theta)}{F_2(w,\rho,\theta)}d\rho^{\otimes 2}( w).
\end{equation}
The difference is that the first argument of the denominator $F_2$ is $w$ rather than $z$, and thus is integrated over. However, this discrepancy corresponds to replacing $Z^{x,\zeta}_t$ with $\Psi^{w,\zeta}_t$ in the denominator of $F_2$ appearing in equation \eqref{eqn:EnormK}. Since $\Psi_{t}^{(w_1,w_2),\zeta}=(Z_t^{w_1,\zeta},Z_t^{w_2,\zeta})$, it inherits the same properties as $Z_t^{x,\zeta}$. Thus, we may replace $Z^{x,\zeta}_t$ by $\Psi^{w,\zeta}_t$ without effecting the subsequent analysis. This implies that Lemma \ref{lemma:doobL2} still holds if Assumption \ref{assm:StructuredF} is changed so that $f$ satisfies \eqref{eqn:Fstar} instead of \eqref{eqn:f}.
\end{remark}
    \section{Stochastic Approximation Lemma}\label{App:I}
Let $\bm{y}=(y_1,\ldots,y_N)\in (\bar B(R_0))^N$ and denote the empirical measure on $y_1,\ldots ,y_N$ by $\zeta^N$. Then, according to \eqref{eqn:IPSvsFM}, the solution $(X^{i,\bm{y}}_{t,\tau},a^{i,\bm{y}}_{t,\tau})$ to \eqref{eqn:ansatz} corresponds to the solution to \eqref{eqn:MVRODE1}-\eqref{eqn:MVRODE2}, with initial condition $x=y_i$ and $\zeta=\zeta^N$. Thus, for all $t\in[0,1]$, we define 
\[
Z^{i,\bm{y}}_t := Z^{y_i,\zeta^N}_t, \qquad \mathfrak{m}^{N,\bm{y}}_t := (Z^{\cdot,\zeta^N}_t)_\#\zeta^{N}= \frac{1}{N}\sum_{i=1}^N\delta_{Z^{i,\bm{y}}_t},
\]
where $Z^{x,\zeta}_t$ is the map defined in Appendix \ref{App:G}. 

Furthermore, let us introduce a piecewise constant approximation to $Z^{i,\bm{y}}_t$ that corresponds to the discrete-time transformer \eqref{eqn:IPS}. For each $i\in[1:N]$, define $\bar Z_t^{i,L,\bm{y}}:\Omega\times [0,1]\mapsto \R^{2d}$ by 
\[
\bar Z^{i,L,\bm{y}}_t := \sum_{r=0}^L\mathbbm{1}_{Lt\in[r,r+1)}z^{i,L,\bm{y}}_r,
\]
where each $z^{i,L,\bm{y}}_r$ is a random variable taking values in $\R^{2d}$ and defined on the complete probability space $(\Omega,\mathcal F,\mathbb P)$ that was introduced in Section \ref{sec:MainTheorem}. Since $\bar Z^{i,L,\bm{y}}_t$ is a finite sum of $(\mathcal F \otimes \mathcal B([0,1]))$-measurable functions, $(\omega,t)\mapsto \bar Z^{i,L,\bm{y}}_t(\omega)$ is  $(\mathcal F\otimes \mathcal B([0,1]))$-measurable. We define the empirical measure associated to $((\bar Z^{i,L,\bm{y}})_{i=1}^N)_{t\in[0,1]}$ by
\[
\bar m^{N,L,\bm{y}}_t:= \frac{1}{N}\sum_{i=1}^N\delta_{\bar Z^{i,L,\bm{y}}_t}.
\]
Furthermore, we assume that the map $\bm{y}\mapsto (\bar Z^{i,L,\bm{y}}_t(\omega))_{i=1}^N$ is continuous for every $\omega\in \Omega$ and $t\in[0,1]$.
The following Lemma combines the estimates from Appendices \ref{App:F} and \ref{App:H} to bound the difference between the continuous-time mean-field model and its discrete counterpart. 
\begin{lemma}\label{lemma:StochasticApprox}
Suppose that $F$ with its associated $f,J_1,J_2,J_3$ satisfy Assumptions \ref{assm:StructuredF},\ref{assm:LipCont}, and \ref{assm:MultiLinear}. Suppose further that Assumptions \ref{assm:cont} and \ref{assm:gevrey} hold for the flow map $Z^{x,\zeta}_t$ for some $n \ge 1$. Then, for \\$r\in[0:L-1],i\in[1:N]$ and $ \tau \in [1:T]$, we define 
\begin{align}\label{eqn:Delta_r'_i}
\Delta_{r,\tau}^{i,\bm{y}} &= \left\Vert \int_0^\frac{r+1}{L}F(Z^{i,\bm{y}}_s,\mathfrak{m}^{N,\bm{y}}_s,\nu_{s,\tau})-F(\bar Z^{i,L,\bm{y}}_s,\bar m^{N,L,\bm{y}}_s,\bar\nu^{N,L}_{s,\tau})ds\right\Vert^2,
\end{align}
where $\bm\nu_\tau,\bm{\bar \nu}^{N,L}_\tau$ are the parameters of models \eqref{eqn:ansatz} and \eqref{eqn:CTE2} at training step $\tau$, respectively. 
Then, there exists a constant $C_5=C_5(\Lambda_F,D_1,D_2,C_Z,R_Z, C^*_M)$ such that for any training step \\$\tau \in [1: T]$ and $l \in [0:L-1]$
\begin{align*}
    \mathbb E^1\left[\max_{r\in[0:l]}\sup_{\bm{y}\in (\bar B(R_0))^N}\max_{i\in[1:N]} \Delta_{r,\tau}^{i,\bm{y}} \right] &\le \frac{C_5}{L}\sum_{r'=0}^l\mathbb E^1\left[\max_{ r \in[0:r']}\max_{\bm{y}\in (\bar B(R_0))^N}\max_{i \in[1:N]} \Vert Z^{i,\bm{y}}_{\frac{r}{L}}-\bar Z^{i,L,\bm{y}}_{\frac{r}{L}}\Vert^2\right]\\&+ C_5\left(\frac{1}{L^2}+\frac{1}{HL^{\frac{2n}{2n+1}}}+\sum_{j=0}^{\tau-1}\kappa^0_{j,\tau}\mathscr{L}_j\right),
\end{align*}
$\mathbb P^0$-almost surely, where $\kappa^0_{i,T}$ is defined in \eqref{eqn:kappa} and $\mathscr{L}_j$ is defined by
\begin{equation}\label{eqn:mathscr_L}
    \mathscr{L}_{\tau} =\mathbb E^1 \left[\sup_{t\in [0,1]}\sup_{\bm{y}\in (\bar B(R_0))^N}\max_{i\in[1:N]} \Vert X^{i,\bm{y}}_{t,\tau}-\bar X^{i,L,\bm{y}}_{t,\tau}\Vert^2+\Vert a^{i,\bm{y}}_{t,\tau}-\bar a^{i,L,\bm{y}}_{t,\tau}\Vert^2\right].
\end{equation}
\end{lemma}
\begin{proof}
Since $(t,\bm{y})\mapsto Z^{i,\bm{y}}_{t,\tau}$ and $\bm{y}\mapsto\bar Z^{i,L,\bm{y}}_{\tau}$ are continuous for every $\omega\in\Omega$, by assumption, and all subsequent variables depend continuously on $Z^{i,\bm{y}}_{t,\tau}$ and $\bar Z^{i,L,\bm{y}}_{t,\tau}$, it follows that all suprema over $(t,\bm{y})\in[0,1]\times (\bar B(R_0))^N$ are $\mathcal F$-measurable. This is because $[0,1]\times (\bar B(R_0))^N$ is separable and so the suprema can be taken over a countable dense subset.

As discussed in Section 2.1.3 in \cite{carmona2018probabilistic} Vol. II., for any $\mathcal F$-measurable random variable $\mathcal Z$, the section $\mathcal Z(\omega^0,.)$ is $\mathcal F^1$-measurable $\mathbb P^0$-a.e. in $\omega^0$. Therefore, there exists a set subset $\Omega^0_{\mathrm{w.d.,2}}\subseteq \Omega^0$ with full $\mathbb P^0$-measure such that the expectation with respect to $\mathbb P^1$ on all subsequent variables is well-defined for all $\omega^0\in \Omega^0_{\mathrm{w.d.,2}}$.   \\

Let us expand $\Delta^{i,\bm{y}}_{r',\tau}$ into different sources of error according to
\begin{align}\label{eqn:AllErrors}
     \Delta^{i,\bm{y}}_{r',\tau}&\le 3\underbrace{\left\Vert \int_0^\frac{r'+1}{L} F(Z^{i,\bm{y}}_s,\mathfrak{m}^{N,\bm{y}}_s,\nu_{s,\tau})-F( Z^{i,\bm{y}}_s, \mathfrak{m}^{N,\bm{y}}_s,\nu_{r_s,\tau})ds\right\Vert^2}_{(c)^{i,\bm{y}}_{r'}}\\&+3\underbrace{\left\Vert \int_0^{\frac{r'+1}{L}} F( Z^{i,\bm{y}}_s, \mathfrak{m}^{N,\bm{y}}_s,\nu_{r_s,\tau})-F(Z^{i,\bm{y}}_s, \mathfrak{m}^{N,\bm{y}}_s,\hat\nu^{N,L}_{r_s,\tau})ds\right\Vert^2}_{(d)^{i,\bm{y}}_{r'}}\notag\\
     &+3\underbrace{\left\Vert \int_0^{\frac{r'+1}{L}} F\left( Z^{i,\bm{y}}_s,\mathfrak{m}^{N,\bm{y}}_s,\hat\nu^{N,L}_{r_s,\tau}\right)-F(\bar Z^{i,L,\bm{y}}_s,\bar m^{N,L,\bm{y}}_s,\bar\nu^{N,L}_{r_s,\tau})ds\right\Vert^2}_{(e)^{i,\bm{y}}_{r'}}\notag.
 \end{align}
 For the term labeled $(c)$, applying the Cauchy-Schwarz inequality and the Lipschitz continuity of $F$, Assumption \ref{assm:LipCont}, gives
 \begin{equation*}
     (c)^{i,\bm{y}}_{r'}\le \Lambda_F^2  \frac{r'+1}{L}\int_0^\frac{r'+1}{L}  W^2_2(\nu_{s,\tau},\nu_{r_s,\tau})ds, \qquad \mathbb P-a.s.
 \end{equation*}
 Lemma \ref{lemma:GradCont} states that $ W_2(\nu_{s,\tau},\nu_{t,\tau})$ is $D_2$-Lipchitz continuous uniformly in training step $\tau$, and so  we have that
\begin{equation}\label{eqn:(c)}
       (c)^{i,\bm{y}}_{r'}\le \Lambda_F^2 D_2^2L^{-2}, \qquad \mathbb P-a.s.
 \end{equation}
 For the term labelled $(e)$, we apply the Cauchy-Schwarz inequality and the Lipschitz continuity of $F$ to get
\begin{align*}
     (e)^{i,\bm{y}}_{r'}\le  \frac{3\Lambda_F^2(r'+1)}{L}\int_{0}^{\frac{r'+1}{L}} \Vert Z^{i,\bm{y}}_t-\bar Z^{i,L,\bm{y}}_{t}\Vert^2+ W_\infty^2\left(\mathfrak{m}^{N,\bm{y}}_t,\bar m^{N,L,\bm{y}}_t \right)+ W_2^2(\hat \nu^{N,L}_{r_t,\tau}, \bar \nu^{N,L}_{r_t,\tau})dt.
 \end{align*}
  Since $\hat \nu^{N,L}_{t,0}=\bar \nu^{N,L}_{t,0}$, the Wasserstein distance between $\hat \nu^{N,L}_{\frac{r}{L},\tau}$ and $ \bar \nu^{N,L}_{\frac{r}{L},\tau}$ can be bounded using the coupling induced by pushing forward $\bar \nu^{L,H}_{t,0}$ by their respective flow maps. This coupling yields
 \begin{align*}
   W_2^2(\hat \nu^{N,L}_{\frac{r}{L},\tau}, \bar \nu^{N,L}_{\frac{r}{L},\tau}) \le \int\Vert \Phi_{\frac{r}{L},\tau}[\bm\nu](\theta)-\Phi_{\frac{r}{L},\tau}[\bm{\bar\nu}^{N,L}](\theta)\Vert_\mathrm{F}^2 d\bar\nu^{N,L}_{\frac{r}{L},0}(\theta).
\end{align*}
By applying the Lipschitz continuity of the flow map $\Phi[\bm\nu]$ in the measure argument, Lemma \ref{lemma:AdamLip} whose hypotheses are satisfied because of Lemma \ref{lemma:GradCont}, we obtain
\begin{align}\label{eqn:boundparams}
    W_2^2(\hat \nu^{N,L}_{\frac{r}{L},\tau}, \bar \nu^{N,L}_{\frac{r}{L},\tau})\le \frac{D_1}{B} \sum_{j=0}^{\tau-1}\sum_{b=1}^B \kappa^0_{j,\tau}W^2_{\infty}(\bar\rho^{N,L,b}_{\frac{r}{L},j},\rho^{N,b}_{\frac{r}{L},j}),
\end{align}
where $\kappa^0_{j,\tau}$ is defined in \eqref{eqn:kappa} and $D_1$ is the constant from Lemma \ref{lemma:AdamLip}. Therefore, $(e)$ is $\mathbb P$-a.s. bounded above by
\begin{align}\label{eqn:(e)intermediate}
 (e)^{i,\bm{y}}_{r'}&\le 3\Lambda_F^2\left(\int_{0}^{\frac{r'+1}{L}} 2 W_\infty^2\left(\mathfrak{m}^{N,\bm{y}}_t,\bar m^{N,L,\bm{y}}_t \right)dt+\frac{D_1}{BL}\sum_{r=0}^{r'}\sum_{j=0}^{\tau-1}\sum_{b=1}^B \kappa^0_{j,\tau}  W^2_{\infty}(\bar\rho^{N,L,b}_{\frac{r}{L},j},\rho^{N,b}_{\frac{r}{L},j})\right).
\end{align}
 Since $\bar Z^{i,L,\bm{y}}_t$ is constant over the interval $[r_t,t]$ and $Z^{i,\bm{y}}_{t}$ is $C_Z R_Z$-Lipschitz continuous in $t$ due to Assumption \ref{assm:gevrey}, Young's inequality implies that
\[
\Vert Z^{i,\bm{y}}_{t}-\bar Z^{i,L,\bm{y}}_{t}\Vert^2\le 2\Vert Z^{i,\bm{y}}_{r_t}-\bar Z^{i,L,\bm{y}}_{r_t}\Vert^2+2\frac{(C_Z R_Z)^2}{L^2}, \qquad \mathbb P-a.s.
\]
Notice that the expectation under $\mathbb P^1$, averaged over the batch of training data, of the $\infty$-Wasserstein distance between $\bar\rho^{N,L,b}_{\frac{r}{L},j}$ and $\rho^{N,b}_{\frac{r}{L},j}$ is less than or equal to $\mathscr L_j$ defined in equation \eqref{eqn:mathscr_L} for every $r\in[0:L-1]$.
Therefore, for $\mathbb P^0$-a.e. $\omega^0$ and adjusting for the time discretisation error, the expectation of the right-hand side of equation \eqref{eqn:(e)intermediate} under $\mathbb P^1$ is bounded according to
\begin{align}
    \mathbb E^1\bigg[ \max_{r'\in[0: l]} \sup_{\bm{y}\in (\bar B(R_0))^N}\max_{ i \in [1:N]}(e)^{i,\bm{y}}_{r'}\bigg] &\le \frac{12\Lambda_F^2}{L} \sum_{r=0}^l\mathbb E^1 \left[\max_{r'\in[0:r]}\sup_{\bm{y}\in(\bar B(R_0))^N}\max_{i\in[1:N]}\Vert Z^{i,\bm{y}}_{\frac{r'}{L}}-\bar Z^{i,L,\bm{y}}_{\frac{r'}{L}}\Vert^2\right]\notag\\& +3\Lambda_F^2 \left(\frac{4(C_Z R_Z)^2}{L^2}+ D_1\sum_{j=0}^{\tau-1}\kappa^0_{j,\tau} \mathscr{L}_j \right).\label{eqn:(e)}
 \end{align} 
 Recall the definition of $M^{x,\zeta}_{r,\tau}$ given by equation \eqref{eqn:Martingale}. Notice that, for every $i\in[1:N]$, \\$\bm{y}\in(B(R_0))^N$, $(d)^{i,\bm{y}}_{r'}$ is $\mathbb P$-a.s. bounded above by the supremum of $\Vert M^{x,\zeta}_{r'+1,\tau}\Vert^2$ over $(x,\zeta )\in \mathscr D$, denoted $M^{*}_{r'+1,\tau}$.
The current setting satisfies the assumptions for Lemma \ref{lemma:doobL2} (by design), thus Lemma \ref{lemma:doobL2} applies and gives
\begin{align}\label{eqn:(d)}
    \mathbb E^1 \left[\max_{r\in[0:l]}\sup_{\bm{y}\in (\bar B(R_0))^N}\max_{i\in[1:N]}(d)^{i,\bm{y}}_{r'}\right] \le C_M^* \left( L^{-2}+ H^{-1}L^{-\frac{2n}{2n+1}}\right),
\end{align}
where $n$ is the constant in Assumption \ref{assm:gevrey}. The Lemma concludes by combining the bounds \eqref{eqn:(c)}, \eqref{eqn:(e)} and \eqref{eqn:(d)} according to equation \eqref{eqn:AllErrors}.
\end{proof}
    \section{Proof of Theorem \ref{theorem:WTS}}\label{app:J}
Suppose that $(X^{i,\bm{y}}_{t,\tau},a^{i,\bm{y}}_{t,\tau})_{i=1}^N,(\bar X^{i,L,\bm{y}}_{t,\tau},\bar a^{i,L,\bm{y}}_{t,\tau})_{i=1}^N$ are the solutions to \eqref{eqn:ansatz} and  \eqref{eqn:CTE1}-\eqref{eqn:CTE2}, given the initial condition $\bm y\in (\bar B(R_0))^N$. The first of which exists by Lemma \ref{lemma:GlobalWellDefined}. A standard Gr\"onwall argument using the Lipschitz continuity of $\Gamma,\mathcal K$ and $\partial_\mu\ell$ yields that, after modifying it on a $\mathbb P$-null set if necessary, $\bm{y}\mapsto (\bar X^{i,L,\bm{y}}_{t,\tau},\bar a^{i,L,\bm{y}}_{t,\tau})$ is continuous for every $\omega\in \Omega$. This can be done as $\mathcal F$ is complete. Furthermore, $(\bar X^{i,L,\bm{y}}_{t,\tau},\bar a^{i,L,\bm{y}}_{t,\tau})$ is $\mathcal F$-measurable by Lemma \ref{lemma:GlobalWellDefined}. Therefore, the assumptions on $\bar Z^{i,L,\bm{y}}_{t}$ detailed at the start of Appendix \ref{App:I} are satisfied by $(\bar X^{i,L,\bm{y}}_{t,\tau},\bar a^{i,L,\bm{y}}_{t,\tau})$ .

Recall that the continuous time extensions \eqref{eqn:CTE1}-\eqref{eqn:CTE2} to the hidden state and adjoint variables round $t$ up and down, respectively, so that
\[
\bar Z^{i,L,\bm{y}}_{t,\tau} = (\bar X^{i,L,\bm{y}}_{t,\tau},\bar a^{i,L,\bm{y}}_{t,\tau}) = (\bar X^{i,L,\bm{y}}_{r_t,\tau},\bar a^{i,L,\bm{y}}_{r_t+\frac{1}{L},\tau}).
\]
In Lemma \ref{lemma:StochasticApprox}, the time component is discretised by approximating $Z^{i,\bm{y}}_t$ with $Z^{i,\bm{y}}_{r_t}$. However, an identical bound would occur if we rounded time up and down in half of the components each.

By Young's inequality, we have that
\begin{align*}
    \Vert a^{i,\bm{y}}_{\frac{r}{L},\tau}-\bar a^{i,L,\bm{y}}_{\frac{r}{L},\tau}\Vert^2&\le 2\Vert a^{i,\bm{y}}_{1,\tau}-\bar a^{i,L,\bm{y}}_{1,\tau}\Vert^2\\& +2\left\Vert \int_\frac{r}{L}^1  \mathcal K(X^{i,\bm{y}}_{t,\tau},\rho^{N,\bm{y}}_{t,\tau},\nu_{t,\tau},a^{i,\bm{y}}_{t,\tau})- \mathcal K(\bar X^{i,L,\bm{y}}_{t,\tau},\bar \rho^{N,L,\bm{y}}_{t,\tau},\bar\nu^{N,L}_{t,\tau},\bar a^{i,L,\bm{y}}_{t,\tau})dt\right\Vert^2.
\end{align*}
 The function $\mathcal K$ is bounded and Lipschitz as a result of Lemma \ref{lemma:BoundedGamma}, Corollary \ref{cor:LipmathcalK} and Remark \ref{remark:GlobalLip}, thus $\mathcal K$ satisfies Assumption \ref{assm:LipCont}. Recall that equation \eqref{eqn:IPSvsFM} identifies $X^{i,\bm{y}}_{t,\tau},a^{i,\bm{y}}_{t,\tau}$ with the solution $Z^{x,\zeta}_{t,\tau}=(X^{x,\zeta}_{t,\tau},a^{x,\zeta}_{t,\tau})$ to equations \eqref{eqn:FMY}-\eqref{eqn:FMp}, with $\bm\nu$ taken as $\bm\nu_\tau$, $x=y^i$ and $\zeta$ taken as the empirical measure on $y^1,\ldots y^N$. Corollary \ref{cor:FlowMap_wd} demonstrates that the map $(x,\zeta)\mapsto Z^{x,\zeta}_{t,\tau}$ is Lipschitz, which can be extended to include Lipschitz in time by the boundedness of $\mathcal K$ and $\Gamma$. After modifying $Z^{x,\zeta}_t$ on a $\mathbb P^0$, if necessary, this Lipschitz estimate holds for all $\omega^0\in \Omega^0$. Corollary \ref{cor:FlowMap_wd} also proves that $Z^{x,\zeta}_{t,\tau}$ is $\mathcal F^0$-measurable for fixed $(t,x,\zeta)\in [0,1]\times \mathscr{D}$. Therefore, $Z^{x,\zeta}_{t,\tau}$ satisfies \ref{assm:cont}. Furthermore, Lemmas \ref{lemma:BoundedGamma} and \ref{lemma:GlobalWellDefined} combined imply that $\Gamma$ and $\mathcal K$ are bounded by $R_\theta^2R_X$ and $ R_a\tilde{B}_\mathcal{K}$, respectively. Hence, for any $(x,\zeta)\in \mathscr D$
 \[
 \Vert \partial_t Z^{x,\zeta}_{t,\tau}\Vert \le R_\theta^2R_X+R_a\tilde{B}_\mathcal{K}.
 \]
 Therefore, the flow map $Z^{x,\zeta}_{t,\tau}$ satisfies Assumption \ref{assm:gevrey} with $n=1$.\\
 
  Let $z=(x,a)\in\R^{2d}$, $w=(x,a,\tilde{x},\tilde{a})\in\R^{4d}$. We use subscript $j$ to index the $j$-th block of $d$ components so that odd subscripts correspond to the hidden states $X^{i,b}_{t,\tau}$, while even subscripts correspond to the adjoint variables $a^{i,b}_{t,\tau}$. Then, $\nabla_x\mathcal H(x,\rho,\nu,a)$ satisfies \ref{assm:StructuredF} with
\begin{align}\label{eqn:fgz}
    \begin{cases}
    &f^{\gamma z}_1(z,w,\theta) =f^{\gamma z}_3(z,w,\theta) = \beta \langle \theta_Q z_1,\theta_K(w_1+w_3) \rangle,  \\&f^{\gamma z}_2(z,w,\theta) =\beta (\theta_Q )^T\theta_K \,(w_1-w_3)\langle\theta_V w_1, \theta_O z_2\rangle.
    \end{cases}
\end{align}
These functions $f_1,f_2,f_3$ are clearly jointly continuous in each variable and satisfy the boundedness requirement of Assumption \ref{assm:StructuredF}.
These choices of $f_1,f_2,f_3$ ensure that the corresponding $J_1,J_2,J_3$, defined in \ref{def:J}, are linear in each variable or linear in $ww^T$. Therefore, each $J_1,J_2,J_3$ satisfies Assumption \ref{assm:MultiLinear}, where the associated $\widetilde J$ will be bounded by a constant that only depends on $R_\theta,\beta$ due to Remark \ref{remark:GlobalLip}. See Section \ref{sec:Jtilde} for further details.\\

The second term in $\mathcal K$ takes the form
\[
\int \partial_\mu \mathcal H(y,\rho|_x,\nu,p)(x)d\rho(y,p) =  \iint \frac{\exp( f^{\gamma \mu}_1((x,a),w,\theta))\, f^{\gamma \mu}_2((x,a),w,\theta)}{\int \exp( f^{\gamma \mu}_3(w,\tilde w, \theta))d\rho^{\otimes 2}(\tilde w)}d\rho^{\otimes 2}(w) d\nu(\theta),
\]
where
\begin{align}\label{eqn:fgmu}
\begin{cases}
    f_1^{\gamma \mu}(z,w,\theta) &=  \beta \langle \theta_Q w_1, \theta_K (z_1+w_3)\rangle,\\
    f_2^{\gamma \mu}(z,w,\theta) &= \beta (\theta_K)^T  \theta_Q w_1 \langle \theta_V (z_1-w_3), \theta_O w_2 \rangle+ \theta_V^T\theta_Ow_2,\\
     f_3^{\gamma \mu}(w,\tilde w,\theta) &=  \beta \langle \theta_Q w_1, \theta_K (\tilde w_1+\tilde w_3)\rangle.
\end{cases}
\end{align}
These choices of $f_1,f_2,f_3$ only satisfy the continuity and boundedness requirement for Assumption \ref{assm:StructuredF}.  
Nevertheless, as noted in Remark \ref{remark:mathcalK}, the conclusion of Lemma \ref{lemma:L1Mstar} still holds for the second term of $\mathcal K$.
Furthermore, there does not exist a singular $\widetilde{J}$ such that $J_1^{\gamma\mu}$ or $J^{\gamma\mu}_2$ satisfy Assumption \ref{assm:MultiLinear}, where these functions correspond to $f_1^{\gamma\mu}$ and $f_2^{\gamma\mu}$ according to definition \eqref{def:J}. However, as shown in Section \ref{sec:Jtilde}, $J_1^{\gamma\mu}, J^{\gamma\mu}_2$ and $J_3^{\gamma\mu}$ can be expressed as a sum of at most three terms of the form of $\widetilde{J}$ in Assumption \ref{assm:MultiLinear}. As a result, Remark \ref{remark:subadditivity} states that Lemma \ref{lemma:TaylorApproxE} still applies for the given $J^{\gamma\mu}_j$, $j=1,2,3$. Since these are the only instances in which these assumptions are used, the conclusion of Lemma \ref{lemma:doobL2} remains valid. These Assumptions \ref{assm:StructuredF} and \ref{assm:MultiLinear} are only required in Lemma \ref{lemma:StochasticApprox} so that  Lemma \ref{lemma:doobL2} can be applied.

Accordingly, Lemma \ref{lemma:StochasticApprox} can be applied to both terms of $\mathcal K$, and so, combining them by Young's inequality, we deduce that there exists a constant $C_6$ such that
\begin{align}\label{eqn:boundaintermediate}
&\mathbb E^1 \bigg[\max_{r\in[l:L]}\sup_{\bm{y}\in (\bar B(R_0))^N}\max_{i\in[1:N]}\Vert a^{i,\bm{y}}_{\frac{r}{L},\tau} - \bar a^{i,L,\bm{y}}_{\frac{r}{L},\tau}\Vert^2\bigg] \\&\le \mathbb E^1 \bigg[\sup_{\bm{y}\in (\bar B(R_0))^N}\max_{i\in[1:N]} \Vert a^{i,\bm{y}}_{1,\tau} - \bar a^{i,L,\bm{y}}_{1,\tau}\Vert^2 \bigg]  +\frac{C_6}{L}\sum_{r'=l}^{L-1} \mathscr{W}_{r',\tau} +C_6\left(\frac{1}{L^2}+\frac{1}{HL^{\frac{2}{3}}}+\sum_{j=0}^{\tau-1}\kappa^0_{j,\tau}\mathscr{L}_j\right), \notag
\end{align}
where $\mathscr{L}_j$ is defined in \eqref{eqn:mathscr_L} and $\mathscr{W}_{l,\tau}$ is defined by
\begin{equation}
    \mathscr{W}_{l,\tau} = \mathbb E^1\left[\max_{r\in[l:L-1]}\sup_{\bm{y}\in (\bar B(R_0))^N}\max_{i\in[1:N]}\Vert X^{i,\bm{y}}_{\frac{r}{L},\tau}-\bar X^{i,L,\bm{y}}_{\frac{r}{L},\tau}\Vert^2+\Vert a^{i,\bm{y}}_{\frac{r+1}{L},\tau}-\bar a^{i,L,\bm{y}}_{\frac{r+1}{L},\tau}\Vert^2 \right].
\end{equation}
Here, we use the convention $\sum_{r'=L}^{L-1} \mathscr{W}_{r',\tau}=0$. Note that the time steps on the hidden state and adjoint variables have been rounded up and down, respectively, as discussed at the start of this section.\\

The argument for the difference in hidden states follows similarly. Since $X^{i,\bm{y}}_{0,\tau} = \bar X^{i,L,\bm{y}}_{0,\tau}$, for every $r\in[0:L-1]$
\[
\Vert X^{i,\bm{y}}_{\frac{r+1}{L},\tau} - \bar X^{i,L,\bm{y}}_{\frac{r+1}{L},\tau}\Vert^2 = \left\Vert \int_0^\frac{r+1}{L}\Gamma(X^{i,\bm{y}}_{t,\tau}, \mu^{N,\bm{y}}_{t,\tau},\nu_{t,\tau})-\Gamma(\bar X^{i,L,\bm{y}}_{t,\tau},\bar \mu^{N,L,\bm{y}}_{t,\tau},\bar \nu^{N,L}_{t,\tau})dt\right\Vert^2.
\]
By Lemmas \ref{lemma:BoundedGamma},\ref{lemma:Lipderivs}  and Remark \ref{remark:GlobalLip}, the velocity field $\Gamma$ is Lipschitz and bounded and so satisfies Assumption \ref{assm:LipCont}. Furthermore, $\Gamma$ satisfies Assumption \ref{assm:StructuredF} with
\begin{align*}
    f_1^\gamma(z,w,\theta) &=f_3(z,w,\theta)= \beta \langle \theta_Q z_1,\theta_K w_1 \rangle,\\
    f_2^\gamma(z,w,\theta) &=(\theta_O)^T\theta_V w_1 ,
\end{align*}
which are all linear in $z,w$. Hence, the associated functions $J^\gamma_1,J^\gamma_2,J^\gamma_3$, as defined in \ref{def:J}, satisfy Assumption \ref{assm:MultiLinear} for some $\widetilde J(\theta)$ that is bounded in terms of $\beta$, $R_\theta$, as a result of Remark \ref{remark:GlobalLip}.
Therefore, Lemma \ref{lemma:StochasticApprox} applies with $Z^{i,\bm{y}}_{t,\tau}, \bar Z^{i,L,\bm{y}}_{t,\tau}$ taken as $X^{i,\bm{y}}_{t,\tau}$ and $\bar X^{i,L,\bm{y}}_{t,\tau}$, respectively, to yield that there exists a constant $C_7$ such that
\begin{align*}
    \mathbb E^1\bigg[&\max_{i\in[1:N]}\sup_{\bm{y}\in (\bar B(R_0))^N}\max_{r\in[0:l+1]}\Vert X^{i,\bm{y}}_{\frac{r}{L},\tau} - \bar X^{i,L,\bm{y}}_{\frac{r}{L},\tau}\Vert^2 \bigg] \\&\le  \frac{C_7}{L}\sum_{r'=0}^l\mathbb E^1\left[\max_{ r \in[0:r']}\sup_{\bm{y}\in (\bar B(R_0))^N}\max_{i \in[1:N]} \Vert X^{i,\bm{y}}_{\frac{r}{L},\tau}-\bar X^{i,L,\bm{y}}_{\frac{r}{L},\tau}\Vert^2\right]\notag+C_7\left(\frac{1}{L^2}+\frac{1}{HL^{\frac{2}{3}}}+\sum_{j=0}^{\tau-1}\kappa^0_{j,\tau}\mathscr{L}_j\right).
\end{align*}
By applying the discrete Gr\"onwall's Lemma, we deduce that
\begin{equation}\label{eqn:boundXfinal}
       \mathbb E^1\bigg[\max_{i\in[1:N]}\sup_{\bm{y}\in (\bar B(R_0))^N}\max_{r\in[0:L]}\Vert X^{i,\bm{y}}_{\frac{r}{L},\tau} - \bar X^{i,L,\bm{y}}_{\frac{r}{L},\tau}\Vert^2\bigg] \le  C_7\left(\frac{1}{L^2}+\frac{1}{HL^{\frac{2}{3}}}+\sum_{j=0}^{\tau-1}\kappa^0_{j,\tau}\mathscr{L}_j\right),
\end{equation}
where we have absorbed the $\exp(C_7)$ factor from the Gr\"ownall Lemma back into $C_7$. By using the Lipschitz continuity of the terminal condition for the adjoint variable, Assumption \ref{assm:Obj}, we obtain
\begin{align}\label{eqn:diffininitialcondition}
\mathbb E^1 \bigg[\max_{i\in[1:N]}\sup_{\bm{y}\in (\bar B(R_0))^N} \Vert a^{i,\bm{y}}_{1,\tau} - \bar a^{i,L,\bm{y}}_{1,\tau}\Vert^2 \bigg] &\le 2\Lambda_\ell^2\;\mathbb E^1 \bigg[\max_{i\in[1:N]} \sup_{\bm{y}\in (\bar B(R_0))^N}\Vert X^{i,\bm{y}}_{1,\tau} - \bar X^{i,L,\bm{y}}_{1,\tau}\Vert^2 \bigg]\notag\\&\le 2C_7 \Lambda_\ell^2\left(\frac{1}{L^2}+\frac{1}{HL^{\frac{2}{3}}}+\sum_{j=0}^{\tau-1}\kappa^0_{j,\tau}\mathscr{L}_j\right) .
\end{align}
The second inequality uses equation \eqref{eqn:boundXfinal}.
Since $\mathscr{W}_{r',\tau}$ can be bounded above by the contributions from $(X,a)$ separately, combining equations \eqref{eqn:boundaintermediate}, \eqref{eqn:boundXfinal} and \eqref{eqn:diffininitialcondition} gives for any $l\in[1:L]$
\begin{align*}
\mathscr{W}_{l-1,\tau}\le \frac{C_6}{L}\sum_{r'=l}^{L-1} \mathscr{W}_{r',\tau}+(C_6+C_7(1+2\Lambda_\ell^2))\left(\frac{1}{L^2}+\frac{1}{HL^{\frac{2}{3}}}+\sum_{j=0}^{\tau-1}\kappa^0_{j,\tau}\mathscr{L}_j\right),
\end{align*}
where by convention $ \sum_{r=L}^{L-1}\mathscr{W}_{r,\tau}=0$. 
Therefore, by the discrete G\"onwall inequality
\begin{equation}\label{eqn:boundmathscrW}
    \mathscr{W}_{0,\tau}\le C_8\left(\frac{1}{L^2}+\frac{1}{HL^{\frac{2}{3}}}+\sum_{j=0}^{\tau-1}\kappa^0_{j,\tau}\mathscr{L}_j\right),
\end{equation}
for some constant $C_8=C_8(C_6,C_7,\Lambda_\ell)$.
Notice that $\bar X^{i,L,\bm{y}}_{t,\tau}, \bar a^{i,L,\bm{y}}_{t,\tau}$ are constant on each interval $[\frac{r}{L},\frac{r+1}{L})$. Also, the boundedness of $\Gamma$ and $\mathcal K$, Lemma \ref{lemma:BoundedGamma}, implies that $$\Vert X^{i,\bm{y}}_{t,\tau}-X^{i,\bm{y}}_{r_t,\tau}\Vert\le \frac{R_\theta^2R_X}{L},\qquad \Vert a^{i,\bm{y}}_{t,\tau}- a^{i,\bm{y}}_{r_t,\tau}\Vert \le \frac{ R_a\tilde{B}_\mathcal{K}}{L}. $$ Hence, the time discretisation error between $\mathscr{L}_j$ and $\mathscr{W}_{0,j}$ is of order $\mathcal O(L^{-2})$, and using Young's inequality, we get
\[
\mathscr{L}_\tau\le 2\mathscr{W}_{0,\tau}+ \frac{2 R_\theta^4R_X^2+2R_a^2\tilde{B}_\mathcal{K}^2}{L^2}.
\]
Hence, substituting \eqref{eqn:boundmathscrW} in to the above equation and redefining $C_8$ to absorb the $\mathcal O(L^{-2})$ time discretisation error, we deduce that for every $\tau\in [1:T]$ 
\begin{equation}\label{eqn:mathscrLbeforeGronwall}
    \mathscr{L}_{\tau}\le C_8\left(\frac{1}{L^2}+\frac{1}{HL^{\frac{2}{3}}}+\sum_{j=0}^{\tau-1}\kappa^0_{j,\tau}\mathscr{L}_j\right).
\end{equation}
\textbf{Base Case: $\tau=0$}\\
The validity of Lemma \ref{lemma:StochasticApprox} is restricted to $\tau \ge 1$, since it relies on Lemmas \ref{lemma:AdamLip} and \ref{lemma:GradCont}, which only hold for $\tau \in[1:T]$. These Lemmas are used to control the expression
\[
\int \Vert \Phi_{\frac{r}{L},\tau}[\bm\nu](\theta)-\Phi_{\frac{r}{L},\tau}[\bm{\bar\nu}^{N,L}](\theta)\Vert_\mathrm{F}^2 \,d\bar\nu^{N,L}_{\frac{r}{L},0}(\theta).
\]
However, at $\tau=0$, $\Phi_{\frac{r}{L},0}$ is the identity map, and so this expression is identically zero. As a result, the argument in the proof of Lemma \ref{lemma:StochasticApprox} applies at $\tau=0$, with any term that comes from the above expression set to zero. Therefore, the same conclusion to Lemma \ref{lemma:StochasticApprox} holds, except without the $\mathscr{L}_j$ term. Applying the preceding Gr\"onwall argument yields that \eqref{eqn:mathscrLbeforeGronwall} holds for $0\le \tau\le T$ with the convention $\sum_{j=0}^{-1}\kappa_{j,0}\mathscr{L}_j=0$.

Notice that $\kappa^0_{j,T}$ in \eqref{eqn:kappa} coincides with $\kappa^\lambda_{i,T}$ from \eqref{eqn:kappa_} when $\lambda=0$, so the factors $\alpha_{\tau+1,T}$ reduce to one. Furthermore, since the proof of Lemma \ref{lemma:RecursiveSum} does not involve division by $\lambda$, its conclusion still holds with $\kappa^\lambda_{i,T}$ replaced by $\kappa^0_{i,T}$. Therefore, there exists a constant $C_9$ that scales exponential in the total step length $\sum_{j=0}^T\eta_j$ such that for every $\tau\in[1:T]$
\begin{equation}\label{eqn:boundmathscrL}
    \mathscr{L}_\tau \le C_9\left( \frac{1}{L^2}+ \frac{1}{L^{2/3}H}\right),
\end{equation}
holds $\mathbb P^0$-almost surely. This is precisely the conclusion to Theorem \ref{theorem:WTS}.\\

 Since the maximum of $W^2_{\infty}(\bar\rho^{N,L,b}_{\frac{r}{L},j},\rho^{N,b}_{\frac{r}{L},j})$ over $0\le r < L, 1\le b \le B$ is bounded above in expectation by $\mathscr{L}_j$ for all $0\le j<T$, substituting this into \eqref{eqn:boundparams} yields
\[
\E^1\left[\max_{0\le r < L}W_2^2(\hat \nu^{N,L}_{\frac{r}{L},\tau},\bar \nu^{N,L}_{\frac{r}{L},\tau})\right]\le D_1\sum_{j=0}^{\tau-1}\kappa^0_{j,\tau}\mathscr{L}_j.
\]
Applying \eqref{eqn:boundmathscrL} to bound $\mathscr{L}_\tau$ then gives
\[
\max_{0\le \tau\le T}\E^1\left[\max_{0\le r< L} W_2^2(\hat \nu^{N,L}_{\frac{r}{L},\tau},\bar \nu^{N,L}_{\frac{r}{L},\tau})\right]\le D_1C_9\left(\frac{1}{L^2}+\frac{1}{L^{2/3}H}\right) \max_{0\le \tau \le T}\sum_{j=0}^{\tau-1}\kappa^0_{j,\tau}, \qquad \mathbb P^0-a.s.
\]
Lemma \ref{lemma:ParamDiv} now follows by bounding the sum of $\kappa^0_{i,T}$ defined in \eqref{eqn:kappa}. By interchanging the order of summation, we obtain
\[
\sum_{j=0}^{\tau-1}\kappa^0_{j,\tau} = \sum_{i=1}^T\sum_{j=0}^{i-1}\frac{(1-\beta_1)\eta_i}{1-\beta_1^i}(\beta_1^{i-j-1}+2\beta_2^{i-j-1})= C_\kappa\sum_{i=1}^\tau\eta_i,
\]
where the geometric series is evaluated to give the constant $C_\kappa$defined in \eqref{eqn:ckappa}.
\subsection{Existence of $\widetilde{J}$}\label{sec:Jtilde}
The existence of $\widetilde{J}$ is clear when $f_1,f_2,f_3$ are either linear or independent of $z,w$. Therefore, the only problematic $f$ are $f^{\gamma z}_2, f_1^{\gamma\mu}$ and $f^{\gamma\mu}_2$ given in the previous section. For ease of notation, we shall write $A=\beta \theta_K^T\theta_Q$, $U=\theta_O^T\theta_V$ and denote the indices of the $j$-th block of $d$ components by $I_j=[(j-1)d+1:jd]$. We state the relevant $\widetilde J$ in block form, where most of the blocks are identically zero.

Using Definition \ref{def:J}, we see that the $J^{\gamma z}_2$ corresponding to $f^{\gamma z}_2$ in \eqref{eqn:fgz} satisfies Assumption \ref{assm:MultiLinear} with $\widetilde{J}$ given by
\begin{align*}
J^{\gamma z}_2(z,w,u,\theta) &= \left\langle \widetilde{J}_2^{\gamma z} ,z\otimes w\otimes w\otimes u\right\rangle_\mathrm{HS},\\ (\widetilde{J}_2^{\gamma z})_{ijkl} &= \begin{cases}
    A_{kl}\, U_{ij} & \text{if }i\in I_2,\, j\in I_1, \,k\in I_1,\, l\in I_1\\
    -A_{kl}\, U_{ij} & \text{if }i\in I_2,\, j\in I_1,\, k\in I_3,\, l\in I_1\\
    0 & \text{otherwise}
\end{cases},
\end{align*}
for $i\in[1:2d],j,k\in[1:4d]$ and $l\in[1:d]$. Since $\langle \theta_Q w_1,\theta_K z_1\rangle$ is linear in both $z,w$, only the term  $\langle \theta_Q w_1,\theta_Kw_3\rangle$ in $f_1^{\gamma\mu}$, defined in \eqref{eqn:fgmu}, needs further study. Notice that it can be expressed as 
\[
\beta\langle \theta_Q w_1,\theta_Kw_3\rangle = \left\langle \widetilde{J}^{\gamma\mu}_{1,2}, w w^T \right\rangle_\mathrm{F}, \qquad (\widetilde{J}_{1,2}^{\gamma \mu})_{ij} = \begin{cases}
    A_{ij} &\text{if }i\in I_3,\, j \in I_1\\
    0 & \text{otherwise}
\end{cases},
\]
for $i,j\in[1:4d]$. Thus, $f_1^{\gamma\mu}$ is the sum of two terms that satisfy Assumption \ref{assm:MultiLinear}. Likewise, the final term in $f_2^{\gamma\mu}$ defined in \eqref{eqn:fgmu} is linear in $w_2$ and so doesn't require further study. The other terms of $J_2^{\gamma\mu}$ satisfy Assumption \ref{assm:MultiLinear} with $\widetilde{J}$ given by 
\begin{align*}
        \beta\langle \theta_K u,\theta_Q w_1\rangle\langle \theta_V z_1, \theta_O w_2 \rangle &= \left\langle \widetilde{J}^{\gamma\mu}_{2,2}, z\otimes w^{\otimes 2}\otimes u\right\rangle_{\mathrm{HS}},\\ (\widetilde{J}_{2,2}^{\gamma z})_{ijkl} &= \begin{cases}
    A_{lk}\, U_{ji} & \text{if } i\in I_1,\, j\in I_2,\, k\in I_1,\, l\in I_1\\
    0 & \text{otherwise}
\end{cases},\\
        -\beta\langle \theta_K u,\theta_Q w_1\rangle\langle \theta_V w_3, \theta_O w_2 \rangle &= \left\langle \widetilde{J}^{\gamma\mu}_{2,3}, w^{\otimes 3}\otimes u\right\rangle_{\mathrm{HS}},\\ (\widetilde{J}_{2,3}^{\gamma z})_{ijkl} &= \begin{cases}
    -A_{lk}\, U_{ji} & \text{if } i\in I_3,\, j\in I_2,\, k\in I_1,\, l\in I_1\\
    0 & \text{otherwise}
\end{cases}.
\end{align*}

    \section{Preliminary Lemmas on AdamW}\label{App:K}
\subsection{Properties of $\kappa_{i,T}$}
\begin{lemma}\label{lemma:RecursiveSum}
Let $D\ge0$,  $\kappa^\lambda_{i,T}$ be defined by \eqref{eqn:kappa_}, with $\beta_1\le \beta_2$, and $(A_j)_{j=0}^T$ a non-negative sequence such that $A_j\ge(1-\eta_j\lambda)A_{j-1}$. Suppose $u_j$ satisfies 
    \begin{equation}\label{eqn:recursion}
            u_j \le A_j+D\sum_{i=0}^{j-1}\kappa^\lambda_{i,j}u_i, 
    \end{equation}
for all $j\ge 0$, where by convention $\sum_{i=0}^{-1}=0$. Then, there exists a constant $\widetilde D=\widetilde D(\lambda,\beta_1,\beta_2,D)$ such that
 \begin{align}\label{eqn:cgamma}
  u_T &\le \left(A_0+\sum_{j=1}^T\left(A_j-(1-\eta_j\lambda)A_{j-1}\right)\right)exp\left[\widetilde DS_T\right],\qquad S_T=\sum_{j=0}^T\eta_j.
 \end{align}
 When $(A_j)^T_{j=0}$ is non-decreasing, the bound simplifies to $u_T\le A_T(1+\lambda S_T)\exp(\widetilde{D}S_T)$.
\end{lemma}
\begin{proof}
Let $v_j$ be defined by the sequence that satisfies \eqref{eqn:recursion} with equality and $v_0=A_0$.  Notice that $\kappa^\lambda_{i,T}$, defined in equation, \eqref{eqn:kappa_} satisfies the recursion relationship
\[
\kappa^\lambda_{j,T} = (1-\eta_T\lambda)\kappa^\lambda_{j,T-1}+\eta_T\frac{(1-\beta_1)}{1-\beta_1^{T}}(\beta_1^{T-j-1}+2\beta_2^{T-j-1}),\qquad \kappa^\lambda_{T-1,T} = 3\eta_{T} \frac{1-\beta_1}{1-\beta_1^{T}}, 
\]
for any $j\in[0:T-2]$. Therefore, substituting this recursion formula into the series expression for $v_T$ gives
\begin{align*}
v_T = A_T&+ D\sum_{i=0}^{T-2}\left((1-\eta_T\lambda)\kappa^\lambda_{i,T-1}v_i+(1-\beta_1)\frac{\eta_T(\beta_1^{T-i-1}+2\beta_2^{T-i-1})}{1-\beta_1^T}v_i \right)+ 3Dv_{T-1}\eta_T\frac{1-\beta_1 }{1-\beta_1^T},
\end{align*}
for $T\ge 1$. The second term is precisely $(1-\eta_T\lambda)(v_{T-1}-A_{T-1})$, and so 
\begin{align}
    v_T&=A_T+ (1-\eta_T\lambda)(v_{T-1}-A_{T-1})+ D\eta_T (1-\beta_1)\sum_{i=0}^{T-1}\frac{\beta_1^{T-i-1}+2\beta_2^{T-i-1}}{1-\beta_1^T}v_i\notag\\
    &\le(1+\eta_T\left( DC_\kappa-\lambda)\mathbbm{1}_{DC_{\kappa}>0}\right)\max_{0\le j\le T-1}v_j+A_T-(1-\eta_T\lambda)A_{T-1},\label{eqn:conv_gronwall}
\end{align}
having simplified the geometric series via the bound
\begin{equation}\label{eqn:ckappa}
    \frac{1-\beta_1}{1-\beta_1^T}\left(\sum_{i=0}^{T-1}\beta_1^{T-i-1}+2\beta_2^{T-i-1}\right)=1+2\frac{1-\beta_2^T}{1-\beta_1^T}\frac{1-\beta_1}{1-\beta_2}\le 1+2B_\beta^2 =:C_\kappa,
\end{equation}
where $B_\beta$ is defined in Lemma \ref{lemma:BoundDelta}. Since either $v_T<v_j$ for some $j\in[0:T-1]$, or $v_T$ is bounded by equation \eqref{eqn:conv_gronwall}, we obtain
\[
\max_{0\le j \le T} v_j\le(1+\eta_T\left( DC_\kappa-\lambda)\mathbbm{1}_{DC_\kappa>\lambda}\right)\max_{0\le j\le T-1}v_j+\left(A_T-(1-\eta_T\lambda)A_{T-1}\right),\]
for $T\ge 1$. Here, we have used the assumed inequality $A_j\ge (1-\eta_j\lambda)A_{j-1}$ for every $j\in[1:T]$. Hence,  using the convention $\prod_{i=T+1}^TA_i=1$ and applying the discrete Gr\"onwall inequality with $v_0=A_0$ gives
\begin{align*}
    \max_{0\le i \le T}v_i &\le A_0\prod_{j=1}^T(1+\eta_j(DC_\kappa-\lambda)\mathbbm{1}_{DC_\kappa>\lambda})\\&+
    \sum_{l=1}^T\left(A_l-(1-\eta_l\lambda)A_{l-1}\right)\prod_{j=l+1}^{T}(1+\eta_j(DC_\kappa-\lambda)\mathbbm{1}_{DC_\kappa>\lambda})\\
    &\le \left(A_0+\sum_{j=1}^T\left(A_j-(1-\eta_j\lambda)A_{j-1}\right)\right)\exp\left[(DC_\kappa-\lambda)S_T\mathbbm{1}_{DC_\kappa>\lambda}\right],
\end{align*}
where $S_T=\sum_{j=1}^T\eta_j$. If the sequence $(A_j)_{j=0}^T$ is non-decreasing, the telescoping sum simplifies to give
\[
\max_{0\le i \le T}v_i\le A_T(1+\lambda S_T)\exp\left[(DC_\kappa-1)S_T\mathbbm{1}_{DC_\kappa>\lambda}\right].
\]
By induction, we see that $u_i\le v_i$ for every $i\in[0:T]$ and so $u_T\le \max_{1\le i \le T}v_i$, which completes the proof.
\end{proof}
\begin{proposition}\label{prop:KappaSum}
Take $\kappa^\lambda_{j,T}$ defined by equation \eqref{eqn:kappa_}. Then, with  $C_\kappa$ defined in \eqref{eqn:ckappa}, we have that 
\[
\sum_{j=0}^{T-1}\kappa^\lambda_{j,T}\le C_\kappa\lambda^{-1}.
\]
\end{proposition}
\begin{proof}
Recall the expression for $\kappa_{i,T}$ given in \eqref{eqn:kappa}. By interchanging the order of summation, we obtain
\begin{align*}
    \sum_{i=0}^{T-1}\kappa^\lambda_{i,T}= (1-\beta_1)\sum_{j=1}^T\frac{\eta_{j}\alpha_{j+1,T}}{1-\beta_1^j}\sum_{i=0}^{j-1}(\beta_1^{j-i-1}+2\beta_2^{j-i-1}).
\end{align*}
Then, simplifying the geometric series via equation \eqref{eqn:ckappa} gives
\[
\sum_{i=0}^{T-1}\kappa^\lambda_{i,T} = C_\kappa  \sum_{j=1}^T \eta_j\alpha_{j+1,T}.
\]
Since $-\eta_j\alpha_{j+1,T} =\lambda^{-1}(\alpha_{j,T}-\alpha_{j+1,T})$ and $\alpha_{T+1,T}=1$ by convention, we have that
\begin{align*}
    \sum_{i=0}^{T-1}\kappa^\lambda_{i,T}
    & = \frac{C_\kappa}{\lambda}\sum_{j=1}^T (\alpha_{j+1,T}-\alpha_{j,T})\le \frac{C_\kappa}{\lambda}.
\end{align*}
\end{proof}
\subsection{Stability of AdamW}\label{App:K2}
\begin{lemma}\label{lemma:BoundDelta}
Let $\hat m^\mathcal R_k, \hat v^\mathcal R_k$ be the mean and variance accumulators for the gradients $(g_1,\ldots,g_k)$, $g_i \in \mathscr{V}$ defined in Algorithm \ref{alg:AdamW} for $\mathcal R:\mathscr{V}\mapsto \mathscr{V}$ either the identity map or given by equation \eqref{eqn:R}. Suppose that the constants $\beta_1,\beta_2\in (0,1)$ in Algorithm \ref{alg:AdamW} are such that $\beta_1\le \beta_2$. Then, we have
\[
\left\Vert \mathcal R \left(\frac{\hat m^{\mathcal R}_j}{\sqrt{\hat v^{\mathcal R}_j}+\varepsilon} \right)\right\Vert_\infty \le \sqrt{\frac{(1-\beta_2^j)(1-\beta_1)}{(1-\beta_1^j)(1-\beta_2)}}\le B_\beta,
\]  
where $B_\beta:=(1-\beta_1)^{1/2}(1-\beta_2)^{-1/2}$.
\end{lemma}
\begin{proof}\label{proof:BoundUpdtae} This proof adapts that given in Lemma 4.2 from \cite{pmlr-v235-xie24e}. 
Write $\varepsilon_j = \varepsilon \sqrt{1-\beta_2^j}$ and notice that each choice of $\mathcal R$ is component-wise $1$-homogeneous and sub-additive. Furthermore, the output of $\mathcal R$ is constant within blocks of $\R^{k\times 4d}$, and so the variance accumulator $\hat v^\mathcal R$ is constant within these blocks. As a result, the $1$-homogeneity of $\mathcal R$ allows the division by the variance accumulators to be factored out, yielding
\begin{align*}
   \left\Vert \mathcal R\left(\frac{\hat m^\mathcal R_j}{\sqrt{\hat v^\mathcal R_j}+\varepsilon}\right) \right\Vert_\infty &= \frac{\sqrt{1-\beta_2^j}}{1-\beta_1^j} \left\Vert\frac{(1-\beta_1)\mathcal R\left(\sum_{i=1}^j \beta_1^{j-i}g_i\right)}{\sqrt{v^\mathcal R_j}+\varepsilon_j} \right\Vert_\infty\\&\le \frac{\sqrt{1-\beta_2^j}}{1-\beta_1^j} \left\Vert\frac{(1-\beta_1)\sum_{i=1}^j \beta_1^{j-i}\mathcal R\left(g_i\right)}{\sqrt{v^\mathcal R_j}+\varepsilon_j} \right\Vert_\infty.
\end{align*}
The final inequality results for $\mathcal R$ either being the identity or satisfying a componentwise triangle inequality.
Then, by the Cauchy-Schwarz inequality, we have that
\begin{align}\label{eqn:BoundUpdate1}
   \left\Vert \frac{\mathcal R\left(\hat m^\mathcal R_j\right)}{\sqrt{\hat v^\mathcal R_j}+\varepsilon} \right\Vert_\infty &\le \frac{\sqrt{1-\beta_2^j}}{1-\beta_1^j}\left\Vert(1-\beta_1)\sum_{i=1}^j \frac{\beta_1^{j-i}\mathcal R( g_i)^{\odot 2}}{\left(\sqrt{v^\mathcal R_j}+\varepsilon_j\right)^{\odot 2}}\right\Vert_\infty^{\frac{1}{2}}\left\vert(1-\beta_1)\sum_{i=1}^j\beta_1^{j-i}\right\vert^{\frac{1}{2}}.
\end{align}
The final geometric series simplifies to $\sqrt{1-\beta_1^j}$.
By substituting in the recursion relation for $v^\mathcal R_j$ given in Algorithm \ref{alg:AdamW}, we obtain
\begin{align*}
   \sum_{i=1}^j \frac{\beta_1^{j-i}(\mathcal R (g_i))^{\odot 2}}{(\sqrt{v^\mathcal R_j}+\varepsilon_j)^{\odot 2}}&=\frac{1}{1-\beta_2}\sum_{i=1}^j \beta_1^{j-i}\frac{v^\mathcal R_i-\beta_2v^\mathcal R_{i-1}}{(\sqrt{v^\mathcal R_j}+\varepsilon_j)^{\odot 2}}\\
    &= \frac{1}{1-\beta_2}\frac{v^\mathcal R_j}{(\sqrt{v^\mathcal R_j}+\varepsilon_j)^{\odot 2}}+\frac{1}{1-\beta_2}\sum_{i=1}^{j-1}(\beta_1-\beta_2) \frac{\beta_1^{j-i-1} v^\mathcal R_i}{(\sqrt{v^\mathcal{R}_j}+\varepsilon_j)^2}.
\end{align*}
The second equality uses the initial condition $v^\mathcal R_0=0$. Since $(\beta_1-\beta_2)<0$ and each component of $v^\mathcal R_i$ is greater than zero, the final sum in the penultimate line is componentwise less than or equal to zero. Therefore, substituting the preceding expression into  \eqref{eqn:BoundUpdate1} gives
\[
\left\Vert \mathcal R\left(\frac{\hat m^\mathcal R_j}{\sqrt{\hat v^\mathcal R_j}+\varepsilon}\right) \right\Vert_\infty\le \sqrt{\frac{1-\beta_2^j}{1-\beta_1^j}}\left\Vert \frac{1-\beta_1}{1-\beta_2}\frac{v^\mathcal R_j}{\left(\sqrt{v^\mathcal R_j}+\varepsilon_j\right)^{\odot 2}} \right\Vert_\infty^{\frac{1}{2}}\le \sqrt{\frac{1-\beta_2^j}{1-\beta_1^j}\frac{1-\beta_1}{1-\beta_2}}.
\]
\end{proof}
\begin{lemma}\label{lemma:UpdateStability}
Given sequences of gradients $G^{(1)}=\{g^{(1)}_i\}_{i=1}^j,G^{(2)}=\{ g^{(2)}_i\}_{i=1}^j $ in $\mathscr{V}$, and let $\hat m^{\mathcal R,(i)}_j$, $\hat v^{\mathcal R,(i)}_j$, $i=1,2$ be the momentum and variance accumulators defined in equation \eqref{eqn:AdamWalg} computed using $G^{(i)}$. Then, with $\mathcal R$ either the identity or given by \eqref{eqn:R}, the Adam updates satisfy the Lipchitz estimate
\begin{align*}
    \left\Vert\frac{\hat m^{\mathcal R,(1)}_j}{\sqrt{\hat v^{\mathcal R,(1)}_j}+\varepsilon} -\frac{\hat m^{\mathcal R,(2)}_j}{\sqrt{\hat v^{\mathcal R,(2)}_j}+\varepsilon} \right\Vert^2_\mathrm{F} &\le \frac{2(1-\beta_1)}{\varepsilon^2(1-\beta_1^j)}\left[\sum_{i=1}^j\left(\beta_1^{j-i}+2 \beta_2^{j-i}\right)\Delta^2_i \right],
\end{align*}
where $ \Delta_i = \Vert g^{(1)}_i-g^{(2)}_i\Vert_\mathrm{F}$.
\end{lemma}
\begin{proof}
Since the division in the Adam updates is applied componentwise, adding and subtracting $\hat m^{\mathcal{R},(2)}_j$ divided by $ \sqrt{\hat{v}_j^{\mathcal{R},(1)}}+\varepsilon$  gives
\begin{align}\label{eqn:AdamLip1}
   \left\Vert\frac{\hat m^{\mathcal R,(1)}_j}{\sqrt{v^{\mathcal R,(1)}_j}+\varepsilon} -\frac{\hat m^{\mathcal R,(2)}_j}{\sqrt{v^{\mathcal R,(2)}_j}+\varepsilon} \right\Vert_\mathrm{F} &\le \left\Vert\frac{\hat m^{\mathcal R,(1)}_j-\hat m^{\mathcal R,(2)}_j}{\sqrt{\hat v^{\mathcal R,(1)}_j}+\varepsilon}\right\Vert_\mathrm{F} \\&+ \left\Vert\frac{\left(\hat m^{\mathcal R,(2)}_j\right) \odot\left( \hat v^{\mathcal R,(1)}_j-\hat v^{\mathcal R,(2)}_j\right)}{\left(\sqrt{\hat v^{\mathcal R,(1)}_j}+\varepsilon\right)\left(\sqrt{\hat v^{\mathcal R,(2)}_j}+\varepsilon\right)\left(\sqrt{\hat v^{\mathcal R,(1)}_j}+\sqrt{\hat v^{\mathcal R,(2)}_j}\right)}\right\Vert_\mathrm{F},\notag
   \end{align}
   where all multiplication and division is understood to hold componentwise.
  The squared Frobenius norm is equal to the sum of the squared Frobenius norms on each block that $\mathcal R$ and $\hat v^{\mathcal R,(2)}_j$ are constant on. Write $\hat v^{\mathcal R,(1)}_{j,i}$ for a representative of the $i$-th out of $N(\mathcal R)$ blocks that $\mathcal R$ is constant on. With this notation, we get that
\begin{align}\label{eqn:bound_v1}
    \Bigg\Vert&\frac{\big(\hat m^{\mathcal R,(2)}_j\big) \odot\big( \hat v^{\mathcal R,(1)}_j-\hat v^{\mathcal R,(2)}_j\big)}{\big(\sqrt{\hat v^{\mathcal R,(1)}_j}+\varepsilon\big)\big(\sqrt{\hat v^{\mathcal R,(2)}_j}+\varepsilon\big)\big(\sqrt{\hat v^{\mathcal R,(1)}_j}+\sqrt{\hat v^{\mathcal R,(2)}_j}\big)}\Bigg\Vert^2_\mathrm{F}\\& \qquad\qquad\qquad\le \left\Vert \frac{\mathcal R\left(\hat m^{\mathcal R,(2)}_j\right)}{\sqrt{\hat v^{\mathcal R,(2)}_j}+\varepsilon}\right\Vert_\infty^2\sum_{i=1}^{N(\mathcal R)} \frac{\left( \hat v^{\mathcal R,(1)}_{j,i}-\hat v^{\mathcal R,(2)}_{j,i}\right)^2}{\left(\sqrt{\hat v^{\mathcal R,(1)}_{j,i}}+\varepsilon\right)^2\left(\sqrt{\hat v^{\mathcal R,(1)}_{j,i}}+\sqrt{\hat v^{\mathcal R,(2)}_{j,i}}\right)^2}\notag\\&\qquad\qquad\qquad\le\frac{(1-\beta_2^j)(1-\beta_1)}{(1-\beta_1^j)(1-\beta_2)}\sum_{i=1}^{N(\mathcal R)} \frac{\left( \hat v^{\mathcal R,(1)}_{j,i}-\hat v^{\mathcal R,(2)}_{j,i}\right)^2}{\left(\sqrt{\hat v^{\mathcal R,(1)}_{j,i}}+\varepsilon\right)^2\left(\sqrt{\hat v^{\mathcal R,(1)}_{j,i}}+\sqrt{\hat v^{\mathcal R,(2)}_{j,i}}\right)^2},\notag
\end{align}
  where we have used the training step dependent bound in Lemma \ref{lemma:BoundDelta}. Then, writing $g_{\tau,i}$ for the $i$-th block of gradients at training step $\tau$,  the Cauchy-Schwarz inequality implies that each block satisfies the following  Lipschitz continuity estimate
   \begin{align*}
         \left( v^{\mathcal R,(1)}_{j,i}- v^{\mathcal R,(2)}_{j,i}\right)^2 &\le \left((1-\beta_2)\sum_{\tau=1}^j\beta_2^{j-\tau}\left(\Vert g^{(1)}_{\tau,i}\Vert -\Vert g^{(2)}_{\tau,i}\Vert \right)\left(\Vert g^{(1)}_{\tau,i}\Vert+\Vert g^{(2)}_{\tau,i}\Vert\right)\right)^2\\
         &\le \left((1-\beta_2)\sum_{\tau=1}^j\beta_2^{j-\tau}(\Vert g^{(1)}_{\tau,i}\Vert -\Vert g^{(2)}_{\tau,i}\Vert)^2\right) \left( (1-\beta_2)\sum_{\tau=1}^j\beta_2^{j-\tau}(\Vert g^{(1)}_{\tau,i}\Vert+\Vert g^{(2)}_{\tau,i}\Vert)^2\right)
        . 
    \end{align*}
  Applying the reverse triangle inequality, Young's inequality and restoring the bias correction by dividing by $(1-\beta_2^j)^2$, on the above expression, yields
    \begin{align}\label{eqn:bound_v}
    \left( \hat v^{\mathcal R,(1)}_{j,i}- \hat v^{\mathcal R,(2)}_{j,i}\right)^2&\le  2\left(\hat v^{\mathcal R,(1)}_{j,i}+\hat v^{\mathcal R,(2)}_{j,i}\right) \left( \frac{(1-\beta_2)}{(1-\beta_2^j)}\sum_{\tau=1}^j\beta_2^{j-\tau}\Vert g^{(1)}_{\tau,i}-g^{(2)}_{\tau,i}\Vert^2\right)
   \end{align}
 Therefore, by substituting this bound back into equation \eqref{eqn:bound_v1} and applying the fact that $(\sqrt{x}+\sqrt{y})^2\ge x+y$ for $x,y\ge 0$ to the denominator, we obtain
 \begin{align}\label{eqn:change_in_v}
     \Bigg\Vert&\frac{\big(\hat m^{\mathcal R,(2)}_j\big) \odot\big( \hat v^{\mathcal R,(1)}_j-\hat v^{\mathcal R,(2)}_j\big)}{\big(\sqrt{\hat v^{\mathcal R,(1)}_j}+\varepsilon\big)\big(\sqrt{\hat v^{\mathcal R,(2)}_j}+\varepsilon\big)\big(\sqrt{\hat v^{\mathcal R,(1)}_j}+\sqrt{\hat v^{\mathcal R,(2)}_j}\big)}\Bigg\Vert^2_\mathrm{F}\le  \frac{2(1-\beta_1)}{\varepsilon^2(1-\beta_1^j)}\sum_{\tau=1}^j\beta_2^{j-\tau}\Delta^2_\tau.
 \end{align}
Furthermore, the Cauchy-Schwarz inequality implies that 
\begin{align}\label{eqn:change_in_m}
  \left\Vert \hat m^{\mathcal R,(1)}_j- \hat m^{\mathcal R,(2)}_j\right\Vert^2_\mathrm{F} &\le \left((1-\beta_1)\sum_{\tau=1}^j\frac{\beta_1^{j-\tau}}{(1-\beta_1^j)^2}\Delta_\tau^2 \right)\left((1-\beta_1)\sum_{\tau=1}^j \beta_1^{j-\tau}\right)\\& \le  \left((1-\beta_1)\sum_{\tau=1}^j\beta_1^{j-\tau}\Delta_\tau^2 \right)(1-\beta_1^j)^{-1}\notag.
\end{align}
By applying Young's inequality on equation \eqref{eqn:AdamLip1} and substituting in bounds \eqref{eqn:change_in_v} and \eqref{eqn:change_in_m} into the resulting expression gives the desired result.
\end{proof}

\end{appendices}

\end{document}